%% file: main.tex
\documentclass[letterpaper,11pt]{article}

\input{macros}

\usepackage{fontawesome}

\usepackage{comment}

\usepackage{boxedminipage}
\usepackage{xcolor}
\usepackage{cite}
\usepackage{enumitem}
\usepackage{url}
\usepackage{xspace}

\hypersetup{
    linkcolor=black
}

\title{
From Fairness to Infinity: 
Outcome-Indistinguishable (Omni)Prediction in Evolving Graphs} 
\remove{
A New Approach to Online Outcome-Indistinguishability, Applied to Edge Formation

Predicting Edge-Formation in Evolving Graphs via
Online Outcome-Indistinguishability 

Link Prediction for Evolving Graphs\\
New Applications for a New Technique in Online Outcome-Inditinguishability

Link Prediction for Evolving Graphs\\
A Case Study for a New Approach to online Outcome-Inditinguishability

Something about fairness?
(Social loss)

From Fairness to Infinity: 
Outcome-Indistinguishable Edge (Omni)Prediction in Evolving Graphs

From Fairness to Infinity: 
Outcome-Indistinguishable (Omni)Prediction in Evolving Graphs

From Fairness to Infinity: 
Outcome-Indistinguishable (Omni)Prediction in Professional Networks

\cynthia{Swap omniprediction?  Here or in intro.}
}
\author{Cynthia Dwork\\ Harvard  \and Chris Hays\\ MIT \and Nicole Immorlica  \\ Microsoft Research  \and Juan C. Perdomo \\ Harvard \and Pranay Tankala \\ Harvard }

\date{\today}
\begin{document}

\maketitle

\pagenumbering{gobble}
\pagenumbering{arabic}





\title{}
\thispagestyle{empty}
\begin{abstract}
Professional networks provide invaluable entree to opportunity through referrals and introductions. A rich literature shows they also serve to entrench and even exacerbate a status quo of privilege and disadvantage. 
Hiring platforms, equipped with the ability to nudge link formation, provide a tantalizing opening for beneficial structural change.  
We anticipate that key to this prospect will be the ability to estimate the likelihood of edge formation in an evolving graph. 

Outcome-indistinguishable prediction algorithms ensure that the modeled world is indistinguishable from the real world by a family of statistical tests.  Omnipredictors ensure that predictions can be post-processed to yield loss minimization competitive with respect to a benchmark class of predictors for many losses simultaneously, with appropriate post-processing.
We begin by observing that, by combining a slightly modified form of the online $\kalg$ algorithm of Vovk (2007) with basic facts from the theory of reproducing kernel Hilbert spaces, one can derive simple and efficient online algorithms satisfying 
outcome indistinguishability and omniprediction, with guarantees that improve upon, or are complementary to, those currently known. This is of independent interest.

 We apply these techniques to evolving graphs, obtaining online outcome-indistinguishable omnipredictors for rich --- possibly infinite --- sets of distinguishers that capture properties of pairs of nodes, and their neighborhoods.  This yields, \textit{inter alia}, multicalibrated predictions of edge formation with respect to pairs of demographic groups, and the ability to simultaneously optimize loss as measured by a variety of social welfare functions.

\end{abstract}
\clearpage
\setcounter{page}{1}


\setcounter{tocdepth}{2}
\tableofcontents

\clearpage

\section{Introduction}
Professional networks provide invaluable entree to opportunity through referrals and introductions. 
%
%
A rich literature shows they may also serve to entrench and even exacerbate a status quo of privilege and disadvantage.  For example, in a network with two disjoint groups with equal ability distribution, homophily can, through job referrals, result in the draining of opportunity from the smaller group to the larger~\cite{bolte2020role,calvo2004effects,okafor2020social}. 
%
Remedies are few.  Hiring platforms, equipped with the ability to nudge link formation, provide a tantalizing opening for beneficial structural change. 



Key to this prospect is the ability to estimate edge formation in an evolving network.
%
This is a prediction problem for the universe of pairs of network nodes (individuals) $(i,j)$, suggesting that standard prediction methods can be applied.  While this intuition is correct, the situation is complicated by the fact that edge formation need not be a property of the endpoints alone, but can also depend on the topology and other features of the neighborhoods of the principals $i$ and $j$. 
For example, the probability that the edge  $(i,j)$ forms may be a function of the number of contacts that $i$ and $j$ have in common among other factors. 
Let us informally call this the problem of complex domains. 
To complicate matters even further, these features change over time as individuals grow their networks, switch jobs, \textit{etc}. 
We treat edge prediction in a social network as an online, distribution-free problem and aim to make predictions that are valid and useful, \emph{regardless} of the underlying edge formation process.

Since one of our overarching goals is fairness in networking, we certainly want these predictions to satisfy a rich collection of ``fair accuracy'' criteria, which we express in the language of \textit{outcome indistinguishability}~\cite{dwork2021outcome} and \textit{multicalibration}~\cite{hkrr}.  
Moreover, we would like the predictions to be simultaneously loss minimizing (with appropriate post-processing) with respect to a benchmark class of predictors, for a collection of loss functions expressing goals of social welfare; that is, we want {\em omniprediction}~\cite{gopalan2022omnipredictors,garg2024oracle}.  
Putting these together, we want low-regret, online, outcome-indistinguishable  omnipredictors for complex domains.  We would also like the predictors to be computationally efficient.  This is the {\em \fair{} edge omniprediction problem} solved herein. 

%


{\em Outcome indistinguishability} (OI) frames learning not as loss minimization – the dominant paradigm in supervised machine learning --- but instead as satisfaction of a collection of “indistinguishability” constraints. Outcome indistinguishability considers two alternate worlds of individual-outcome pairs: in the natural world, individuals’ outcomes are generated by Real Life’s true distribution; in the simulated world, individuals’ outcomes are sampled according to a predictive model. Outcome indistinguishability requires the learner to produce a predictor in which the two worlds are computationally indistinguishable.  This is  captured by specifying a class of distinguishers to be fooled by the predictor.

Simplifying for ease of exposition, one may define a class of distinguishers corresponding to a (possibly infinite) collection of (possibly intersecting) demographic groups and prediction values, in which case outcome indistinguishability ensures that the predictor is calibrated simultaneously on each group when viewed in isolation.  This is {\em multicalibration}, defined in the seminal work of H\'{e}bert-Johnson, Kim, Reingold, and Rothblum~\cite{hkrr}\footnote{\cite{dwork2021outcome} defines a hierarchy of outcome indistinguishability results, according to the degree of access to the predictor that is given to the distinguishers. When not otherwise specified, we are referring to {\em sample-access} OI.  The term {\em multicalibration} has become more general than its usage here, referring also to a class of real-valued functions (see, \textit{e.g.}, \cite{gopalan2022omnipredictors}). 
For equivalences, see~\cite{dwork2021outcome,gopalan2022omnipredictors}.}; the view of simultaneous calibration in different demographic groups as a potential {\em fairness} goal was introduced by Kleinberg, Mullainathan, and Raghavan~\cite{kleinberg2016inherent}.

(Online) {\em omnipredictors}~\cite{gopalan2022omnipredictors,garg2024oracle} produce predictions that can be used to ensure loss minimization for a wide, even infinite, collection of loss functions, with respect to a benchmark class of predictors.  For example, in the batch case one might train a predictor to optimize squared loss, but later one might wish to deploy the predictor in a way that minimizes 0-1 loss {\em with no further training}.  Omnipredictors make this possible. 
Omniprediction, too, can be expressed in the language of outcome indistinguishability~\cite{loss_oi}. 

A full treatment of fairness in networking requires  understanding which kinds of links will advance social and/or individual welfare and which nudges are likely to be most beneficial.  We hope our work serves as an important first step towards addressing these questions.  In addition, as it is infeasible to make predictions for all non-edges and a random nudge may likely be useless, platform-assisted fair networking will require policies for focusing the platform's attention, a subject for future work. 
\subsection{Our contributions and related work.}

We initiate the study of online outcome indistinguishability and omniprediction for link formation.  Our technical starting point is a novel, randomized variant of Vovk's $K29^*$ online prediction algorithm~\cite{vovk2007non}. Our algorithm, which we call the \algname{}, achieves kernel outcome indistiguishability, that is, indistinguishability with respect to any infinite collection of real-valued functions in a reproducing kernel Hilbert space.\footnote{Informally for now, RKHSs are potentially very rich classes of non-parametric functions.}
To our knowledge, our work is the first in the multigroup fairness literature to use kernel methods (see, however, \cite{perez2017fair, tan2020learning, perez2023fair} for other applications to fairness).
Building on this new algorithm, we design efficient \kf s that capture rich information necessary for the fair link prediction criteria mentioned above.   

In particular, using the $\algname$, we obtain outcome indistinguishability with respect to distinguishers that 
take into account socially meaningful collections of edges (for example, edges between pairs of demographic groups), graph topology (\textit{e.g.}, number of mutual connections, isomorphism class of the local neighborhoods), as well as any bounded function (including those computable by graph neural networks).

Link predictions may be used for a variety
of downstream decisions; for example loss functions may be used to measure predictive accuracy or desirability of outcomes.  Moreover the precise loss function may not be known at prediction time. In particular, a predictive system may need to be fixed in advance of A/B testing to determine which of several candidate loss functions encourages desirable behavior.  We show how to address these problems by using the \algname{} to achieve computationally efficient low-regret omniprediction with respect to potentially infinite and continuous-valued comparison classes; it is precisely the connection to \kf s that makes this possible. Our algorithms do not depend on access to a regression oracle (cf.,~\cite{garg2024oracle}).

Finally, we extend our results to quantile regression and high-dimensional regression, which will be of general interest in forecasting, and we examine the relationship of {\em offline} kernel methods with previous results in batch outcome indistinguishability.  In the offline setting, \cite{hkrr,dwork2021outcome} showed equivalence of weak agnostic learning and outcome indistinguishability.  When the comparator class is contained in a reproducing kernel Hilbert space whose corresponding kernel function is efficiently computable, this learning problem has an efficient solution. This yields efficient methods for finding outcome-indistinguishable predictors in both the batch and online cases, even in settings where the distinguisher class is infinite.

\paragraph{Relation to the graph prediction literature.}
A great deal of research addresses link formation, typically in the batch setting, in which a subset of edges are presented as training data; see, for example, the book~\cite{hamilton2020graph}. 
A few papers have also considered prediction on \textit{evolving} graphs \cite{kumar2019, trivedi2019, ma2020, rossi2020, yu2023}.
Graph machine learning is a very active area of research with many research directions left unexplored \cite{morris2024}.
These approaches tend to focus on specific representations of graphs, which may be tailored to the semantics of nodes and edges.  Our approach differs in two main respects: first, we consider the online case in which the graph is evolving over time; at any given time step the algorithm may be given a pair of vertices $(i,j)$ and the goal is to predict whether an edge will form between them at the given time. Secondly, inspired by the observation that online calibrated forecasting can be achieved by {\em backcasting}~\cite{foster2021forecast}, we take a more formal approach, {\em ignoring} the semantics of the nodes and edges.  The semantics are introduced via the class of distinguishers.

\paragraph{Comparison with previous work in algorithmic fairness.}
We postpone detailed comparison to previous work in multicalibration, outcome indistinguishability and omniprediction to Sections~\ref{sec: OI} and~\ref{sec: omni} respectively.  
Connections between outcome-indistinguishable simple edge prediction and forms of graph regularity were investigated in~\cite{dwork2023pseudorandomness}.  
Our algorithm is the first online $\cO(\sqrt{T})$ omnipredictor that can compete with infinite or real-valued comparison classes $\cH$.  
Our results are non-asymptotic (\textit{i.e.}, hold for all $T$), and the constants hidden in the big-$\cO$ are usually small.
Unlike previous online 
algorithms, we require neither a regression oracle for omniprediction~\cite{garg2024oracle} nor explicit enumeration over all distinguishers for outcome indistinguishability \cite{gupta2022online}.  Unlike our work, \cite{garg2024oracle} offers the stronger guarantee of {\em swap} omniprediction (see Section~\ref{sec: omni}).  Finally, our bound for outcome indistinguishability error may deteriorate by a factor of $m$ for RKHSs that contain $m$ \textit{arbitrary} Boolean-valued functions, such as (pairs of) arbitrary demographic group memberships; for the other real-valued function classes mentioned above and in Section~\ref{sec: link prediction}, we pay no such price.

\paragraph{Paper organization.} The remainder of this paper is organized as follows.  Section~\ref{sec: link prediction} gives a full formulation of the \fair{} link prediction problem.  Section~\ref{sec: OI} introduces our main algorithm and results for online outcome indistinguishability.
Our results on omniprediction appear in Section~\ref{sec: omni}. Additional miscellaneous results are derived in Section~\ref{sec: other kernel functions}.

\subsection{Overview of technical results.}

Our work has two main sets of technical results. The first set concerns online outcome indistinguishability and the second set concerns efficient, $\sqrt{T}$, online omniprediction. In both cases, we focus on developing machinery for online prediction that we later specialize to link prediction. As a byproduct of these investigations, we also arrived at new results for online quantile and vector regression, as well as kernel batch algorithms and notions of distance to multicalibration that are of independent interest. 

\paragraph{Online outcome indistinguishability~\cite{dwork2021outcome}.} The technical starting point of our paper is a result by Vovk \cite{vovk2007non} which guarantees online outcome indistinguishability with respect to specific classes of functions $\cF$ that form an RKHS, or reproducing kernel Hilbert space. We review both of these concepts below. 

An algorithm guarantees online outcome indistinguishability with respect to a class $\cF \subseteq \{\cX \times [0,1] \rightarrow \R\}$ of {\em distinguishers}
if it is guaranteed to generate a sequence of predictions $p_t$ satisfying the following guarantee:
\begin{align*}
  \left| \sum_{t=1}^T \E  (y_t- p_t) f(x_t, p_t) \right|  \leq \littleo(T) \text{ for all } f\in \cF.
\end{align*}
Here, $(x_t, y_t)$ are an arbitrary sequence of (feature, outcome) pairs in $\in \cX \times \{0,1\}$, which can be chosen adversarially and adaptively, and the expectation is taken over the internal randomness of the algorithm. Notably, $y_t$ can be chosen with knowledge of the entire history $\{(x_{t'},p_{t'},y_{t'})\}_{t'=1}^{t-1}$, and may depend on $x_t$ and in some cases $p_t$ (see \Cref{sec: link prediction} for details).

In other words, a sequence of predictions is outcome-indistinguishable if no distinguisher in $\cF$ can reliably (with constant advantage) tell the difference between outcomes drawn according to the learner's predictions  $p_t$, and the true outcomes $y_t$ (see \Cref{subsubsec:oi_def} for further discussion). 

\paragraph{RKHSs, the $\kalg$ algorithm, and the \algname{}.} A reproducing kernel Hilbert space (RKHS) $\cF \subset \{\cX \rightarrow \R\}$ is a class of functions that can be defined over arbitrary domains (\textit{e.g.}, graphs).
Functions in an RKHS have the property that they can be \textit{implicitly} represented by a kernel function $k:\cX \times \cX \rightarrow \R$. 
Indeed, each kernel $k$ represents a unique RKHS $\cF_k$.\footnote{Common classes of functions like linear functions or polynomials are an RKHS, but we will see many others.}

The kernel representation enables one to design computationally efficient learning algorithms with guarantees that hold over all functions in the RKHS $\cF$, without necessarily having to explicitly solve a search problem over $f \in \cF$ (\textit{e.g.}, weak agnostic learning). The efficiency of learning over $\cF$ reduces to efficient evaluation of the kernel $k$. 
In addition to their computational benefits, RKHSs can be very expressive. By carefully designing the kernel function $k$, one can guarantee that the corresponding RKHS of functions $\cF_k$ contains specific classes of distinguishers of interest.\footnote{See \Cref{sec: OI} for a overview of RKHS and formal definition of norms in these spaces. Briefly, an RKHS is a Hilbert space and hence has an inner product $\langle \cdot, \cdot \rangle_{\cF} \rightarrow \R$. This inner product defines a norm $\fnorm{f}^2 = \langle f, f \rangle_{\cF}$ which serves a complexity measure for functions $f$ in the space $\cF$.}

Building on the work of Vovk \cite{vovk2007non} and insights from \cite{foster2021forecast}, we introduce the Any Kernel algorithm, which guarantees online indistinguishability with respect to any RKHS $\cF$. The algorithm is hyperparameter free, and runs in polynomial time whenever the kernel $k$ is bounded and efficiently computable. We summarize its main guarantees below.
\begin{theorem}[Informal]
Let $k$ be any kernel function and let $\cF$ be its associated RKHS. Then, the \algname{} generates a sequence of predictions $p_t\sim \Delta_t$ such that for any $f \in \cF$:
\begin{align*}
  \left|  \sum_{t=1}^T \E_{p_t}(y_t - p_t) f(x_t, p_t) \right|  \leq \fnorm{f} \sqrt{1 + \E_{p_t} \sum_{t=1}^T p_t(1-p_t) k((x_t, p_t),(x_t,p_t))} \leq B \cdot \fnorm{f} \sqrt{T}
\end{align*}
The second inequality holds if $k((x_t, p_t),(x_t,p_t)) \leq B^2$ for all $t$.
Here, $\fnorm{f}$ is the norm of $f$ in $\cF$ and the expectations are taken over the distributions $\Delta_t$ produced by the algorithm.
\end{theorem}

The proof of the theorem above draws heavily on the ideas from the literature on game-theoretic statistics \cite{shafer2005probability}, defensive forecasting \cite{vovk2005defensive}, and forecast hedging \cite{foster2021forecast}.
The \algname{} extends Vovk's $\kalg$ algorithm~\cite{vovk2007non} so as to work for \emph{any} kernel $k$ and correspondingly any RKHS $\cF$. More specifically, $\kalg$ requires the kernel $k$ to be continuous in the prediction $p$ and hence can only guarantee indistinguishability with respect to functions $f:\cX \times [0,1] \rightarrow \R$ that are continuous in $p$.\footnote{In our analysis, it helps to distinguish between the set of features $\cX$ and the predictions $p \in [0,1]$.} Removing this restriction enables us to consider binary distinguishers or tests that are not continuous in $p$. These were the central focus of the initial work on outcome indistinguishability \cite{dwork2021outcome} and multicalibration \cite{hkrr}.

To operationalize this result and guarantee indistinguishability with respect to a pre-specified collection of functions $\cF'$, there are two main sets of technical challenges. First, we need to understand how the choice of kernel $k$ relates to its corresponding RKHS $\cF_k$ so that we can guarantee that $\cF' \subseteq \cF_k$. Second, we need to pay special attention to ensure that the kernel can be computed efficiently, has bounded values $k((x,p),(x,p)) \leq \cO(1)$, and that the functions $f' \in \cF'$ have bounded norm in the RKHS $\cF_k$ ($\| f'\|_{\cF_k}$ is bounded). 

\begin{figure}[t!]
\begin{boxedminipage}{\textwidth}
\begin{center}
\vspace{2pt}
{\centering{\underline{The \algname}}}
\end{center}
\textbf{Input:} A kernel $k: (\cX \times [0,1])^2 \rightarrow \R$ \\

For $t=1, 2, \dots$
\begin{enumerate}
    \item Given history $\{(x_i, p_i, y_i)\}_{i=1}^{t-1}$ and current features $x_t$ define $S_t \; : \; [0, 1] \to \R$ as $$S_t(p) \defeq \sum_{i=1}^{t-1} k((x_t, p), (x_i, p_i))(y_i -p_i) + \frac{1}{2} k((x_t, p), (x_t, p)) (1-2p).$$
    \item If $\sign \,S_t(0) = \sign \,S_t(1) \neq 0$, \textbf{return} $\Delta_t = p_t = \frac{1}{2}(1 + \sign\,S_t(0))$.
    \item Define $\eps_t  = 1 / (10 t^3 B_t)$ for $B_t= \max_{t'\leq t} k((x_{t'},p_{t'}), (x_{t'},p_{t'}))$. If $k$ is continuous in $p$: 
    \begin{itemize}[$\bullet$]  
            \item Run binary search to find $p_t \in [0, 1]$ such that $|S_t(p_t)| \leq \eps_t$; \textbf{return} $\Delta_t = p_t$ w.p. 1 
    \end{itemize}
    \item Else, if $k$ is not continuous in $p$: 
    \begin{itemize}[$\bullet$]
        \item Run binary search to find $q, q' \in [0, 1]$ with $0 \leq |q'_t - q_t| \leq  \eps_t$ and $\sign\,q_t = \sign\,S_t(0)$ and $\sign\,q_t' = \sign\,S_t(1)$.
        \item \textbf{return} \[\Delta_t = \begin{cases}
            q_t &\text{with probability } \tau\\
            q_t' &\text{with probability } 1 - \tau.
        \end{cases} \quad \text{ for } \tau = \frac{|S_t(q_t')|}{|S_t(q_t)|+|S_t(q_t')|} \in [0, 1]\]
    \end{itemize}    
\end{enumerate}
\vspace{2pt}
\end{boxedminipage}
\caption{Pseudocode for the \algname{}. Steps 1-3 are as in \cite{vovk2007non}. Step 4 is inspired by \cite{foster2021forecast}. In each iteration, solve the binary search problems in steps 3 or 4 using at most $\log(1/\eps_t)$ oracle evaluations of $S_t$. Each evaluation of $S_t$ requires $t$ evaluations of the kernel $k$, hence the runtime at round $t$ is $\widetilde{\cO}(t \cdot \mathsf{time}_k )$. If $k$ is forecast-continuous $\Delta_t$ is just a point mass at $p_t$. Otherwise, $\Delta_t$ is near deterministic: it is supported on just 2 points $q_t,q_t'$ which are very close together, $|q - q'|\leq \cO(t^{-3})$. See \cref{thm:indistinguishability_main} for formal guarantees.}
\label{fig:online_protocol_randomized}
\end{figure}

Our results on online outcome indistinguishability directly address these core issues. Building on the rich literature on RKHS, we specialize our results to the link prediction problem and design efficient, bounded kernels whose RKHS contain interesting distinguishers $f$ on graphs. 
These in particular include powerful predictors such as deep (graph) neural networks. 

\begin{proposition}[Informal]
Consider the link prediction problem where $x_t$ consists of a pairs of individuals $(i_t,j_t)$ and a graph $G_t$. For each of the following classes of functions $\cF'$, there exists a computationally efficient and bounded kernel whose corresponding RKHS $\cF_k$ contains $\cF'$:

\begin{enumerate}
    \item All pairs of demographic groups. $\cF'$ consists of distinguishers which examine whether the pair $(i,j)$ belong to any pair of demographic groups from a finite list. 
    \item Number of connections and isomorphism classes. $\cF'$ consists of tests that examine the number of mutual connections between the pair $(i_t,j_t)$, or the isomorphism class of their local neighborhoods.
    \item An arbitrary pre-specified set of bounded functions. $\cF'$ is a finite benchmark class of deep learning based link predictors (\textit{e.g.}, graph neural networks), or any other bounded function.
\end{enumerate}

Furthermore, the norms of $f' \in \cF'$ in the corresponding RKHS $\cF_k$ are all $\cO(1)$ in each setting. Therefore, the \algname{} instantiated with these kernels guarantees online indistinguishability with respect to any of the $\cF'$ above with indistinguishability error bounded by $\cO(\sqrt{T})$.\footnote{The functions $f'$ in these constructions can additionally depend on the prediction $p$. For instance, by letting $f'$ examine whether predictions belong to a particular bin $[a,b] \subseteq [0,1]$.} 
\end{proposition}

While developed for the link prediction problem, the guarantees of the \algname{} hold for general domains and can also be used to generate indistiguishability with respect to other interesting classes of functions such as low degree polynomials over the Boolean hypercube (see \Cref{prop:low-deg-booleans}). Furthermore, by leveraging composition properties of kernels, we can also guarantee predictions which are indistiguishable with respect to sums or products of tests in different RKHSs. This in particular implies indistinguishability with respect to practically important predictors like random forests or gradient boosted decision trees.

\paragraph{Online omniprediction results.} 
While the first set of results focused on algorithms that guaranteed valid \emph{predictions} $p_t$, our second set of results pertain to the design of algorithms that lead to useful \emph{decisions} $\yhat_t$.\footnote{Note that $\yhat_t$ need not be of the same type as $y_t$; for example, the first might be any value in $[0,1]$ while the second might be Boolean.}
Assuming that the learner's utility over data $(x_t, \yhat_t, y_t)$ is captured by a loss function $\ell$, we aim to achieve lower average loss than functions in a benchmark class $\cH$:\footnote{Unlike previous work on omniprediction, we allow losses to depend on $x$. See \Cref{util} for detailed discussion of this point.}
\begin{align}
\label{eq:loss_min_informal}
    \frac{1}{T} \sum_{t=1}^T \ell(x_t, \yhat_t, y_t) \leq \inf_{h \in \cH} \frac{1}{T} \sum_{t=1}^T \ell(x_t, h(x_t), y_t) +  o(1).
\end{align}

In the link prediction context, predictions have the added advantage that they are likely \emph{performative} \cite{perdomo2020performative}. By informing downstream decisions, such as the link recommendations made to a user, predictions don't just forecast the future: they actively shape the likelihood of edge formation. This means that platforms are likely to experiment with the choice of loss function $\ell$. They may choose losses favoring predictions to match outcomes, \textit{e.g.}, squared loss $(\yhat-y)^2$, or ``loss'' functions that favor specific outcomes over others, like link formation $1-y$.

Given the diversity of plausible goals, we design online algorithms that generate predictions which can be post-processed to produce good decisions for a wide variety of losses. Importantly, each individual loss may correspond to a different high level objective (forecasting vs. steering). In particular, we generate algorithms which satisfy the following omniprediction definition.

Let $\cH$ be a benchmark class of functions and $\cL$ be a class of losses. An algorithm $\cA$ is an $(\cL, \cH, \Reg(T))$-online omnipredictor if it generates predictions $p_t$ such that for all losses $\ell \in \cL$, 
    \begin{align}
      \sum_{t=1}^T \ell(x_t, \pi_\ell(x_t, p_t), y_t) \leq \inf_{h \in \cH} \sum_{t=1}^T\ell(x_t, h(x_t), y_t) + \Reg(T). \label{eq:defomnipredictor}
    \end{align}
Here, $\pi_\ell(x, p) \in \argmin_{\hat{y}} p \cdot \ell(x, \hat{y}, 1) + (1-p) \cdot \ell(x, \hat{y} , 0)$ (the $\argmin$ may not be unique) and $\Reg: \cN \rightarrow \R_{\geq0}$ is $o(T)$. We refer to $\Reg$ as the regret bound for the algorithm $\cA$. Since it is sublinear in $T$, if we divide through by $T$, an online omnipredictor is guaranteed to achieve \Cref{eq:loss_min_informal} not just for a specific loss, but for any loss $\ell \in \cL$.

Conceptually, our technical approach for online omniprediction is most closely related to the work by \cite{loss_oi} which illustrates a connection between outcome indistinguishability and omniprediction in the batch setting. They show how given a set of losses $\cL$ and a function class $\cH$, one can construct a class of distinguishers $\cF$ (that depends on $\cL$ and $\cH$) such that any predictor that is indistinguishable with respect to $\cF$ is also a $(\cL, \cH)$-omnipredictor. Therefore, omniprediction reduces to outcome indistinguishability. 

We prove a similar reduction in the online setting. Moreover, we illustrate how one can leverage the \algname{} and RKHS machinery we developed previously in order to provably achieve the necessary indistinguishability guarantees in a computationally efficient manner. Taken together, we achieve unconditionally efficient (vanilla) online omnipredictors with $\sqrt{T}$ regret for common losses $\cL$ and rich (infinite, real-valued) comparator classes $\cH$. We now give a brief overview of the main ingredients that go into the proof of this result.

First, as in \cite{loss_oi} and \cite{kim2023making}, we show that algorithms which satisfy certain decision and hypothesis outcome indistinguishability conditions (OI) are also omnipredictors. Given a comparator class $\cH$ and set of losses $\cL$, we say that an algorithm $\cA$ satisfies {\em online hypothesis OI} if it generates a sequence of predictions that are outcome indistinguishable with respect to the following class of functions,
\begin{align}
\cF_{HOI}(\cL, \cH)  = \{\partial \ell(x, h(x_t)): \ell \in \cL, h \in \cH \} \text { where } \partial \ell(x, \yhat) = \ell(x,\yhat, 1) - \ell(x, \yhat, 0).
\end{align}
Similarly, we say that a online algorithm satisfies online decision OI if it is outcome indistinguishable with respect to the following class of tests:
\begin{align}
    \cF_{DOI}(\cL) = \{\partial \ell(x, \pi_\ell(x)): \ell \in \cL\} \text { where } \pi_\ell(x, p) = \argmin_{\hat{y}} \E_{\yt \sim \Ber(p)}\ell(x, \hat{y}, \yt). 
\end{align}
Using these definitions, we prove the following lemma.

\begin{lemma}[Informal]
Let $\cL$ be a class of loss functions and $\cH$ be a comparator class. If $\cA$ is online outcome indistinguishable with respect to the union of $\cF_{DOI}(\cL)$ and $\cF_{HOI}(\cL, \cH)$ with indistinguishability error bounded by $\Reg$, then $\cA$ is an online omnipredictor with regret rate $\cO(\Reg)$.
\end{lemma}

While it is interesting that this relationship, first identified in \cite{loss_oi}, carries over to the online setting, it is not quite useful without also knowing that the necessary indistinguishability requirements are also efficiently achievable. The main technical contributions of our work towards establishing online omniprediction is the design of efficiently computable kernel functions whose corresponding RKHSs contain the requisite distinguishers for hypothesis and decision OI. 

We defer a detailed presentation of these constructions to \cref{sec: omni}. 
However, the main technical ideas behind these results rely heavily on the theory behind reproducing kernel Hilbert space and the fact that it is relatively simple to compose kernel functions together. This ease of composition also allows one to characterize their corresponding (composed) function spaces. Being able to reason about composition is fundamental to these constructions since decision and hypothesis OI are both defined in terms of composition of functions (\textit{i.e.}, $\partial \ell(x, \pi_\ell(p)),~ \partial \ell(x, h(x))$). 
A technical challenge of our work is showing how certain RKHS remain closed under post-processing. In particular, as a stepping stone to proving the necessary decision OI guarantees, we identify natural conditions on RKHSs $\cF$ which guarantee that if $\ell(x,p,y)$ is in $\cF$ then so is $\ell(x, \pi_\ell(p),y)$.

Our results can be used to guarantee $\sqrt{T}$ online omniprediction with respect to various different kinds of comparator classes $\cH$ and losses $\cL$. However, in the following theorem we instantiate this general recipe to provide an end to end guarantee for classes $\cH$ and $\cL$ that are commonly considered in the literature.  We refer the reader to Section~\ref{sec: omni} for further examples.

\begin{theorem}[Informal]
There exist an efficient kernel $k$, such that the \algname{} instantiated with kernel $k$ is a ($\cH, \cL, \cO(\sqrt{T})$)-online  omnipredictor for the following settings, 
\begin{itemize}[$\bullet$]
    \item The comparator class $\cH$ contains all low-depth regression trees taking values in $[-1,1]$ and all functions $h'$ in a pre-specified finite set $\cH'$.
    \item The set of losses $\cL$ is any smooth, proper scoring rule\footnote{Proper scoring rules $\ell$ are those which are optimized by reporting the true likelihood of outcome. That is, if $y\sim \Ber(p)$, then $p$ is a minimizer of this expectation, $\E_{y \sim \Ber(p)}\ell(x,\yhat, y)$.}, loss function that is strongly convex in $\yhat$, or an arbitrary bounded loss $\ell'$ in a pre-specified finite collection $\cL'$.
\end{itemize}
In the link prediction context, one can in particular choose losses mapping onto the utility of a range of different decisions, including predictive performance (\textit{e.g.}, $\ell(x, \yhat, y) = (\yhat - y)^2$) and desirability of outcomes (\textit{e.g.}, $\ell(x, \yhat,y) = 1- y$ if the goal is link formation)\footnote{Losses like $1-y$ make sense in settings where the learner's predictions $\yhat$ actively change the likelihood of the outcome $y$ (for instance, by influencing the platforms recommendation decisions).}.

Loss functions may also be feature-dependent, like losses that more heavily weight decisions that affect a pair of individuals from different demographic groups or for which the induced subgraph on a pair of individuals has a certain structure (like having $c \in \N$ neighbors in common).
\end{theorem}

This result pushes the boundary of what is achievable in terms of online omniprediction in several ways. First, to the best of our knowledge, it is the first $\sqrt{T}$ online omniprediction guarantee which holds for comparison classes $\cH$ that are real-valued, or of infinite size (there are infinitely many low-depth regression trees). Second, the statements are unconditional. The computational efficiency of our algorithm does not rely on the existence of an online regression oracle for the class $\cH$. 

Furthermore, we can include any function $h': \cX \rightarrow [-1,1]$ in the class $\cH$. 
In the context of link prediction, this implies that the algorithm can compete with any bespoke comparison function that a platform may already be using (\textit{e.g.}, deep network). 
Furthermore, as we mentioned previously, these results hold even for the performative case where the outcomes $y_t$ depend the near-deterministic distribution $\Delta_t$ from which the predictions are sampled from. For the reader familiar with the performative prediction literature, this guarantee is best understood as a novel form of online performative \emph{stability}. It does not quite imply performative \emph{optimality} or performative omniprediction as in \cite{kim2023making}. See \Cref{subsec:performativity} for more details.

\paragraph{Other results.} As a serendipitous consequence of our investigation into kernel methods for online indistinguishability and omniprediction, we obtain algorithms for other online prediction problems. These are not directly related to the link prediction problem which is our main focus, but are of independent interest.

We design a new algorithm for online multicalibrated quantile regression. In quantile regression, outcomes $y$ are real-valued instead of binary. Given a quantile $q \in [0,1]$, the goal is output a prediction $p$ such that $y \in \R$ is less than $p \in \R$ exactly a $q$ fraction of the time. In the batch setting where $(x,y) \sim \cD$, one aims to find a predictor $h$ that minimizes the error: $$|\Pr_{(x,y) \sim \cD}[y \leq h(x)] - q|.$$

Quantile regression is a common problem in domains like weather forecasting or financial prediction, where one is interested in deriving confidence intervals or predicting the likely range of outcomes, rather than the average outcome. In \Cref{subsec:quantile_regression}, we introduce a new online algorithm, the \qalgname{}, which satisfies the following guarantee for the online setting where
``Real Life'' draws (real-valued) outcomes $y_t\sim o_t$ from a different distribution $o_t$ at every time step:
\begin{align*}
    \sum_{t=1}^T \E_{p_t \sim \Delta_t, y_t \sim o_t}[(\1\{y_t \leq p_t\} - q) f(x_t,p_t)] \leq \|f\| \sqrt{T} \text{ for all } f\in \cF
\end{align*}
Like the \algname{}, the \qalgname{} works for any RKHS $\cF$ and runs in polynomial time whenever the associated kernel $k$ is efficiently computable. 
Furthermore, using our previous results relating kernels $k$ to their corresponding RKHSs $\cF_k$, one can instantiate the algorithm to guarantee online quantile multicalibration with respect to common real-valued functions $\cF$. These results complement those in \cite{garg2024oracle} and \cite{roth2022uncertain} since the functions $f$ can now be real-valued, the set $\cF$ can be of infinite size, and the algorithm does not depend on enumeration over $\cF$ or access to a computational oracle. 

In addition to quantiles, one can also extend the algorithm to high dimensional regression, where $y$ is now a vector in a compact set $\cY \subseteq \R^d$ instead of a scalar in $\R$. Drawing on the theory of matrix valued kernels \cite{vectorRKHSreview,micchelli2005learning}, we introduce the $\valgname$ which satisfies the following guarantee for any vector valued RKHS $\cF \subseteq \{\cX \times \cY \rightarrow \cY\}$, 
\begin{align*}
    \sum_{t=1}^T (y_t - p_t)^\top f(x_t, p_t) \leq \fnorm{f} \sqrt{T}.
\end{align*}
The computational efficiency of the $\valgname$ relies on the ability to solve a \emph{variational inequality}. These have been the subject of intense study within the optimization literature and efficient algorithms exist for various common choices of matrix valued kernels.

Beyond these contributions, and inspired by the recent works by \cite{qiao2024distance, blasiok2023unifying} we also initiate the study of distance to \emph{multi}calibration (previous work addresses distance to simple calibration) and analyze how straightforward instantiations of the $\algname$ can be used to generate predictions that satisfy small distance to multicalibration in the online setting. 

Lastly, we observe that any function class that is an RKHS with an efficient kernel also admits a weak agnostic learner (WAL). This connection implies that any multicalibration algorithm that relied on an oracle WAL for a class $\cF$ is unconditionally efficient for the case where $\cF$ is an RKHS. 

\section{The Link Prediction Problem}
\label{sec: link prediction}

\paragraph{Data.} We represent a professional network as a graph $G_t$ consisting of nodes (people) and edges (connections between people) that evolve over time. Each node $i$ is associated with a features $z_{i,t}$ containing information that pertains specifically to $i$, such as their employment and demographic information. This can vary over time. In addition to this node-level information, the graph $G_t$ is defined by a set of undirected edges detailing which individuals are connected at time~$t$. 
Edges can be added to or removed from the graph arbitrarily at every time step and need not follow any predefined dynamic or process such as triadic closure \cite{simmel1908soziologie}. The underlying set of nodes can also change. 
The only restriction we will make is that the platform has the ability to \emph{observe} the entire graph $G_t$ as it evolves over time.\footnote{While the platform has the ability to examine all of $G_t$, algorithms need not read the entire input $G_t$. They only examine the subset of $G_t$ relevant to the distinguishers.}

\paragraph{Prediction protocol.} At every time step $t$, the platform is presented with a pair of individuals 
$a_t = (i,j)$ and generates a prediction $p_t$ regarding the likelihood that $i$ and $j$ will be connected at the next time step ($i$ and $j$ may or may not be connected at time $t$). After producing the prediction, the platform then observes a binary outcome $y_t$, which is 1 if $i$ and $j$ are connected at time $t+1$ and 0 otherwise. As per our earlier observability comment, the platform observes the outcome $y_t$ before having to make a prediction at time $t+1$.
Variants of this prediction problem were proposed as early as 2003 \cite{liben2003link}.

In our setting, we allow the outcome $y_t$ to also depend on the distribution $\Delta_t$ where $p_t$ is drawn from.\footnote{The difference between $y_t$ depending on the distribution $\Delta_t$ versus the draw $p_t\sim \Delta_t$ is relatively neglible since in all our algorithms, $\Delta_t$ is only ever supported on 2 points which are very close together. For intuition, one can essentially assume that Nature chooses $y_t$ while knowing $p_t$ up to some small rounding error.} That is, predictions can be \emph{performative} \cite{perdomo2020performative} and influence the likelihood of the outcome. This dynamic naturally occurs whenever the platform uses predictions to inform recommendations. For instance, a platform such as LinkedIn may opt to recommend that a pair of individuals connect via the ``People You May Know'' panel if $p_t$ is above some threshold. Forecasts in this setting are hence likely to be self-fulfilling (although our results hold for any dynamic).

\paragraph{Notation.} We denote by $\cZ$ the set of possible node-level features of an individual, at any point in time.
We define the graph $G_t$ to be a set $\{(v, z_{v,t}, \Gamma_t(v))\}_{v \in V_t}$, where $v \in \N$ is the id of a node, $z_{v,t} \in \cZ$ are the node-level features of $v$ at time $t$, and $\Gamma_t(v) \subseteq V_t$ is the set of nodes containing $v$ and its immediate neighbors at time~$t$. 
Here, $V_t \subseteq \N$ is the set of nodes present in the graph at time $t$. We will use $\Gamma_G^{(r)}(v)$ to denote the set of nodes that are at distance at most $r$ from $v$ in  $G$.
If the sequence of graphs $\{ G_t \}_{t=1}^T$ is clear from context, we will write $\Gamma_t(v) = \Gamma_{G_t}(v)$, and adopt the shorthands $\Gamma_G^{(1)}(v) = \Gamma_t(v)$ for $v$'s immediate neighborhood. 

Furthermore, we will (exclusively) use $\cU = (\N \times \N) \times \cG$ to refer to the universe of possible elements $u = (a, G)$ consisting of pairs of individuals $a = (u,v)$ and graphs $G \in \cG$. We will use $\cX$ to refer to a general set.

\subsection{Formal desiderata.}
\label{subsec:desiderata_link_prediction}
The dynamics underlying professional networking are complex. In this paper, we address the challenge of efficiently generating forecasts that are guaranteed to be \textit{a) valid} and \textit{b) useful}, without imposing any modeling assumptions regarding how networks evolve.

\subsubsection{Validity and outcome indistinguishability.} 
\label{subsubsec:oi_def}
Defining what it means for a forecast of arbitrary, non-repeatable events to be valid is in and of itself a challenging task. 
However, one common perspective within the sciences is that a theory, or prediction, is valid if it withstands efforts to falsify it. This viewpoint was recently formalized in the computer science literature by \cite{dwork2021outcome} who introduced the notion of \emph{outcome indistinguishability} (OI).
Briefly, a predictor is outcome indistinguishable if no analyst can refute the validity of the predictor on the basis of a particular set of computational tests.

This idea of the analyst is operationalized via a class $\cF_A$ of {\it distinguishers}  that take in a set of observation information $x$, a prediction $p$, a binary outcome $y$, and return a score (think True/False).\footnote{This corresponds sample-access OI, the second level in the OI hierarchy presented in \cite{dwork2021outcome}. For ease of presentation, we assume that all distinguishers $A$ are deterministic.} A sequence of predictions $p_t$ is outcome indistinguishable with respect to $\cF_A$ if, when averaged over the sequence, all distinguishers $A \in \cF_A$ give (approximately) the same output in the case where they are given $(a)$ the synthetic outcome $\yt_t \sim \Ber(p_t)$ sampled according to the learner's prediction $p_t$ and $(b)$ the true outcome $y_t$ revealed by "Real Life". That is, 
\begin{align}
\label{eq:informal_oi}
\frac{1}{T} \sum_{t =1}^T \E_{\yt \sim \Ber(p_t)}A(x_t, p_t, \yt_t)  \approx  \frac{1}{T} \sum_{t =1}^T A(x_t,p_t, y_t).
\end{align}

In their initial work, \cite{dwork2021outcome} focused on the batch, or distributional setting, where features are sampled from a fixed, \emph{static} distribution $x \sim \cD$, and outcomes $y$ are sampled from some conditional distribution, $y \sim \Ber(p^*(x))$. 
As discussed previously, networking dynamics are complex and the likelihood of a link forming between any pair of individuals changes as networks evolve. Assuming any kind of static, or slowly moving distribution over  $(x,y)$ is a non-starter for the link prediction problem.

Instead of generating predictions that are indistinguishable under a specific choice of static distribution, we tackle the challenge
of (efficiently) producing predictions that are outcome indistinguishable against arbitrary sequences $\{ (x_t, p_t, y_t) \}_{t=1}^T$. That is, ``Real Life''' can choose outcomes $y_t \in \{0,1\}$ arbitrarily, and the choice of $y_t$ may even depend on the learners predictions. 
Formally, we aim to generate link predictions that satisfy the following online outcome indistinguishability guarantee:

 \begin{definition}
    An algorithm $\cA$ is $(\cF, \Reg)$-online outcome indistinguishable if it generates a transcript $\{(x_t,\Delta_t, y_t)\}_{t=1}^T$ such that for all distinguishers $f \in \cF$ 
    \begin{align}
    \label{eq:onlineoidef}
      \bigg| \sum_{t=1}^T \E_{p_t \sim \Delta_t}(y_t - p_t)f(x_t, p_t) \bigg|\leq \Reg(T,f)
    \end{align}
    where the indistinguishability error rate $\Reg: \N \times \cF \rightarrow \R_{\geq0}$ is $o(T)$ for every $f$. 
\end{definition}

Although stated differently, the condition above is essentially equivalent to that presented in \Cref{eq:informal_oi} since,
\begin{align*}
   A(x_t,p_t, y_t) -  \E_{\yt_t \sim \Ber(p_t)}A(x_t, p_t, \yt_t) = (y_t -p_t)(A(x_t,p_t, 1)  - A(x_t, p_t,0)) = (y_t -p_t) f_A(x_t,p_t),
\end{align*}
for $f_A(x,p) = A(x,p, 1)  - A(x, p,0)$. Therefore,
\begin{align*}
    \bigg|\lim_{T\rightarrow \infty}\frac{1}{T} \sum_{t =1} \E_{\yt \sim \Ber(p_t)}A(x_t, p_t, \yt_t) - A(x_t,p_t, y_t) \bigg| =0 \iff        \bigg| 
 \lim_{T\rightarrow \infty}\frac{1}{T} \sum_{t=1}^T (y_t - p_t)f_A(x_t, p_t) \bigg| = 0.
\end{align*}
Although initially defined with respects functions $f$ that are binary valued --- where $f$ was the characteristic function of a set or demographic group \cite{hkrr} --- the distinction between binary and real-valued functions has since been blurred in the multicalibration literature. In this work, we keep to earlier conventions and refer to the above guarantee (\Cref{eq:onlineoidef}) as indistinguishability since we focus  mostly on real-valued $f$ and because we work with a formulation of omniprediction that is expressed in terms of outcome indistinguishability~\cite{loss_oi}. However, we do so with the understanding that both terms are very tightly linked. 

Returning to the intuition that predictions will be regarded as valid (for now!) if they cannot be falsified, we note that predictions satisfying \Cref{eq:onlineoidef} with $\Reg(T,f) = \cO(\sqrt{T})$ cannot be refuted on the basis of a common class of tests based on the theory of martingales. To see this, assume that the outcomes $y_t$ are the realizations of a stochastic process $(Y_t)_{t=1}^T$ where the binary random variables $Y_t$ are not necessarily independent nor identically distributed, but satisfy $\E Y_t = p^*_t$. Then, it's not hard to check that $Z_t = \sum_{t=1}^T Y_t- p^*_t$ is a martingale with bounded differences. 
By Azuma-Hoeffding, the best one can guarantee on the deviations $|\sum_{t=1}^T y_t- p^*_t|$ is that they scale at $\cO(\sqrt{T})$ rates. 
Therefore, a sequence of predictions $(p_t)_{t=1}^T$ that are OI with respect to the constant function $f=1$ and satisfy $|\sum_{t=1}^T Y_t- p_t| \leq \cO(\sqrt{T})$ behave \emph{as if} they were the true sequence $(p^*_t)_{t=1}^T$ that generate the data. We cannot refute them on the basis of these martingale tests. 

The above online OI guarantee is stronger, it holds not just on average over the sequence but even with respect to distinguishers that also examine information present in $x_t$ and the prediction $p_t$ itself. We will develop link prediction algorithms that fool distinguishers which examine a wide variety of information about the pair of individuals including their node-level features, their mutual connections, and the features of people to whom they are connected.

\subsubsection{Utility and omniprediction.} \label{util}
\label{subsubsec:omni_def}
In addition to the notion of empirical validity above, we aim to generate predictions that are \textit{useful} for decision-making. 
We will thus move beyond analysis of predictions $p_t$ and consider \textit{decisions} $\yhat_t$ made on the basis of a prediction $p_t$ and the relevant context $x_t$.

We will also assume that decision-makers' utilities can be specified by a (class of) \textit{loss function}(s). 
For example, decision-makers may want to forecast outcomes, so that predictions closely match outcomes, or steer them, so that desirable outcomes occur more often.
In such cases, a loss function will encode some notion of distance between predictions and outcomes. Or, it might simply produce higher outputs when outcomes are undesirable and lower outputs when they they are desirable.
As we noted previously, our ``platform'' setting allows for performativity, meaning that outcomes $y$ can depend on decisions $\yhat$ --- this is the power of the platform that we wish to exploit and what gives us hope that the latter goal of steering subjects towards desirable outcomes may be attainable.

We will focus on minimizing loss with respect to the best fixed action in retrospect:
An algorithm $\cA$ generating a transcript of (feature, decision, outcomes) tuples $\{ x_t, \hat y_t, y_t \}_{t=1}^T$ achieves $\Reg(T)$ regret with respect to a comparison, or benchmark, class of functions $\cH$ and loss $\ell$ if
\begin{align*}
    \sum_{t=1}^T \ell(x_t, \yhat_t, y_t) \leq \min_{h \in \cH} \sum_{t=1}^T \ell(x_t, h(x_t), y_t) + \Reg(T) .
\end{align*}

In the equation above, we note that loss functions can depend on \textit{features} $x_t$ as well as predicted and realized outcomes.
This is because many loss minimization settings in complex domains depend on {the object we are making predictions about} as well as on the prediction and realized outcome.
For example, one may wish to more heavily weight decisions that affect disadvantaged demographic groups, in which case the loss function will depend on the features of individuals.
However, one can always drop the $x$ argument to $\ell$ for losses that do not depend on features (as in in prior work on omniprediction \cite{loss_oi, garg2024oracle}).

In link prediction, a platform may want to determine which links are likely to form or make recommendations that nudge certain links towards forming.
The utility of a decision in an evolving network may also depend on characteristics of the decision subjects, such as the demographic group membership of the pair of individuals across a potential connection.
We allow for loss functions that take into account characteristics of pairs of individuals (and also their neighborhoods and neighbors' features).

Finally, we will focus on creating predictors that can be efficiently post-processed so as to minimize loss, with respect to a given comparator class, for any in \textit{large classes of loss functions}. These are called omnipredictors \cite{gopalan2022omnipredictors, gupta2022online}. 
    Online omnipredictors can be defined formally as follows.
    \begin{definition} \label{def:omnipredictor}
    An algorithm $\cA$ is an $(\cL, \cH, \Reg)$-online omnipredictor if it generates a transcript $\{(x_t, \Delta_t, y_t)\}_{t=1}^T$ such that for all $\ell \in \cL$ there exists a $\pi_\ell: \cX \times [0,1] \rightarrow [0,1]$ such that
    \begin{align}
        \sum_{t=1}^T \E_{p_t \sim \Delta_t}\ell(x_t, \pi_\ell(x_t, p_t), y_t) \leq \inf_{h \in \cH} \sum_{t=1}^T \ell(x_t, h(x_t), y_t) + \Reg(T). \label{eq:defomnipredictor}
    \end{align}
    where $\Reg: \cN \rightarrow \R_{\geq0}$ is $o(T)$.
    \end{definition}
    \noindent In particular, we will take $\pi_\ell$ to be 
    \begin{align*}
    	\pi_\ell(x, p) &\in \argmin_{\hat{y} \in [0, 1]} \E_{y \sim \mathrm{Ber}(p)} [\ell(x, \hat{y}, y)], \\
     &= \argmin_{\hat{y} \in [0, 1]} p \cdot \ell(x, \hat{y}, 1) + (1-p) \cdot \ell(x, \hat{y}, 0),
    \end{align*}
    which is a simple optimization problem over the unit interval that can be efficiently solved.
    (We will assume argmin returns the set of values achiving a minimum, and that $\pi_\ell$ is an arbitrary member of this set.) 
    Finally if $\ell$ is invariant to $x$, the $x$ argument to $\pi_\ell$ can also be dropped.
    %

    %
    We focus on omnipredictors for two reasons.
    First, link predictions may be used for a variety of downstream decisions on a platform.
    As mentioned previously, a class of loss functions can simultaneously be used to measure \textit{predictive quality} (\textit{e.g.}, squared loss: $\ell(x, \hat y, y) = (y - \hat y )^2$) or \textit{desirability of outcomes} (\textit{e.g.}, link formation: $\ell(x, \yhat, y) = 1-y$, which is minimized when an edge forms).
    Additionally, platforms may use link predictions within different ``People You May Know'' recommendations serving different goals (\textit{e.g.}, different types of connections), and they may hope to tailor other on-platform experiences on the basis of the predicted evolution of the network.
    Second, the loss function may not be known at prediction time: for example, a predictive system may need to be fixed in advance of A/B tests determining which loss function in a certain class gives the best proxy for some long-term objective.

In Section \ref{sec: omni}, we discuss learning algorithms which are omnipredictors with respect to large classes of losses (\textit{e.g.}, all bounded differentiable loss functions) and with expressive comparator classes, like deep neural nets.

\section{Online Outcome Indistinguishability and Applications to 
Link Prediction}
\label{sec: OI}

 In this section, we consider the first task detailed in \Cref{subsec:desiderata_link_prediction} of generating link predictions for an evolving network that satisfy the following outcome indistinguishability guarantee:
\begin{align*}
   \sum_{t=1}^T (p_t - y_t) f(x_t, p_t) \leq o(T)  \text{ for all } f \in \cF.
\end{align*}
We are specifically interested in designing online algorithms that are $(a)$ computationally-efficient, $(b)$ indistinguishable with respect to rich classes of functions $\cF$ defined on complex, graph-based domains $\cU$, and $(c)$ achieve the optimal $\cO(\sqrt{T})$ outcome indistinguishability error, henceforth {\em OI error}.

We present a more detailed comparison to prior work later on. However, briefly, previous online algorithms for this problem which achieved the optimal $\sqrt{T}$ OI error bound were either computationally inefficient for super polynomially sized sets $\cF$ \cite{foster2006calibration, gupta2022online}, could only achieve the above guarantee for restricted classes of functions $f$ that were continuous in the forecast $p$ \cite{vovk2007non}, or which where binary valued \cite{gupta2022online}. Our algorithm overcomes these issues and achieves all three of the above desiderata.  This will enable new possibilities for \emph{omniprediction} as we detail in \Cref{sec: omni}, accomplished by appropriate choice of the \kf , folding the benchmark functions into the corresponding RKHS $\cF$.

\paragraph{Technical approach.} We develop new, general-purpose algorithms guaranteeing online outcome indistinguishability and then specialize them to the link prediction setting.
In particular, we focus on developing algorithms which guarantee calibration with respect to sets $\cF$ that form a \emph{reproducing kernel Hilbert space} (RKHS). Intuitively, an RKHS is a set of functions $\cF \subseteq \{\cX \rightarrow \R\}$ that are implicitly represented by a kernel function $k: \cX \times \cX \rightarrow \R$, for a universe $\cX$.

This kernel based viewpoint is  useful for our link prediction problem because it provides a computationally efficient way to guarantee calibration with respect to rich classes of functions defined on graphs. 
Building on the theory of RKHSs, we design computationally efficient kernels that guarantee indistinguishability with respect to classes of distinguishers that take into account graph topology (\textit{e.g.}, number of mutual connections, isomorphism class of the local neighborhoods), or functions computable by arbitrary finite sets of pre-specified functions, like graph neural network link predictors.

Our technical approach is directly builds on a result by Vovk \cite{vovk2007non} that is in turn inspired by the breakthrough work of \cite{foster1998asymptotic}. 
In his paper, which predates the definition of multicalibration by \cite{hkrr} or OI \cite{dwork2021outcome}, Vovk introduces an algorithm that guarantees indistinguishability with respect to any RKHS of functions $f(u,p)$ that are continuous in $p$. 
Drawing on ideas from \cite{foster2021forecast}, we introduce the \algname{}, which guarantees indistiguishability with respect to \emph{any} RKHS $\cF$, not just those that are continuous in $p$. 

\subsection{The algorithm.}
\label{sec: the algorithm}

We now formally present our online \algname{}, which forms the backbone of our later results. The algorithm builds on the earlier $\kalg$ algorithm from \cite{vovk2007non} that is in turn inspired by Kolmogorov's 1929 proof of the weak law of large numbers \cite{kolmogorov1929loi}.
The reader familiar with reproducing kernel Hilbert spaces can skip the brief background highlights outlined below.

\paragraph{Background on reproducing kernel Hilbert spaces.} 
Our guarantees are stated in terms of a kernel $k$ and its associated reproducing kernel Hilbert space $\cF_k$. We drop the subscript when it is clear from context. We briefly review the basic facts behind RKHSs here and provide a self-contained formal review of the facts we need.  In \cref{sec:rkhs_background}, we list out various kernels and RKHS that we then use to instantiate the algorithm. We refer the reader to texts such as \cite{paulsen2016introduction, steinwart2008support} for further background on this material.

\begin{definition}
Let $\cX$ be an arbitrary set. A function $k : \cX \times \cX \to \R$ is a \emph{kernel} on $\cX$ if it satisfies 
\begin{enumerate}
    \item Symmetry: $k(x,x') = k(x',x)$ for all $x,x' \in \cX$.
    \item Positive Definiteness: $\sum_{i=1}^n \sum_{j=1}^n \lambda_i\lambda_j k(x_i, x_j) \geq 0$ for all $n \in \N$, $x_1, \ldots, x_n \in \cX$ and $\lambda \in \R^n$.
\end{enumerate}
\end{definition}

Every kernel $k$ is associated with a unique Hilbert space $\cF \subseteq \{\cX \rightarrow \R \}$ of real-valued functions. By virtue of being a Hilbert space, $\cF$ is equipped with an inner product $\langle \cdot, \cdot \rangle_{\cF}: \cF \times \cF \rightarrow \R$ that defines a norm on the elements $f \in \cF$, $\|f \|_{\cF}^2 = \langle f,f\rangle_{\cF}$. The set is called a \emph{reproducing} kernel Hilbert space since for every element $x \in \cX$, there exists an element $\Phi(x) \in \cF$ such that 
\begin{align*}
    f(x) = \langle f, \Phi(x) \rangle_{\cF} \text{ for all } f\in \cF,
\end{align*}
where $\langle \cdot, \Phi(x) \rangle_{\cF}$ is continuous.
The function $\Phi: \cX \rightarrow \cF$ is called the reproducing kernel or feature map. It also satisfies the property that for all $x,x' \in \cX$,
\begin{align*}
k(x,x') = \langle \Phi(x), \Phi(x') \rangle_{\cF}.
\end{align*}
Given any kernel $k$, or equivalently a feature map $\Phi$, the Moore-Aronszajn theorem provides an explicit characterization of the set of functions $\cF$. In particular, 
\begin{align*}
    \cF = \overline{\mathsf{span}} \{\Phi(x): x \in \cX\}, 
\end{align*}
where,
\begin{align*}
    {\mathsf{span}} \{\Phi(x): x \in \cX\} = \bigg\{f \; : \; f = \sum_{i=1}^n \lambda_i \Phi(x_i) \text{ for all } n \in \N, x_1, \dots, x_n \in \cU \text{ and } \lambda \in \R^n \bigg\},
\end{align*}
and the overline denotes the completion of the set. 
That is, $\cF$ is the set of all finite linear combinations of feature maps $\Phi$ augmented with the limits of any Cauchy sequences of such linear combinations.

Throughout our work we will use the fact that kernels \emph{compose}. 
That is, if $k_1$ and $k_2$ are kernels for RKHSs $\cF_1 \subseteq \{\cX_1 \rightarrow \R\}$ and $\cF_2 \subseteq \{\cX_2 \rightarrow \R\}$. Then $k_1 + k_2$ is a kernel for $\cF_1 + \cF_2$ and $k_1 \cdot k_2$ is a kernel for $\cF_1 \cdot \cF_2$ where, 
\begin{align*}
    \cF_1 + \cF_2 &\subseteq \{f_1(x_1) + f_2(x_2) \; : \; x_1 \in \cX_1, x_2 \in \cX_2, f_1 \in \cF_1, f_2 \in \cF_2\}, \quad \text{and}\\
        \cF_1 \cdot \cF_2 &\subseteq \{f_1(x_1)f_2(x_2) \; : \; x_1 \in \cX_1, x_2 \in \cX_2, f_1 \in \cF_1, f_2 \in \cF_2\}.
\end{align*}
A direct implication of the first line is that two different RKHSs on the same domain can be combined to make a new one, where the set of functions in the RKHS contains the union of functions in each of the RKHSs.
Further details are deferred to \cref{lem:kernelsum} and \cref{lem:productkernel}. 
However, the key point is that these composition properties make it easy to ``mix and match'' various indistinguishability guarantees.

\paragraph{Description of algorithm.} The algorithm is at a high-level very simple. It only takes as input a kernel function $k$,
\begin{align*}
k : (\cX \times [0, 1]) \times (\cX \times [0, 1]) \to \R.
\end{align*}
At every round $t$, it constructs a function $S_t: [0,1] \rightarrow \R$ defined from the history $\{(x_i, p_i, y_i)\}_{j=1}^{t-1}$. 
If the kernel is continuous, it chooses a prediction $p_t$ that is a zero of $S_t$, $S_t(p_t) \approx 0 $. If the kernel $k$ is discontinuous in $p$, it instead finds two points $q_1$ and $q_2$ which are very close together (i.e., $|q_1 - q_2| \approx 0$) and outputs a distribution $\Delta_t$ supported on $q_1,q_2$ such that the expectation of $S_t$ over $\Delta_t$ is approximately 0. Both of these search problems are efficiently solved via binary search. The algorithm in which the kernel $k$ is continuous is the same as in Vovk's $\kalg$ algorithm, while the discontinuous case is new. In particular, the procedure in the discontinuous case draws on ideas from \cite{foster2021forecast} and their results on near deterministic calibration. 

\remove{
\jcpnote{I think this was out of place}

Summarizing: the key points are (1) the existence of the feature map $\Phi$; (2) the Moore-Aronsajn theorem, which is what lets us ensure that meaningful functions are in the RKHS $\cF$; and closure under addition and multiplication.

}

\paragraph{Guarantees of algorithm.} With these preliminaries out of the way, we now state the main guarantees of the theorem.

\begin{theorem}
\label{thm:indistinguishability_main}
 Let $k$ be a kernel with associated RKHS $\cF$. Then, the \algname{} (\Cref{fig:online_protocol_randomized}) instantiated with kernel $k$ generates a transcript $\{(x_t, \Delta_t, y_t)\}_{t=1}^T$ such that for any $f \in \cF$:
    \begin{align*}
         \left|\sum_{t=1}^T \E_{p_t \sim \Delta_t}f(x_t, p_t)(y_t - p_t) \right| \leq \|f\|_{\cF} \sqrt{1 + \sum_{t=1}^T \E_{p_t \sim \Delta_t} p_t (1-p_t) k((x_t, p_t), (x_t, p_t))}.
    \end{align*}
If $k$ is forecast-continuous, then the guarantee is deterministic since $\Delta_t$ is a point mass. Otherwise, it is near-deterministic. The distribution $\Delta_t$ is supported on points that are $\cO(t^{-3})$ apart.\footnote{One could change this from $\cO(t^{-3})$ to $\cO(t^{-\alpha})$ for any $\alpha > 3$ without changing the asymptotic runtime.}
If the kernel is bounded by $B$, $$\sup_{(x,p) \in \cX \times [0,1]} k((x,p),(x,p)) \leq B,$$ then the per round runtime of the algorithm is bounded by ${\cO}(t  \cdot  \log(t B) \cdot \mathsf{time}(k))$, where $\mathsf{time}(k)$ is a uniform upper bound on the runtime of computing the kernel function $k$.
\end{theorem}
\begin{proof}
    If $\sign\,S_t(0) = \sign\,S_t(1)\neq 0$ in round $t$, selecting $p_t = (1 + \sign\,S_t(0))/2$ guarantees that,  
    \begin{align*}
        S_t(p_t)(y_t - p_t)\leq 0,
    \end{align*}
regardless of whether $y_t$ is 1 or 0. Otherwise, $p_t \sim \Delta_t$ where $\Delta_t$ places probability $\tau$ on $q_t$ and $1-\tau$ on $q_t'$. In this case, letting $\tau' = 1-\tau$, we can write:
\begin{align*}
\E_{p_t \sim \Delta_t}[S_t(p_t)(y_t-p_t)] &= \tau S_t(q_t)(y_t - q_t) + (1-\tau)S_t(q_t')(y_t - q_t') \\
& = [\tau S_t(q_t) + \tau'S_t(q_t')](y_t-q_t') + \tau S(q_t)(q_t'-q_t) 
\end{align*}
By choice of $\tau = |S_t(q_t')|/(|S_t(q_t)|+|S_t(q_t')|)$, and the fact that $S_t(q_t')$ and $S_t(q_t)$ have opposite signs, the term inside the brackets is equal to 0 (this is the forecast hedging idea from \cite{foster2021forecast}). Summarizing, we have that:
\begin{align*}
\E_{p_t \sim \Delta_t}[S_t(p_t)(y_t-p_t)] = \tau S_t(q_t)(q_t'-q_t) \leq |S_t(q_t)||q_t - q_t'| \leq |q_t - q_t'| \cdot t \cdot \max_{t' \leq t} k((x_t,p_t), (x_t,p_t)).
\end{align*}
Since $|q_t - q_t'| \leq \eps_t = 1/ (10B_t t^3)$ where $B_t = \max_{t' \leq t} k((x_t,p_t), (x_t,p_t))$, we conclude that regardless of whether $y_t$ is 0 or 1, 
\begin{align}
\label{eq:forecast_hedging}
    \E_{p_t \sim \Delta_t}[S_t(p_t)(y_t-p_t)] \leq \frac{1}{10 t^2}.
\end{align}
We now seek an upper bound on the expected value of
\[
    \left\|\sum_{t=1}^T (y_t - p_t)\Phi(x_t, p_t)\right\|^2_{\cF} = \sum_{t=1}^T\sum_{s=1}^{T} (y_t-p_t)(y_s-p_s) \langle \Phi(x_t, p_t), \Phi(x_s, p_s)\rangle_{\cF}.
\]
To this end, first observe the symmetry of the summands in $(s, t)$, so the right side simplifies to
\[
    \sum_{t=1}^T (y_t-p_t)^2 \|\Phi(x_t, p_t)\|_{\cF}^2 + 2\sum_{t=1}^T(y_t-p_t)\left(\sum_{s=1}^{t-1} k((x_t, p_t), (x_s, p_s))(y_s-p_s)\right).\\
\]
Next, we apply the identity $(y_t - p_t)^2 = p_t(1-p_t) + (1 - 2p_t)(y_t - p_t)$, which holds for all $y_t \in \{0, 1\}$ and $p_t \in [0, 1]$ and rewrite the above expression as:
\[
    \sum_{t=1}^T p_t(1-p_t) \|\Phi(x_t, p_t)\|^2_{\cF} + 2\sum_{t=1}^T(y_t-p_t)\left(\sum_{s=1}^{t-1} k((x_t, p_t), (x_s, p_s))(y_s-p_s) + \|\Phi(x_t, p_t)\|_{\cF}^2(1-2p_t)\right).
\]
Since the rightmost parenthesized term is, by definition, precisely $S_t(p_t)$, we have shown that
\begin{align*}
    \E\left\|\sum_{t=1}^T (y_t - p_t)\Phi(x_t, p_t)\right\|^2_{\cF} = \E\left[\sum_{t=1}^T p_t(1-p_t)\|\Phi(x_t, p_t)\|^2_{\cF}\right] + 2\sum_{t=1}^T \E\left[S_t(p_t)(y_t-p_t)\right].
\end{align*}
Now, using our earlier result (\cref{eq:forecast_hedging}), we conclude that:
\begin{align*}
     \E\left\|\sum_{t=1}^T (y_t - p_t)\Phi(x_t, p_t)\right\|^2_{\cF}  &\leq \E\left[\sum_{t=1}^T p_t(1-p_t)\|\Phi(x_t, p_t)\|^2_{\cF}\right] + 2\sum_{t=1}^T \frac{1}{10 t^2} \\ 
     & \leq \E\left[\sum_{t=1}^T p_t(1-p_t)\|\Phi(x_t, p_t)\|^2_{\cF}\right] + \frac{2}{10} \cdot \frac{\pi^2}{6}
\end{align*}
where we used the fact that $\sum_{t=1}^\infty t^{-2} = \pi^2 /6$. Noting that 
\begin{align*}
p_t(1-p_t)\|\Phi(x_t, p_t)\|^2_{\cF} = p_t(1-p_t)k((x_t,p_t),(x_t,p_t)),
\end{align*} and applying Jensen's inequality, the above equation implies that: 
\begin{align}
\label{eq:feature_regret}
    \E\left\|\sum_{t=1}^T (y_t - p_t)\Phi(x_t, p_t)\right\|_{\cF} \leq \sqrt{1 + \sum_{t=1}^T \E_{p_t \sim \Delta_t} p_t(1-p_t) k((x_t,p_t),(x_t,p_t)) }.
\end{align}
To conclude the proof, we use the reproducing property $f(x,p) = \langle f, \Phi(x,p) \rangle_\cF$, which, along with Cauchy-Schwarz, relates the indistinguishability error to the above expression as follows:
\begin{align*}
   |  \sum_{t=1}^T \E_{p_t \sim \Delta_t}(y_t - p_t)f(x_t,p_t)|  &= \big|\E_{p_t \sim \Delta_t}[\langle f, \sum_{t=1}^T (y_t - p_t)\Phi(p_t,x_t) \rangle_\cF ] \big| \\ 
   & \leq  \|f\|_\cF \E \left\|\sum_{t=1}^T (y_t - p_t) \Phi(x_t, p_t)\right\|_\cF.  
\end{align*}
\end{proof}

\remove{
We note that the bounds in our main theorem depend adaptively on the \textit{function norm} $\| f \|_\cF$ (i.e., functions with smaller norm yield stronger guarantees).
The bounds also depend on what we call the \textit{feature norm} $k((x_t, p_t), (x_t, p_t))$ in the RKHS. (We call it a feature norm since $\sqrt{k((x_t, p_t), (x_t, p_t))} = \sqrt{\inner{\Phi(x_t, p_t)}{\Phi(x_t, p_t)}} = \fnorm{\Phi(x_t,p_t)}$.)
Our subsequent results will often depend on the being able to guarantee that $\| f \|_\cF$ is bounded for all $f$ in some relevant subset $\cF_0 \subseteq \cF$ and that $k(\cdot, \cdot)$ is uniformly bounded by a constant over the input space $\cX \times [0, 1]$.
\ch{Give some examples/intuition for when this is to be expected.}

}

\paragraph{Discussion.} The bound guarantees non-asymptotic OI error of at most $\sqrt{T}$ for all functions $f$ that lie in the RKHS $\cF$ induced by a pre-specified kernel $k$. 
\footnote{In particular, the bound holds for all values of $T$.}
While the bound holds for all functions in the RKHS, it is \emph{adaptive}. For each $f$, it depends on the norm $\|f\|_{\cF}$ but not on the number of functions $|\cF|$ (which is in fact infinite for every choice of kernel $k$). 
The norm of a function in an RKHS can often be interpreted as an instance-specific notion of complexity. Consequently, the OI error bound satisfies the intuitive property that it is smaller for simple functions, and larger for more complicated functions. 

The guarantees are also adaptive since they depend on norms of the features in the sequence, $k((x_t, p_t), (x_t,p_t)) = \fnorm{\Phi(x_t,p_t)}^2$, and the variance of the predictions $p_t(1-p_t)$. Adapting to the variance is particularly useful in the link prediction setting since we expect most edges in professional networks to be unlikely to form, meaning that the OI error bound is smaller.

We also note that neither the run-time of the algorithm nor the associated regret bounds have any explicit dependence on number of functions $|\cF|$. Both of these properties are determined by the kernel function $k$. 

In the following propositions, we instantiate the theorem above with specific choices of kernel functions $k$, illustrating how it can be used to guarantee indistinguishability with respect to interesting classes of functions $\cF$. We then compare our results to previous work.

\remove{The statement of \Cref{prop:low-deg-booleans} requires definition of some notation and several relevant function classes.
}
We will use multi-index notation to denote $x_S = \prod_{i \in S} x_{i}$ for $S \subseteq [n]$.
Informally, \Cref{prop:low-deg-booleans} states that the algorithm guarantees outcome indistinguishability at $\sqrt{T}$ rates with respect to tests that are the product of a low-degree function on $\cX \subseteq \{0, 1 \}^n$ and either binned functions or functions satifying mild smoothness conditions of the prediction $p$.

\newcommand{\sobkern}[2]{\frac{(e^{{#1}} + e^{-{#1}})(e^{1-{#2}} + e^{{#2}-1})}{2(e  - e^{-1})}}
\begin{corollary} [Low-degree functions on $\{ 0, 1\}^n$] \label{prop:low-deg-booleans} \label{prop:exampleinstantiations}
    Let $\cF_{\mathrm{LowDeg}} \subseteq \{\{-1,1\}^n \rightarrow [-1,1]\}$ be a set of Boolean functions whose Fourier spectrum is supported on monomials of degree at most $d$ (\textit{e.g.}, decision trees of depth $d$, or polynomials).\footnote{
Recall that Boolean functions over $\{-1, 1\}^n$ can always be written as polynomials, and that the Fourier spectrum of functions on $\{-1, 1\}^n$ are simply the coefficients of monomials in the polynomial.
See \cref{ex:booleanfunctions} for more discussion of functions on the Boolean hypercube.}
    \begin{align*}
        \cF_{\mathrm{LowDeg}} = \left\{ f \; : \; \exists \; \alpha \text{ such that } \| \alpha \|_\infty \leq 1, f(x) = \sum_{S \subset [n], \abs{S} \leq d} \alpha_{S} x_S, \forall x \in \{ 0, 1\}^n \right\}.
    \end{align*}
    Furthermore, let $\cF_{\mathrm{\Cont}} \subseteq  \{[0,1] \rightarrow [-1,1]\}$ be the class of continuous, differentiable functions with derivative uniformly bounded in $[-1, 1]$ and $\cF_{\mathrm{Grid}}$ to be the set of functions 
        \begin{align*}
            f_r(p) = \1\left\{ \frac{r-1}{N} \leq p < \frac{r}{N}\right\}
        \end{align*}
    parametrized by some positive integer $N$ and $r \in \{1, \dots, N-1\}$. We also define $f_{N}(p) = 1 \{(N-1)/N \leq p \leq 1 \}$ so the grid covers the whole interval.
    Then, the \algname{} run on the kernel 
    \begin{align*}
         k((x,p),(x',p')) \defeq \bigg( &\sobkern{\min \{ p, p' \}}{\max \{ p, p' \}} \\
         &+ 1\left\{\exists \; r \in [N] \; : \; f_r( p) = f_r(p') = 1 \right\}\bigg) \sum_{S \subset [n], |S| \leq d} x_{S} x'_{S}, 
    \end{align*}
    generates a sequence of predictions such that for all $f_x\in \cF_{\mathrm{LowDeg}}$ and $f_p \in \cF_{\mathrm{\Cont}} \cup \cF_{\mathrm{Grid}}$:
       \begin{align*}
         \left|\sum_{t=1}^T  \E_{p_t \sim \Delta_t}f_x(x_t)f_p(p_t)(y_t - p_t) \right| \leq  6 \sqrt{n^d T}. 
    \end{align*}
\end{corollary}   

\newcommand{\anovafuncbound}{1}
\begin{proof}
    From \Cref{ex:anova}, we have that $\cF_{\mathrm{LowDeg}}$ is the RKHS induced by the kernel
    \begin{align*}
         k_{\mathrm{LowDeg}}(x, x') &= \sum_{S \subset [n], |S| \leq d} x_{S} x'_{S} \\
         &= \sum_{k=1}^d\binom{n}{k} < d \left( \frac{ne}{d}\right)^d < 4 n^d,
    \end{align*}
    since $x_{S}^2 \leq 1$.
    Also, from the example, for $f \in \cF_{\mathrm{LowDeg}}$, the norm of $f$ is the $\ell^2$ norm of the coefficients $\alpha$, which is bounded by 1 by assumption: $\| f \|_{\cF_{\mathrm{LowDeg}}} \leq 1.$ 
    
    Next, from \Cref{ex:sobolev}~\cite{berlinet2011reproducing}, note that $\cF_{\mathrm{\Cont}}$ is in the Sobolev space $W^{1,2}([0,1])$ associated with the kernel,
    \begin{align*}
        k_{\mathrm{\Cont}}(p, p') = \sobkern{\min \{ p, p' \}}{\max \{ p, p' \}}.
    \end{align*}
and with associated function norm:
    \begin{align*}
        \| f \|_{\cF_{\mathrm{\Cont}}}^2 = \int_{0}^1 f(p)^2 \; dp + \int_{0}^1 f'(p)^2 \; dp.
    \end{align*}
    Intuitively, functions in the Sobolev space $W^{1,2}([0,1])$ are differentiable, have bounded $L^2$ norm and have derivative with bounded $L^2$ norm.
    See \cref{ex:sobolev} for a definition and discussion of the Sobolev space $W^{1,2}([0,1])$. 
    Now, by assumption, for all $f \in \cF_{\mathrm{\Cont}}$, it holds $\sup_p f(p)^2 \leq 1$ and $\sup_p f'(p)^2 \leq 1$. Hence, $\| f \|_{\cF_{\mathrm{\Cont}}} \leq \sqrt{2}$.
    Also, $k_{\mathrm{\Cont}}(p, p) \leq 2$.

    Next, we can apply \Cref{thm:finite-membership-rkhs}, to show that $\cF_{\mathrm{Grid}}$ is in the RKHS induced by 
    \begin{align*}
        k_{\mathrm{Grid}}(p, p') &= \sum_{r=1}^{N} f_r(p) f_r(p') \\
        &= 1\left\{\exists \; r \in [N] \; : \; \frac{r}{{N}} \leq p, p' \leq \frac{r+1}{{N}}\right\}.
    \end{align*}
    From the lemma, $\| f \|_{\mathrm{Grid}} \leq 1$ and $k_{\mathrm{Grid}}(p, p) \leq 1$. Defining,
    \begin{align*}
        k \defeq (k_{\mathrm{\Cont}} + k_{\mathrm{Grid}}) \cdot k_{\mathrm{LowDeg}},
    \end{align*}
    from the calculations above we have that for all $x,p \in \cX \times [0,1]$,
    \begin{align*}
       k((x,p), (x,p)) \leq 12 n^d. 
    \end{align*}
    And, by \cref{lem:kernelsum} and \cref{lem:productkernel}, $f_x \cdot f_p \in \cF$ for $\cF$ the RKHS associated with $k$ and for all $f_p \in \cF_{\mathrm{\Cont}} \cup \cF_{\mathrm{Grid}}$ and $f_x \in \cF_{\mathrm{LowDeg}}$.
    
    Applying the triangle and Cauchy-Schwarz inequalities, we have, for all $f_p \in \cF_{\mathrm{\Cont}} \cup \cF_{\mathrm{Grid}}$ and $f_x \in \cF_{\mathrm{LowDeg}}$, $ \|f_p \|_{\cF_{\mathrm{\Cont}} + \cF_{\mathrm{Grid}}} \leq \sqrt{2} + 1$ so
    \begin{align*}
        \| {f_p \cdot f_x} \|_{(\cF_{\mathrm{\Cont}} + \cF_{\mathrm{Grid}}) \cdot \cF_{\mathrm{Grid}}} \leq (\sqrt{2} + 1)\cdot\anovafuncbound.
    \end{align*}
    Finally, applying \cref{thm:indistinguishability_main} with the function and feature norms above, we have the desired bound:
    \begin{align*}
        \bigg| \sum_{t=1}^T f_x(x_t) f_p(p_t) (y_t - p_t) \bigg| &\leq (\sqrt{2} + 1) \sqrt{1 + \sum_{t=1}^T  12 n^d / 4 } \\
        & \leq 3 \sqrt{1 + 3 n^d T} \leq 6 \sqrt{n^d T}.
    \end{align*}
\end{proof}

{We note that there is a great deal of flexibility when deciding how the distinguishers above depend on the prediction $p$. 
Here, we chose a the union of a specific class of indicator functions with the set of continuous, differentiable functions with bounded domain and first derivative. 
However, we could equivalently have chosen a different class of functions satisfying mild smoothness conditions or a different (possibly infinite) partition of $[0,1]$.
Alternately, if $p$ is always in a finite set $\cP$, $| \cP | < \infty$, distinguishers could be chosen to be $\1 \{ p = \bar{p} \}$ for all $\bar{p} \in \cP$.

Before we move on, we state two importance 

\begin{remark}[Boundedness of functions]
    Throughout this work, we will often impose requirements that various functions or their derivative be bounded on [-1,1].
    However, functions can be trivially re-scaled to hold for constants other than 1.
\end{remark}

\begin{remark}[Non-asymptotic results]
    The rates we achieve in this paper are non-asymptotic.
    Throughout, we take care to derive the constant so that dependencies on auxiliary parameters (in the case of \cref{prop:low-deg-booleans}, $n$ and $d$) so their dependence is clear.
    We opt for simpler rather than tighter constants throughout for clarity.
\end{remark}

Our next corollary gives a similar guarantee to the previous for any finite set of bounded functions.

\noindent  
\begin{corollary}[Any set of real-valued functions whose $L^2$ counting measure is bounded uniformly over $x,p$] \label{cor:finiteset}
    Let $\cX$ be any set, let $\cI$ be any index set and let $m$ be a constant. Also, let $\cF = \{ f_i \}_{i \in I}$ be a collection of functions $f_i \; : \; \cX \times [0, 1] \to \cR$ indexed by $\cI$. Suppose that for each $x \in \cX, p \in [0, 1]$, we have
    \begin{equation}
        \label{eq:finite-membership-rkhs-assumption}
        \sum_{i \in \cI} f_i(x, p)^2 \leq m,
    \end{equation}
    Then, the \algname{} run on the kernel 
    \begin{align}
    \label{eq:intersection_oracle}
        k((x,p),(x',p')) \defeq \sum_{i \in I} f_i(x,p) f_i(x',p'),
    \end{align}
    (where we assume the sum can be evaluated in polynomial time in $T$)
    is guaranteed to generate a sequence of predictions such that for all $f \in \cF$, 
    \begin{align*}
         \big|\sum_{t=1}^T \E_{p_t \sim \Delta_t} f(x_t, p_t)(y_t - p_t) \big| \leq \sqrt{mT + 1}.
    \end{align*}
\end{corollary}

\begin{proof}
    The result follows as a direct consequence of \cref{thm:finite-membership-rkhs} and \cref{thm:indistinguishability_main}. The feature norm is uniformly bounded by $m$ and for all $f \in \cF$, $\| f \|_{\cF} \leq 1$.
\end{proof}

A sufficient (but not necessary) condition for \Cref{eq:finite-membership-rkhs-assumption} to hold is that $\cF$ is finite,
in which case $\cF$ might contain arbitrary pre-existing predictors with which we would like the \algname{} to guarantee outcome indistinguishability with respect to.
In other cases, $\cI$ need not be countable, in which case, the sum appearing in \Cref{eq:finite-membership-rkhs-assumption} should be interpreted as an integral with respect to the counting measure on $\cI$. 
In this case, a necessary (but not sufficient) condition for \cref{eq:finite-membership-rkhs-assumption} to hold is that for each $x \in \cX$, there are at most countably many $i \in \cI$ such that $f_i(x) \neq 0$.

 \paragraph{Comparison to prior work.} As per our earlier discussion, the closest work to ours is \cite{vovk2007non}. The $\kalg$ algorithm presented therein achieves a similar guarantee, but requires that the kernel $k(x,p)$ is continuous in $p$. This restriction rules out indistinguishability with respect to binary functions (or any other discontinuous $f$). Distinguishers of this form were the main focus of \cite{hkrr,dwork2021outcome}.  Our algorithm works for \emph{any} kernel, and in particular can be used to guarantee indistinguishability with respect to binary functions as in first example above.
The computation complexity of our algorithm and Vovk's are essentially identical. 

Also closely related to our work, the algorithm in \cite{gupta2022online} guarantees online indistinguishability with respect to a finite set of binary valued functions $\cF$. Furthermore, while their OI error bound scales as $\sqrt{\log|\cF|
}$, the per round computational complexity scales linearly with $|\cF|$. In comparison, 
our algorithm can be used to guarantee indistinguishability with respect to both real- and Boolean-valued functions. Achieving indistinguishability with respect to real-valued functions is crucial for our later results on omniprediction. 

Furthermore, as stated previously, the computational complexity and OI error of the \algname{} have no explicit dependence on the size of $\cF$. Both of these are determined by the kernel $k$. 
As seen in \Cref{prop:low-deg-booleans}, certain infinite classes of functions can be efficiently represented by kernels that can be computed in constant time.
For certain worst-case classes $\cF$, we can still guarantee indistinguishability (as in the second part of \Cref{prop:exampleinstantiations}). However, the kernel in this construction requires enumerating over $\cF$ and both the runtime and OI error scale polynomially with $|\cF|$. Therefore, for the specific case where one aims to be indistinguishable with respect to a finite set of \emph{Boolean} functions not known to be efficiently represented by a kernel, the algorithm in \cite{gupta2022online} is preferable. In that setting, both our procedure and the one in \cite{gupta2022online} have run times linear in $|\cF|$, but their OI error is significantly smaller (polylogarithmic vs polynomial). 

The principal strength of \Cref{cor:finiteset} is that we can guarantee indistinguishability with regards to any real-valued function $f$ that is efficiently computable. This in particular includes any neural network or prediction baseline one might consider. We return to this point in the next section.

\paragraph{Additive models and boosting.} As a final remark before the proof of the proposition, we note that the previous result also guarantees outcome indistinguishability with respect models like random forests or gradient boosted decision trees. These learning algorithms are the gold standard in certain data modalities \cite{gardner2022subgroup,grinsztajn2022tree}. 

In particular, let $\cF_{DTd} \subseteq \{\{\pm1\}^n \rightarrow [-1,1]\}$ be the class of regression trees of depth~$d$. Random forests and gradient-boosted trees are additive ensembles of the form:
\begin{align}
\label{eq:additive_ensembles}
    f(x) = \sum_{i} \lambda_i f_i(x) 
\end{align}
where $\lambda_i$ are real-valued coefficients and $f_i \in \cF_{DTd}$
Since, $\cF_{DT} \subseteq \cF_{\mathrm{LowDeg}}$ (see e.g \cite{odonnel2021analysis}), then the \algname{} instantiated with the kernel from \Cref{prop:low-deg-booleans} guarantees indistinguishability with respect to any $f \in \cF_{DTd}$. Since  indistinguishability is closed under addition, then the same algorithm also guarantees indistinguishability with error $\cO(\gamma \sqrt{n^dT})$ with respect to additive ensembles as in \Cref{eq:additive_ensembles} as long as $\sum_i |\lambda_i|$ is $\cO(\gamma)$.

\subsection{Specializing the Any Kernel algorithm to the link prediction problem} \label{sub:oi_graphs}

Having introduced this technical machinery, we now specialize it to the link prediction problem, turning our attention to designing specific kernels whose corresponding function spaces contain interesting classes of distinguishers that operate on graphs.
The tests we consider fall into two broad categories: those capturing socially salient information and those for which passing these tests likely implies good predictive performance.
Socially salient tests might include whether a pair of individuals belong, respectively, to a specific pair of demographic groups (\textit{i.e.}, multicalibration).
On the other hand, predictive performance tests aim to capture correlations between features, predictions, and outcomes.

In this section, we change notation from $f(x,p)$ to $f(u,p)$ reflect the fact that distinguishers $f$ operate over the universe $\cU$ consisting of pairs of nodes $a=(i,j)$ and a graph $G$.
We will also make liberal use the set of grid indicator functions $\cF_{\mathrm{Grid}} = \{ f_r\}_{r=1}^N$ for a positive integer $N$ where $f_r = 1\{ (r-1)/N \leq p < r/N\}$ for $r=1,\dots,N-1$ and $f_{N} = 1\{ (N-1)/N \leq p \leq 1\}$.
As in \cref{prop:low-deg-booleans}, this choice is somewhat arbitrary: we could equivalently use the sets of functions satisfying mild smoothness conditions or arbitrary partitions of the unit interval.
We will assume $N$ is a universal constant throughout.
 
\paragraph{Group membership tests.}

A simple starting point for socially salient tests are those which given a pair of individuals $(i,j)$ outputs 1 if $i$ belongs to a demographic group $g$ and $j$ belongs to group $g'$. Groups may be defined by, for example, race, ethnicity, gender, age, religion, education, occupation and/or political or organizational affiliation. We will let $g$ be a binary function $\cZ \rightarrow \{0,1\}$ which takes in node-level features $z_{i,t}$ and returns 0 or 1. 
These tests are analogous to multiaccuracy   \cite{hkrr,kim2019multiaccuracy} (if they do not depend on predictions $p$) and multicalibration \cite{hkrr} (if they do), adapted to the link prediction setting, and allowing for arbitrary pairs of demographic groups.
Indeed, cross-group ties are the focus of significant study in the networks literature  \cite{abebe2022effect, calvo2004effects, zeltzer2020gender, stoica2018algorithmic, okafor2020social}, and platforms may wish to ensure predictions are calibrated with respect to them.

\begin{proposition}[Pairs of demographic groups]
\label{prop:all_pairs}
Let $\cG \subseteq \{ \cZ \to \{ 0, 1\} \}$ be a (not necessarily disjoint or finite) collection of demographic group indicator functions on $\cZ$ such that each individual $i$ at any time $t$ belongs to at most $m$ groups for some positive integer $m$:
\begin{align*}
    \max_{t \in [T], i \in V_t}\bigg| \sum_{g \in \cG} g(z_{i,t}) \bigg| \leq m.
\end{align*}
For a positive integer $N$ and given $u = (i,j, G)$ and $u' = (i',j',G')$, define the kernel $k$ to be 
\begin{align*}
    k((u,p),(u',p')) = \1\left\{\exists \; r \in [N] \; : \; f_r( p) = f_r(p') = 1\right\} \sum_{g,g' \in \cF_{G}} g(z_{i}) g'(z_{j}) g(z_{i'}) g'(z_{j'}) 
\end{align*}
where $(z_{i},z_{j})$ are the node-level features of the pair $(i,j)$ in $G$ and $(z_{i'}, z_{j'})$ are the node level features of $(i',j') \in G'$. Then, the \algname{} with kernel $k$ generates a sequence of predictions satisfying,  
\begin{align*}
         \left|\sum_{t=1}^T  \E_{p_t \sim \Delta_t}(y_t - p_t) \1\left\{ g(z_{i,t})=1, g'(z_{j,t})=1,
         f_r(p_t) = 1\right\}\right| \leq  \sqrt{mT + 1}.
    \end{align*}
for all $g,g' \in G$ and $r \in 1, \dots, N$ where $u_t = (i_t,j_t,G_t)$. 
\end{proposition}
\noindent 
 Assuming that checking whether a pair of predictions $p, p'$ fall in the same grid cell and evaluating the indicator functions $g \in \cG$ takes constant time, then the kernel can be naively computed in time $\cO(1)$. Therefore, following \cref{thm:indistinguishability_main}, at time $t$, the algorithm generates a prediction $p_t$ in time ${\widetilde{\cO}}(t m)$.
\begin{proof}
    The result is a direct implication of \cref{cor:finiteset}.
    Let $\cF$ in \cref{cor:finiteset} be the cross product of group membership indicators and grid indicators $\cG \times \cF_{\mathrm{Grid}}$ and notice
    \begin{align*}
        k((u,p),(u',p')) &= \sum_{r=1}^N f_r(p) f_r(p') \sum_{g,g' \in \cF_{G}} g(z_{i}) g'(z_{j}) g(z_i') g'(z_j')\\ &=1\left\{\exists \; r \in [N] \; : \; f_r( p) = f_r( p') =1\right\} \sum_{g,g' \in \cF_{G}} g(z_{i}) g'(z_{j}) g(z_i') g'(z_j') 
    \end{align*}
    is the associated kernel as defined in \cref{cor:finiteset}.
    Notice that \Cref{eq:finite-membership-rkhs-assumption} is satisfied with the $m$ in the statement of the result, since $x$ cannot be in more than $m$ groups and $p$ cannot be in more than one grid cell.
    Thus, we have verified the assumptions in the corollary and the bound holds.
\end{proof}

Closely related to group membership is the idea of homophily \cite{mcpherson2001birds}. 
Informally, homophily is the tendency of individuals to connect those who are similar to themselves.
Homophily may be defined by membership in a demographic group as well as geographic proximity \cite{verbrugge1977structure}, social capital  \cite{borgatti2003network}, and political/social attitudes/beliefs \cite{10.1093/qje/qjr044}.
All of these measures of homophily are \textit{scalar valued} functions of node-level features.
%
%
In these cases, the proposition above can be straightforwardly extended so that the algorithm generates predictions with are outcome indistinguishable with respect to (functions of) these measures.

An alternate formulation of the link prediction problem would also consider edge-level features such the frequency or intensity of interaction between individuals.
For example, the influential notion of \textit{weak ties}, originally characterized qualitatively as a ``combination of the amount of time, the emotional intensity, the intimacy (mutual confiding), and the reciprocal services which characterize the tie'' \cite{granovetter1973strength}, are usually defined quantitatively in terms of interaction intensity (see, \textit{e.g.}, \cite{rajkumar2022causal}).
Our results could be trivially extended to solve this formulation of link prediction where distinguisher may also consider edge level features. However, for simplicity of presentation, we omit including edge-level features.

\paragraph{Network topology tests.} \label{sub:topology}

We now consider tests that depend on the structure of the graph. 
A particularly simple set of such tests is based on embeddedness, or the number of mutual connections between two individuals $(i,j)$ on a graph $G$. The sociological notion of embeddedness, as discussed in \cite{granovetter1985economic}, concerns the degree to which individuals' activities are \textit{embedded} within in social relations, \textit{i.e.}, networks. Formally for $u = (i,j,G)$, we quantify the structural embeddedness of $u$ (following the definition in \cite{easley2010networks}) as
\begin{align}
\label{eq:embed}
    \embed(u) \defeq |\Gamma_G(i) \cap \Gamma_G(j)|.
\end{align}
Note that the pair of individuals themselves need not be connected.
For example, a rich literature studies \textit{long ties} or  \textit{local bridges}, which are {ties with embeddedness zero} (see, \textit{e.g.}, \cite{granovetter1973strength, burt2004structural, jahani_long_2023, easley2010networks}).
Embeddedness is measured and carefully analyzed by digital platforms like LinkedIn in practice \cite{rajkumar2022causal}.
It also underlies classical theories of network evolution through triadic closure \cite{kossinets2006empirical, jackson2007meeting, asikainen2020cumulative, abebe2022effect}.
Here in our next result, we show one can construct an efficient kernel $k$ that guarantees online outcome indistinguishability with respect to embeddedness tests. 

\begin{proposition}[Embeddedness] \label{prop:embeddedness}
For $u = (i,j,G)$ and $u' = (i',j', G')$ define the kernel 
\begin{align*}
   k((u,p),(u',p')) \defeq \1 \{\embed_t(u) = \embed_t(u'), \exists r \in [N] \; : \; f_r(u) = f_r(u') = 1\}.
\end{align*}
Then, the \algname{} run with kernel $k$ generates a sequence of predictions satisfying,  
\begin{align*}
         \big|\sum_{t=1}^T  \E_{p_t \sim \Delta_t}(y_t - p_t) 1\{\embed_t(u_t) = c, f_r(p) = 1 \}\big| \leq  \sqrt{\sum_{t=1}^T\E_{p_t \sim \Delta_t} p_t(1-p_t) + 1} \leq 2\sqrt{T}.
    \end{align*}
for all $c \in \N$ and $r \in [N]$.
\end{proposition}

Since the kernel only checks whether two different pairs of individuals have the predictions that fall in the same grid cell and have an identical number of mutual friends,  the kernel can be computed in the time it takes to compute neighborhood intersections. 

An advantage of the class $1\{\embed_G(u_t) = c, f_r(p) = 1 \}_{c \in \N, r \in [N]}$ is that neither the run time nor OI error depends on the maximum degree of nodes in the graph. 
We also note that the above formulation could be straightforwardly modified to include indicator functions for having embeddedness more or less than $c$, as long as it is efficient to compute embeddedness.
Lastly, we note that the construction can be generalized to include distinguishers of the form $i$ and $j$ have $c$ distance $r$ neighbors in common by simply changing $\Gamma$ to $\Gamma^{(r)}$ in the definitions above.

We can generalize the embeddedness tests above even further to guarantee outcome indistinguishability with respect to all tests that depend on the isomorphism class of the subgraph induced by the neighborhoods $\Gamma(i),\Gamma(j)$.

A function $f$ from graphs $G$ to the real-line is isomorphism-invariant if for any two graphs $G$ and $G$ such that $G$ and $G'$ are isomorphic, it holds that $f(G) = f(G')$.
Abusing notation, we can write isormorphism-invariant functions $f$ as those defined on isomorphism (equivalence) classes $\bar{G}$ where $\bar{G}$ is a set of graphs that are all isomorphic to each other.

Several interesting classes of functions $f$ are isomorphism-invariant. For instance, any function $f$ that just depends on the number of nodes or edges in the graph, the degree distribution, or the spectrum of the graph Laplacian is isomorphism-invariant.
Several classes of isomophism-invariant functions have been studied extensively in the networks literature, like various notions of structural cohesion (which might, \textit{e.g.}, measure the edge density of the induced subgraph in an individual's neighborhood \cite{friedkin1993structural}).

In the following proposition, we will use the following notation: given a set of nodes $S$ and a graph $G$, let $G[S]$ denote the induced subgraph of $S$ on $G$.
Also, we will use $\Gamma(i),\Gamma'(i')$ to refer to the neighborhoods $\Gamma_G(i), \Gamma_{G'}(j)$ for graphs $G, G'$ respectively.
We will write $G \simeq G'$ to denote that $G$ and $G'$ are isomorphic.

\begin{proposition} \label{prop:isomorphism}
Let $\cF_{\mathrm{iso}} \subseteq \{ \cG \rightarrow \R \}$ denote the set of all isomorphism invariant functions and $\cF_{\mathrm{Grid}} = \{f_1, \dots, f_N\}$ be the grid indicator functions on the unit interval as above. 
Furthermore, for $u = (i,j,G)$ and $u' = (i',j',G')$
define the function $k$ to be 
\begin{align*}
   k((u,p),(u',p)) =  \1 \{G[\Gamma(i) \cup \Gamma(j)] \simeq G'[\Gamma'(i) \cup \Gamma'(j)],\; \exists r \in [N] \; : \; f_r(p)=f_r(p')=1\}.
\end{align*}
Suppose all graphs in the sequence $\{ G_t \}_{t=1}^T$ degree bounded by a constant. Then $k$ can be computed in polynomial time and the \algname{} instantiated with the kernel $k$ is guaranteed to generate a sequence of predictions satisfying:
\begin{align*}
         \big|\sum_{t=1}^T  \E_{p_t \sim \Delta_t}(y_t - p_t) f(u_t) \1\{p = \bar{p}\}\big| \leq \|f\|_{\cF}  \sqrt{\sum_{t=1}^T\E_{p_t \sim \Delta_t} p_t(1-p_t) + 1} \leq 2\|f\|_{\cF} 
 \sqrt{T}.
\end{align*}
for any $f \in \cF_{\mathrm{iso}} \subseteq \cF$. For the special case of functions $f_{\bar{G}}(i,j,G) = \1\{G \in \bar{G}\}$ for some isomorphism class $\bar{G}$, the dependence on $\|f_{\bar{G}}\|_{\cF}$ can be removed since $\|f\|_{\cF} \leq 1$ for every $\bar{G}$.

\end{proposition}
\begin{proof}
Let $\bar{G}_1, \bar{G}_2, \dots$ be the sequence of graph isomorphism classes in some ordering (perhaps lexicographic, where all isomorphism classes for graphs of size $n$ come before those of size $n+1$ for all $n \in \N$). Let $\Phi(G)$ be the feature map defined as, 
\begin{align}
    \Phi(G) = (\1\{G \in \bar{G}_1 \} ,\; \1 \{G \in \bar{G}_2 \},\; \dots).
\end{align}
For $u = (i,j,G)$ and $u' = (i', j', G')$,
\begin{align*}
k_{\mathrm{iso}}(u, u') \defeq \langle \Phi(G[\Gamma(i) \cup \Gamma(j)]), \Phi(G'[\Gamma'(i) \cup \Gamma'(j)]) \rangle
\end{align*}
where the inner product $\langle \cdot, \cdot \rangle$ is the standard inner product in $\ell^2$, the Hilbert space of square summable sequences ($\langle x, y\rangle = \sum_{i=1}^\infty x_i y_i$). 
Since $G$ can only be in one of the $\bar{G}_i$, $\Phi(G)$ is a square-summable sequence (only one element is 1, all the others are 0). 
So $k_{\mathrm{iso}}$ is a valid kernel and $k_{\mathrm{iso}}(u, u') \leq 1$  for all $u, u' \in \cU$.
Since all nodes in $G_t$ are assumed to have bounded degree, there are only a constant number of isomorphism classes for the subgraph $G[\Gamma(i) \cup \Gamma(j)]$.
Thus, $k$ can be computed efficiently via brute force search.\footnote{One could also of course run more sophisticated procedures for isomorphism testing if one desires (\textit{e.g.}, Luks' algorithm \cite{luks1982isomorphism}), but these are unnecessary for polynomial runtime guarantee in this setting since our distinguisher only examine the local neighborhood of $(i,j)$ which are at most of constant size.} 

The fact that $\cF_{\mathrm{iso}} \subseteq \cF_{k_{\mathrm{iso}}}$ for $\cF_{k_{\mathrm{iso}}}$, the RKHS associated with kernel $k_{\mathrm{iso}}$, follows from the Moore-Aronszajn Theorem (\cref{thm:moore-aronszajn}) which states that the corresponding RKHS of the kernel $\cF$ is equal to 
\begin{align*}
    \mathsf{span}\{ \Phi(G) : G \text{ is a graph}\}.
\end{align*}
Given any isomorphism invariant function $f$, we can write it as, 
\begin{align*}
    f(G) = \langle \Phi(G), f \rangle = \sum_{i=1}^\infty f(\bar{G}) \1 \{ \bar{G} = \bar{G}_i\},
\end{align*}
where $\bar{G}$ is the set of graphs that are isomorphic to $G$. Here, we used the fact that $f$ is isomorphism-invariant and again slightly abused notation to write $f(\bar{G})$ where $\bar{G}$ is a set, instead of one graph.
Applying \cref{thm:indistinguishability_main} with the function and feature norms above yields the desired result.
\end{proof}

As with embeddedness tests, isomorphism tests can be naturally extended to depend on the distance $r$ neighborhoods of pairs of nodes, by simply replacing each $\Gamma$ in the proposition with $\Gamma^{(r)}$ (for constant $r$).
Various network centrality measures, like $k$-core similarity, betweenness centrality, eigenvalue centrality and others (see, \textit{e.g.}, \cite{rodrigues2019network}) may be computed using the induced subgraph of distance $r$ neighborhoods. Similarly, core-periphery measures \cite{doi:10.1137/120881683} may be similarly defined for distance $r$ neighborhoods. 
In each of these cases, care must be taken to ensure that the measure can be computed efficiently and that the function norms are bounded.

\paragraph{Tests using network topology and neighbors' feature vectors.} \label{sub:alter}

We end this section by considering distinguishers that examine both the local neighborhood structure, as well as the \emph{features} of individuals in these neighborhoods. (The graph isomorphism tests presented previously only examine the structure of the neighborhood, but not their individual features.)

\Cref{cor:finiteset} provides for OI guarantees that hold with respect to very powerful predictors. For example, we may take $\cF$ to be any finite set of graph neural networks, which are currently state-of-the-art for link prediction \cite{NEURIPS2018_53f0d7c5, yun2019graph} and any number of other graph-related tasks (see, \textit{e.g.}, \cite{zhou2020graph}) and are widely deployed across digital platforms that host social networks \cite{zhou2020graph, zhang2020revisiting}. \Cref{cor:finiteset} immediately implies that the \algname{} yields
\begin{align*}
     \bigg|\sum_{t=1}^T \E_{p_t \sim \Delta_t} f(u_t)(y_t - p_t) \bigg| \leq \sqrt{mT}. 
\end{align*}
for all $f \in \cF$.



\textit{R-convolutions (convolutions over relations).} This machinery can also be used to guarantee indistinguishability to functions of the form
\begin{align}
\label{eq:rconv_functions}
f(i,j,G) = \langle w, \sum_{v \in \Gamma(i) \cup \Gamma(j)} \Phi(z_v) \rangle_{\cF}  
\end{align}
where $\Phi(z): \cZ  \rightarrow \cF$ is a feature mapping and $w \in \cF$ is an element in the RKHS. This particular class of functions can be efficiently represented by using the R-convolutional kernel from \cite{haussler1999convolution}, which, given a feature map $\Phi$ and $u  =(i,j,G), u' = (i',j', G')$, computes: 
\begin{align*}
    k(u,u') = \sum_{v \in \Gamma(i) \cup \Gamma(j), v' \in \Gamma'(i') \cup \Gamma'(j')} \langle \Phi(z_v), \Phi(z_{v'}) \rangle_{\cF}
\end{align*}
Assuming that the features $\Phi(v)$ and weight $w$ have norm at most 1, and that any node in the graph has degree at most $d$, the \algname{} guaranteees $\cO(d\sqrt{T})$ indistinguishability to functions of the form in \cref{eq:rconv_functions}. 
The features $\Phi$ may include socially salient measures of diversity \cite{burt1982toward} or bandwidth \cite{aral2011diversity}.

\section{Online Omniprediction and Applications to Link Prediction} 
\label{sec: omni}

Up until this point, we have focused on designing online algorithms which satisfy online outcome indistinguishability with respect to various classes of tests.
In this section, we illustrate how these previous insights and algorithms also imply \emph{loss minimization} with respect to many different objectives $\cL$ and infinitely large benchmark classes $\cH$.

That is, we show how simple adaptations of techniques developed in the previous section expand the scope of possibilities for \emph{online omniprediction}. 
We recall definition of online omnipredictors from \Cref{sec: link prediction}: 
\begin{definition} \label{def:omnipredictor}
    An algorithm $\cA$ is an $(\cL, \cH, \Reg)$-online omnipredictor if it generates a transcript $\{(x_t, \Delta_t, y_t)\}_{t=1}^T$ such that for all $\ell \in \cL$ there exists a $\pi_\ell: \cX \times [0,1] \rightarrow [0,1]$ such that
    \begin{align}
        \sum_{t=1}^T \E_{p_t \sim \Delta_t}\ell(x_t, \pi_\ell(x_t, p_t), y_t) \leq \inf_{h \in \cH} \sum_{t=1}^T \ell(x_t, h(x_t), y_t) + \Reg(T). \label{eq:defomnipredictor}
    \end{align}
    where the regret bound, $\Reg: \cN \rightarrow \R_{\geq0}$, is $o(T)$.
    \end{definition}

Omnipredictors were initially defined by \cite{gopalan2022omnipredictors} for the offline setting and then extended to the online case by \cite{garg2024oracle}. Intuitively, omnipredictors are efficient ``menus of optimality'': They provide a single prediction that can be postprocessed (via $\pi_\ell$) to guarantee lower loss than that achievable by any function in some comparator class $\cH$. 
Briefly, the main contribution of this section is we introduce the first algorithm which guarantees online omniprediction with respect to comparator classes $\cH$ that are real-valued and of infinite cardinality. These constructions are also unconditionally computationally efficient.

To do this, we build on the insight established by \cite{loss_oi} which shows that, in the distributional (offline) setting, given any set of losses $\cL$ and comparator class $\cH$, one can always construct a set of distinguishers $\cF(\cH, \cL)$ such that indistinguishability with respect to $\cF(\cH, \cL)$ implies omniprediction. 
We show that such a connection holds in the online setting too, and illustrate computationally efficient ways of achieving the requisite indistinguishability guarantees via the $\algname$.
\cref{thm:omniprediction-via-vovk} provides a formal statement of this general recipe or meta-theorem for online omniprediction.

The following result (\cref{thm:omnipredictionlede}) follows by using machinery of reproducing kernel Hilbert spaces to instantiate this general recipe with various choices of kernels. In the first part, we illustrate how our techniques can be used guarantee omniprediction with respect to common classes of losses and comparator classes. In the second part, we provide a different instantiation of the theorem specialized to the link prediction setting. 
Although the general framework allows for loss functions that depend on features $x$, we state the result without dependence on features for simplicity and to enable easier comparisons with prior work.

\begin{theorem} \label{thm:omnipredictionlede}
        There exists a computationally efficient kernel $k$, such that the Any Kernel algorithm run with kernel $k$ runs in polynomial time and is a $(\cH, \cL, \cO(\sqrt{(m+n^d)T}))$-omnipredictor, where
        \begin{enumerate}[(a)]
            \item  The comparator class $\cH \subseteq \{ \{-1,1\}^n \rightarrow [-1,1]\}$ contains all regression trees of depth at most $d$ and any pre-specified set of functions $\cH_0 \subseteq \{ \cX \to [-1, 1]\}$ where $|\cH_0| \leq m$.
            \item 
            The set of losses $\cL$ contains any function $\ell \; : \; [0,1] \times \{0, 1\} \to [-1, 1]$ that satisfies at least one of the following conditions:
            \begin{enumerate}[(i)]
                \item The loss $\ell$ is a continuous, differentiable proper scoring rule. That is, $p \in \pi_\ell(p)$ and $\ell \in W^{1,2}_1([0,1])$ (see \Cref{eq:sobnorm} for a formal definition of $W^{1,2}_1([0,1]$).
                \item The loss $\ell(\yhat, y)$ strongly convex in $\yhat$ and is differentiable in $\yhat$ with  $|\frac{\partial}{\partial \yhat}\partial\ell(\yhat,y)|\leq 1$.
                \item The loss $\ell$ is in a pre-speficied finite set $\cL_0  \subseteq \{ [0,1] \times \{0, 1\} \to [-1, 1]\}$ where $|\cL_0| \leq m$.
            \end{enumerate}
        \end{enumerate}
        If the problem domain is link prediction, the loss class $\cL$ may instead be a set of functions of the form $\ell_x(u) \ell_{y}(\yhat, y)$ where\footnote{Recall that, when we are discussing link prediction, $u = (a, G)$ represents an element of the universe $\cU$ where $a = (i,j)$ is an pair of individuals and $G$ is the current state of the graph detailing the existing set of edges and features for every node.} 
                \begin{enumerate}[(a)]
                    \item $\ell_x$ may be any of the tests described in \cref{sub:oi_graphs} such as indicators for any pair of group memberships or ties with embeddedness $c$ (see Equation~\ref{eq:embed}), and
                    \item $\ell_y$ may be any  function described in (b) above, or any finite set of bounded functions rewarding desirable outcomes, such as edge formation (e.g., $\ell_y(\yhat, y) = 1-y$).
                \end{enumerate}
    \end{theorem}
    %
    
    \paragraph{Comparison to prior work.}
    The results we present in this section differ from prior work both in their substance and in the techniques used to prove them.
    \cite{garg2024oracle} considers a more exacting omniprediction definition, called \textit{swap}-omniprediction, for which the function $h\in \cH$ that one compares to depends on the current prediction $p_t$.
    The paper provides an oracle-efficient algorithm that achieves $\cO(T^{7/8})$ swap regret. 
    Furthermore, they prove that $\cO(\sqrt{T})$ (or, in fact $o(T^{0.528})$) regret for online swap-omniprediction is in fact impossible.

    In the same paper, using ideas rooted in online minimax optimization \cite{lee2021online}, they introduce an algorithm which attains $\cO(\sqrt{T\log|\cH|})$ vanilla omniprediction regret for the case where $\cH$ is a finite set of binary valued functions and $\cL$ consists on proper scoring rules or bimonotone loss functions.\footnote{Informally, bimonotone losses are those which satisfy $\ell(\pi_\ell(p),1) = \ell(1,1)$ and $\ell(\pi_\ell(p),0) = \ell(0,0)$. See \cite{garg2024oracle}.} Their algorithm relies on enumerating the functions in $\cH$, and hence has runtime that is linear in $\cH$. 

    In recent, independent work, \cite{hu2024} also introduce new omniprediction algorithms for the offline case where $\cH$ consists of generalized linear models and $\cL$ consists of matching losses. These results are complementary to ours. 
    To the best our knowledge, our work is the first  to attain $\cO(\sqrt{T})$ regret for vanilla online omniprediction over: $a)$ comparator classes $\cH$ that are of infinite size or which map onto real values and $b)$ arbitrary, bounded losses $\ell$.

    \paragraph{Outline of the section and preliminaries.} In \Cref{sub:techresults}, we present our main technical results regarding online omniprediction. These rely on the ability to achieve certain online indistinguishability conditions using kernels. We illustrate how to achieve these in \Cref{sub:dcexs,sub:hpexs,sec:product-loss}. Then, in \Cref{subsec:regression} and  \Cref{subsec:performativity} we discuss implications of these results for online regression and performative prediction. 
    Finally, in \Cref{sub:omnilinkpred}, we apply our new technical machinery to the problem of link prediction in a social network.
    
    Before moving on, we review several pieces of notation that we will repeatedly reuse during this section. Given a loss function $\ell$, we will use $\partial \ell$ to refer to its discrete derivative:
    \begin{align*}
        \partial \ell (x, \yhat) = \ell(x, \yhat, 1) - \ell(x, \yhat, 0).
    \end{align*}
    Given a set of losses $\cL$, we analogosly use $\cL$ to refer to the set of discrete derivatives: 
    \begin{align*}
        \partial \cL \defeq \{ \partial \ell \; | \; \ell \in \cL \}
    \end{align*}
    Throughout our presentation, we will take always take the post-processing function $\pi_\ell$ to be 
    \begin{align}
    \label{eq:piell_argmin}
        \pi_\ell(x, p) &\in \argmin_{\hat{y} \in [0, 1]} \E_{y \sim \mathrm{Ber}(p)} [\ell(x, \hat{y}, y)] = \argmin_{\hat{y} \in [0, 1]} p \cdot \ell(x, \hat{y}, 1) + (1-p) \cdot \ell(x, \hat{y}, 0).
    \end{align}

    Lastly, we also use the fact that there exists an RKHS for the set of smooth functions over the unit interval. The following observation follows from the fact that the functions in $W_B^{1,2}([0,1])$ are a subset of the well-known Sobolev kernel. See \Cref{ex:sobolev} for more details.

    \begin{fact}
        Define $W_B^{1,2}([0,1])$ with parameter $B$ to be the set of continuous, differentiable functions $g:[0,1] \rightarrow [-1,1]$ satisfying
         \begin{align}
         \label{eq:sobnorm}
            \int_{0}^1 g(t)^2 dt + \int_{0}^1 g'(t)^2 dt \leq B^2.
        \end{align}
        $W_B^{1,2}([0,1])$ is contained in the Sobolev space $W^{1,2}([0,1])$. That is, there exists an efficiently computable kernek $k$ with RKHS $\cF_k$ such that $W_B^{1,2}([0,1]) \subset \cF_k$ and for all $f \in W_B^{1,2}([0,1])$ it holds $\| f \|_{W^{1,2}([0,1])} \leq B$ and  $\sup_t k(t, t) \leq \sqrt{3}$.
    \end{fact}
    \subsection{Efficient, $\sqrt{T}$ online omiprediction with respect to rich comparison classes $\cH$.} 
    \label{sub:techresults}

    In this subsection, we present our main result demonstrating how outcome indistinguishability implies omniprediction in the online setting and illustrating how these indistinguishability conditions can be efficiently achieved via the $\algname$.

    The following two OI definitions, hypothesis and decision OI, were first introduced (in the batch setting) by \cite{loss_oi}. We now adapt them to the online case. 
    Decision outcome indistinguishability (DOI) is defined with respect to a class of losses $\cL$. It states that prediction must be approximately indistinguishable with respect to the class of test functions constructed from pairs of loss functions $\ell \in \cL$ and post-processed predictions $\pi_\ell$:
    
    \begin{definition}[Decision OI]
    \label{def:online_decision_oi}
        For a loss class $\cL$ and regret bound $\Rarg{\DOI}(T)$, an algorithm satisfies $(\cL, \Rarg{\DOI}(T))$-\textit{decision outcome indistinguishability (DOI)} if it generates a transcript $\{(x_t, \Delta_t, y_t)\}_{t=1}^T$ such that,
        \begin{align}
            \left|\sum_{t=1}^T \E_{p_t \sim \Delta_t} \partial \ell(x_t, \pi_\ell(x_t, p_t))(p_t - y_t) \right|& \leq \Rarg{\DOI}(T), && \forall \ell \in \cL. 
        \end{align}
    \end{definition}
    The second OI condition, hypothesis outcome indistinguishability (HOI), requires that predictions must be approximately indistinguishable with respect to functions constructed from pairs of comparator functions $h \in \cH$ and loss functions $ \ell \in \cL$:
    \begin{definition}[Hypothesis OI]
    \label{def:online_hypothesis_oi}
        For a loss class $\cL$, comparator class $\cH$, and regret bound $\Rarg{\HOI}(T)$, an algorithm satisfies $(\cL, \cH, \Rarg{\HOI}(T))$-\textit{hypothesis outcome indistinguishability (HOI)} if it generates a transcript $\{(x_t, \Delta_t, y_t)\}_{t=1}^T$ such that:
        \begin{align}
    	\left|\sum_{t=1}^T \E_{p_t \sim \Delta_t}\partial \ell(x, h(x_t))(p_t - y_t) \right|&\leq \Rarg{\HOI}(T),   && \forall (\ell, h) \in \cL \times \cH. 
        \end{align}
    \end{definition}

    Having introduced these two definitions, the result that OI implies omniprediction is almost immediate. The following lemma formally adapts the ideas from \cite{loss_oi} to the online setting.
    
    \begin{lemma} \label{lem:onlinelossoi}
    Fix a comparator class $\cH \subseteq \{ \cX \to [0, 1]\}$, a class of losses $\cL \subseteq \{ \cX \times [0, 1] \times \{ 0, 1\} \to \cR \}$ and regret bounds $\Rarg{\DOI}(T), \Rarg{\HOI}(T) \; : \; \cN \to R$.  If an algorithm $\cA$ satisfies 
    \begin{enumerate}
        \item  $(\cL, \Rarg{\DOI}(T))$-decision OI (\Cref{def:online_decision_oi})
        \item and $(\cL, \cH, \Rarg{\HOI}(T))$-hypothesis OI (\Cref{def:online_hypothesis_oi}),
    \end{enumerate}
    then, $\cA$ is an $(\cL, \cH, \Rarg{\DOI}(T) + \Rarg{\HOI}(T))$-online omnipredictor.
    \end{lemma}
    \begin{proof}
 First, we observe that for all $x \in \cX$ and any pair $(\yhat, y)$ where $y \in \{0,1\}$:
    \begin{align*}
        \ell(x, \yhat, y) = y \ell(x, \yhat, 1) + (1-y) \ell(x, \yhat, 0)  = y(\ell(x, \yhat, 1) - \ell(x, \yhat, 0)) + \ell(x, \yhat, 0) 
    \end{align*}
    A similar expression holds for the following expectation version, 
    \begin{align*}
        \E_{y \sim \Ber(p)} \ell(x, \yhat, y) = p \ell(x, \yhat, 1) + (1-p) \ell(x, \yhat, 0)  = p(\ell(x, \yhat, 1) - \ell(x, \yhat, 0)) + \ell(x, \yhat, 0).
    \end{align*}
    Therefore, 
    \begin{align*}
        | \ell(x, \yhat, y) -  \E_{\yt \sim \Ber(p)} \ell(x, \yhat, \yt) | = |(y - p)(\ell(x, \yhat, 1) - \ell(x, \yhat, 0))| = |(y - p)\partial \ell(x, \yhat)|
    \end{align*}
    Using this decomposition, by the Decision OI guarantee \Cref{def:online_decision_oi}, we know that 
    \begin{align*}
      \sum_{t=1}^T \ell(x_t, \pi_\ell(x_t, p_t), y_t)  \leq     \sum_{t=1}^T \E_{\yt_t \sim \Ber(p_t)}\ell(x_t, \pi_\ell(x_t,p_t), \yt_t)  + \Rarg{\DOI}(T).
    \end{align*}
    Furthermore, since $\pi_\ell$ is the argmin (see \Cref{eq:piell_argmin}), by definition it satisfies the following inequality for any $h$,
    \begin{align*}
        \E_{\yt_t \sim \Ber(p_t)}\ell(x_t, \pi_\ell(x_t, p_t), \yt_t) \leq \E_{\yt_t \sim \Ber(p_t)}\ell(x_t, h(x_t), \yt_t)
    \end{align*}
    Lastly, by the Hypothesis OI guarantee (\Cref{def:online_hypothesis_oi}), 
    \begin{align*}
      \sum_{t=1}^T \E_{\yt_t \sim \Ber(p_t)}\ell(x_t, h(x_t), \yt_t)  \leq     \sum_{t=1}^T\ell(x_t,  h(x_t), y_t)  + \Rarg{\HOI}(T).
    \end{align*}
    Combining all three inequalities, we get our desired result:
    \begin{align*}
        \sum_{t=1}^T \ell(x_t,  \pi_\ell(x_t, p_t), y_t)  &\leq \sum_{t=1}^T\ell(x_t,  h(x), y_t) + \Rarg{\DOI}(T) +  \Rarg{\HOI}(T). \quad \forall h \in \cH \qedhere
    \end{align*}
    \end{proof}

    The advantage of this loss OI viewpoint is that it provides a neat template for algorithm design. More specifically, to achieve omniprediction, we only need to design kernels whose corresponding RKHS contain the required distinguishers and then run the $\algname$ with these kernels. 
    While the main idea is simple, to prove a formal non-asymptotic regret bound we also need to ensure that corresponding function norms of the distinguishers $\fnorm{f}$ and feature norms $k((x,p), (x,p)) = \fnorm{\Phi(x,p)}^2$ are appropriately bounded. If these quantities are not appropriately bounded, then the guarantees from the $\algname$ can become vacuous (recall the bound from \Cref{thm:indistinguishability_main}).

    To address this issue, we further specialize the OI definitions above to the RKHS domain. These specializations, kernel decision and hypothesis OI, are representational conditions on the kernel $k$ and the corresponding RKHS $\cF_k$. Intuitively, they require that a kernel $k$ be efficiently computable, bounded, and that certain functions are contained (and have small norm) in $\cF_{k}$. 

    \begin{definition} [Kernel Decision OI]\label{def:dc}
        Let $\cL$ be a set of loss functions. A kernel $k$ with corresponding RKHS $\cF$ is $\cL$-kernel decision OI (KDOI) with parameter $B$ if, 
        \begin{align}
            \{ \partial \ell \circ \pi_{\ell} \; | \; \ell \in \cL \} \subseteq \cF  \subseteq \{\cX \times [0,1] \rightarrow \R\}\label{cond:ddc},
        \end{align}
        where $\partial \ell \circ \pi_{\ell} (x,p) = \ell(x, \pi_\ell(p),1) - \ell(x, \pi_\ell(p),0)$ and: 
         \begin{align*}
            \sqrt{\sup_{\ell \in \cL} \| \partial \ell \circ \pi_{\ell} \|^2_{\cF} \cdot  
             \sup_{x \in \cX, p \in [0, 1]} k((x, p), (x, p)) } \leq B. 
        \end{align*}
    \end{definition}
    
The condition states that the composition of the discrete derivative of each loss composed with its post-processing function is in the corresponding RKHS and that both the function $\| \partial \ell \circ \pi_{\ell} \|^2_{\cF}$ and feature norms $ k_{\cL}((x, p), (x, p)) = \|\Phi(x,p)\|_{\cF}^2$ are uniformly bounded.
    We note that, by \Cref{lem:scalarmult}, if a function $f $ is in $ \cF$, then so is its negation, $-f$ (RKHSs are closed under scalar multiplication).
    Thus, a sufficient condition for KDOI is that $\ell(\pi_\ell(\cdot), y) \in \cF$ for all $\ell \in \cL, y \in \{ 0, 1 \}$.
    Next, we define an analogous condition for losses composed with comparator functions. 

    \newcommand{\hypcon}{Hypothesis-contained}
    \newcommand{\hc}{HC}
    \newcommand{\deccon}{Decision-contained}
    \newcommand{\dc}{DC}
    \begin{definition}[Kernel Hypothesis OI]\label{def:hc}
        Let $\cH$ be a comparator class and let $\cL$ be a class of loss functions. A kernel $k$ with corresponding RKHS $\cF$  satisfies $(\cL,\cH)$-kernel hypothesis OI (KHOI) with parameter $B$ if,
        \begin{align}
            \{ \partial \ell \circ h \; | \; h \in \cH, \ell \in \cL \} \subseteq \cF \subseteq \{\cX \times [0,1] \rightarrow \R\},
        \end{align}
where $\partial \ell \circ h (x) = \ell(x, h(x),1) - \ell(x, h(x),0)$ and 
        \begin{align*}
             \sqrt{\sup_{h \in \cH, \ell \in \cL} \| \partial \ell \circ h \|^2_{\cF} \cdot  
             \sup_{x \in \cX, p \in [0, 1]} k((x,p),(x,p))} \leq B. 
        \end{align*}
    \end{definition}
    As in the previous setting, a sufficient condition for KHOI is that $\ell(h(\cdot), y) \in \cF$ for all $h \in \cH, \ell \in \cL, y \in \{ 0, 1 \}$. 
    We also note that the kernel version of decision and hypothesis OI are qualitatively different from other conditions in the omniprediction literature, since they allow for infinite and real-valued comparison classes but require the existence of a suitable RKHS containing compositions of loss, post-processing and comparator functions.

    With these definitions in hand, we can now state our main theorem which provides a general recipe for online omniprediction via the $\algname$.
    
    \begin{theorem}[Corollary to \Cref{lem:onlinelossoi}]
    \label{thm:omniprediction-via-vovk}
    Let $\cH \subseteq \{ \cX \rightarrow [0,1] \}$ be a class of comparison functions and let $\cL \subseteq \{ \cX \times [0,1] \times \{0,1\} \rightarrow \R\}$ be a set of losses. 
    
    Let $k_{\cL}$ and $ k_{\cL, \cH}$ be efficient kernels with corresponding RKHSs $\cF_{\cL}$ and $\cF_{\cL, \cH}$ that satisfy $\cL$-KDOI and $(\cL, \cH)$-KHOI with parameters $B_{\mathrm{KDOI}}$ and $B_{\mathrm{KDOI}}$. Then, the $\algname$ with kernel $k_{\cL} + k_{\cH, \cL}$ runs in polynomial time and is an $(\cL, \cH, 2(B_{\mathrm{KDOI}} + B_{\mathrm{KHOI}}) \sqrt{T})$-online omnipredictor.
    \end{theorem}

    \begin{proof}
       Define the function $k \defeq k_{\cL}  + k_{\cL, \cH}$.
       From \cref{lem:kernelsum}, it holds that $k$ is a kernel and that the functions  
       \begin{align*}
           \{ f_1 + f_2 \; | \; f_1\in \cF_{\cL}; f_2 \in \cF_{\cL,\cH} \}
       \end{align*}
       are in the corresponding RKHS, which we will call $\cF$.
       Also, since $k_{\cL}$ and $k_{\cL,\cH}$ can be evaluated in polynomial time, so can $k$, which implies that the $\algname$ runs in polynomial time.
       
       Now, by the fact that $\cF_{\cL}$ and $\cF_{\cL,\cH}$ are closed under scalar multiplication (by \cref{thm:moore-aronszajn}), the zero function is in $\cF_{\cL}$ and $\cF_{\cH, \cL}$. 
       This implies for all $h \in \cH$ and $\ell \in \cL$, we have that $\partial \ell \circ \pi_{\ell} \in \cF$  and $\partial \ell \circ h \in \cF$, since $\partial \ell \circ \pi_{\ell} = \partial \ell \circ \pi_{\ell} + 0$ and $\partial \ell \circ h = 0 + \partial \ell \circ h$.

       Now by the main guarantee for the $\algname$, since we've assumed that norms and kernels are bounded, we have that,
       \begin{align*}
           \left |{\sum_{t=1}^T \E_{p_t \sim \Delta_t}(p_t - y_t)(\partial \ell \circ \pi_\ell )(x_t, p_t) } \right| &\leq B_{\mathrm{KDOI}} \sqrt{1 + \sum_{t=1}^T \E_{p_t \sim \Delta_t} p_t (1-p_t)}\leq B_{\mathrm{KDOI}}\sqrt{1+ \frac{1}{4}T},\\
           \left| {\sum_{t=1}^T \E_{p_t \sim \Delta_t}(p_t - y_t)(\partial \ell \circ h )(x_t) } \right| &\leq B_{\mathrm{KHOI}} \sqrt{1 + \sum_{t=1}^T \E_{p_t \sim \Delta_t} p_t (1-p_t)} \leq B_{\mathrm{KHOI}} \sqrt{1 + \frac{1}{4}T},
       \end{align*}
       which, by \cref{lem:onlinelossoi}, implies the theorem.
    \end{proof}
   \paragraph{Discussion.} We note that the above theorem establishes a precise, non-asymptotic regret bound. It in particular guarantees that for any $\ell \in \cL$, 
   \begin{align*}
       \frac{1}{T} \sum_{t=1}^T \E_{p_t \sim \Delta_t} \ell(x_t, p_t, y_t) &\leq \min_{h \in \cH} \frac{1}{T} \sum_{t=1}^T \ell(x_t, h(x_t),y_t) + \frac{(B_{\mathrm{KDOI}} + B_{\mathrm{KHOI}})\sqrt{1 + \sum_{t=1}^T \E_{p_t \sim \Delta_t} p_t (1-p_t)}}{T} \\
       & \leq \min_{h \in \cH} \frac{1}{T} \sum_{t=1}^T \ell(x_t, h(x_t),y_t) + 2\frac{B_{\mathrm{KDOI}} + B_{\mathrm{KHOI}}}{\sqrt{T}}
   \end{align*}
   \emph{for every} value of $T$ greater than 1. Note that the bound adapts to the variance of the predictions $p_t$.
   Furthermore, the algorithm is very simple and easy to implement. As presented previously in \Cref{sec: the algorithm}, you only need to be able to evaluate the kernel and solve a small binary search problem at every iteration.
    In the next sections, we instantiate our results for several common comparator and loss classes and show how the relevant parameters $B_{\mathrm{KHOI}}$ and $B_{\mathrm{KDOI}}$ are reasonably bounded in natural settings.
    
    More specifically, in \cref{sub:dcexs}, we demonstrate how to construct kernels that satisfy KDOI and in \cref{sub:hpexs}, we demonstrate how to construct kernels to satisfy KHOI.
    Since the kernels for each condition can be constructed separately and then combined (added) to create a kernel to pass into the $\algname$ that satisfies both conditions jointly, the constructions in each subsection can be mixed and matched according to the prediction problem at hand. 

    \subsection{Loss classes satisfying kernel decision OI.}\label{sub:dcexs}

    In this subsection, we present several broad classes of loss functions satisfying kernel decision OI, which says that the composition of the discrete derivatives of loss functions with their associated post-processing functions must be in an RKHS and have bounded function and feature norms.

    Throughout these next two subsections, we restrict our attention to a particular class of losses: those that depend only on decisions $\yhat$ and outcomes $y$, and not on features $x$. We will call these loss classes \textit{feature-invariant}.
    This is the typical setting for omniprediction in prior work \cite{gopalan2022omnipredictors, garg2024oracle} (and for loss or regret minimization).
    Since all of the loss functions in this section will be assumed to be invariant to the feature vectors, we will drop $x$ from the notation and consider $\cL \subset \{ [0,1] \times \{0, 1\} \to \cR \}$. 
    We will also drop the argument for $x$ from each post-processing function $\pi_\ell$. Later on, in \Cref{sec:product-loss}, we will bring the dependence on $x$ back in when we generalize these constructions to \emph{separable} losses.
    
    \paragraph{A naive strategy.} A first attempt to achieve kernel decision OI is to find a rich, expressive RKHS $\cF$ such that $\partial \ell \in \cF$ then hope that the composition $\partial \ell \circ \pi_\ell$ is also in $\cF$.\footnote{Recall that $\partial \cL$ is defined as the set $\{ \partial \ell \; | \; \ell \in \cL \}$, and $\partial \ell(x, p)$ is defined as $\ell(x, p, 1) - \ell(x, p, 0)$.}
    In fact, it is generally straightforward to find such RKHSs that contain $\partial \ell$ for many natural loss classes.
    For example, the set of losses where $\ell(\yhat, y)$ is Lipschitz in $\yhat$ for each $y \in \{0, 1\}$ is contained in an RKHS. This is the Sobolev space mentioned in the preliminaries of this section. 
    Lipschitz loss functions include squared/absolute error on a bounded domain, Huber, exponential, and the hinge loss, among others.

    Unfortunately, the mere fact that $\partial \cL$ is contained in an RKHS $\cF$ does not imply that $\partial \ell \circ \pi_\ell$ is in $\cF$.
  \Cref{prop:counterexample} shows a formal counterexample for the case where $\cF$ is the Sobolev space. 

    \begin{proposition} \label{prop:counterexample}
        There exists a kernel $k$ with RKHS $\cF$ and a set of losses $\cL$ such that $\partial \cL \subseteq \cF$, but
        $$\{ \partial \ell \circ \pi_\ell \; | \; \ell \in \cL \} \not \subseteq \cF.$$
    \end{proposition}
    
    \begin{proof}
    Let $\cL \subseteq \{[0,1] \times \{0,1\} \rightarrow \R\}$ be the set of functions that just depend on $\yhat$ and $y$  such that for all $\ell \in \cL$ and $y \in \{0, 1 \}$, $\ell(\cdot, y) \; : \; [0, 1] \to \cR$ is differentiable and for which both $\ell$ and its derivative with respect to $\yhat$ are square integrable over $[0,1]$:

    \begin{align*}
        \int_{0}^1 \ell(t,y)^2 dt + \int_{0}^1 \ell'(t,y)^2 dt < \infty
    \end{align*}

    Notice that $\partial \cL$ is the Sobolev space $W^{1,2}([0,1])$, which is an RKHS that has an efficient kernel. (See \cref{ex:sobolev} for a definition of Sobolev spaces relevant to our context.) 
    We will show that the postprocessing of a function $\partial \ell \circ \pi_\ell \in \cL$ may \textit{not} be in the Sobolev space.
    Take $\ell(x, \yhat,1) = -(\yhat - 1/2)^2$ and $\ell(x, \yhat,0) = (\yhat - 1/2)^2$, which are each in $W^{1,2}([0,1])$.
    Next, we will argue the postprocessing $\pi_\ell$ is not a continuous function of $p$. In particular,
    \begin{align*}
        \pi_\ell(x, p) &=  \argmin_{\yhat \in [0,1]} \;\; p \cdot (-(\yhat - 1/2)^2) + (1 - p) \cdot (\yhat - 1/2)^2 \\ 
         &  =  \argmin_{\yhat \in [0,1]} \;\; (1 - 2p) (\yhat - 1/2)^2.
    \end{align*}
    is discontinuous in $p$. In particular for $p  < 1/2$, the function evaluates to $c (\yhat
 - 1/2)^2$ for some $c > 0$ and hence is minimized at $1/2$. For $p  > 1/2$ the function evaluates to $c(\yhat - 1/2)^2$ for some $c < 0$ and is hence minimized at either of the end points $\{0,1\}$. 
    Then, 
    \begin{align*}
        \partial \ell \circ \pi_\ell(x, p) = \begin{cases}
            0 & \text{if } p < 1/2, \text{ and} \\
            -1/2 & \text{otherwise,}
        \end{cases}
    \end{align*}
    which is discontinuous and hence not in the Sobolev space since the space only contains continuous functions.
    \end{proof}

    Thus, additional conditions on $\partial \cL$ are necessary to ensure that $\partial \cL \subseteq \cF$ implies KDOI. 
    In  our main result in this subsection, \cref{prop:dc}, we identify natural conditions on $\cL$ which do guarantee decision OI:
    
    \begin{proposition} \label{prop:dc} 
    The following statements are true:

    %
    \begin{enumerate}[(1)]
        \item Let $\cL_{\mathrm{PS}}$ be the set of continuous and differentiable proper scoring rules $\ell(\yhat, y)$. That is, 
        \begin{align*}
            \cL_{\mathrm{PS}} = \{ \ell(\yhat, y): p \in \pi_\ell(p), \;  \partial \ell \in W^{1,2}_1([0,1])  \}
        \end{align*}

        Then, there exists an efficient kernel $k_{\mathrm{PS}}$ satisfying $\cL_{\mathrm{PS}}$-KDOI with parameter $B_{\mathrm{KDOI}} \leq \sqrt{3}$. 
       
        \item Let $\cL_{\mathrm{SC}}$ be the set of continuous, smooth, strongly convex losses $\ell(\yhat, y)$. That is, 
        \begin{align*}
            \cL_{\mathrm{SC}} = \{\ell(\yhat, y): \ell(\yhat, y) \text{ is } \gamma \text{ strongly convex in } \yhat, \; \forall \; y \in \{0,1\}, \text{ and } \abs{\frac{d}{d \yhat} \ell(\yhat, y)} \leq 1  \}
        \end{align*}
        Then, there exists an efficient kernel $k_{\mathrm{SC}}$ satisfying $\cL_{\mathrm{SC}}$-KDOI with parameter $B_{\mathrm{KDOI}} \leq 2\sqrt{3} (3 + 2 \gamma^{-1}) $. 

        \item Let $\cL_m= \{\ell_1(\yhat, y), \dots, \ell_m(\yhat, y)\}$ be any finite set of bounded functions with $|\cL_m| \leq m$. Then, there exists an efficient kernel $k_m$ satisfying $\cL_m$-KDOI with parameter $B_{\mathrm{KDOI}} \leq m$.
    \end{enumerate}
    Moveover, if $\cL = \cL_{\mathrm{PS}} \cup \cL_{\mathrm{SC}} \cup \cL_{\mathrm{m}}$, then the efficient kernel $k = k_{\mathrm{PS}}+ k_{\mathrm{SC}}+ k_m$ satisfies KDOI with constant $4\sqrt{3m}(2 + \gamma^{-1}) $. 
    \end{proposition}
    
    \begin{proof}
        We prove that each of the statements separately.
        Then, applying \cref{lem:kernelsum}, which says that the union of RKHSs is an RKHS associated with the sum of each kernel function, implies the last statement. \\
    
    \noindent \emph{Proof for $\cL_{\mathrm{PS}}$.} 
    $\pi_\ell$ is the identity function, so $\partial \ell \circ \pi_\ell = \partial \ell$. The result follows from the assumption that $\partial \ell$ is in $W^{1,2}([0,1])$ (see \cref{ex:sobolev} for discussion) and the function norm is bounded by $1$ and the feature norms are bounded by $3$. \\
    
    \noindent \emph{Proof for $\cL_{\mathrm{PS}}$.}
    Our strategy will be to show that $\Pi_\cL = \{ \pi_\ell \; : \; \ell \in \cL\}$ consist of functions in the Sobolev space $W^{1,2}([0,1])$. Then, we will apply \cref{lem:cdbfd}, which states that the composition of functions in a Sobolev space $W^{1,2}([0,1])$ are in the space and the norm of the composition of functions in the space with bounded norm is bounded. 

    The convexity of $\ell$ in its second argument implies $\ell$ is differentiable almost everywhere and continuous. 
    This implies that the discrete derivative function $\partial \ell$ is differentiable almost everywhere and continuous, which implies that $\partial \ell$ is in $W^{1,2}([0,1])$.
    Also, since $\abs{\frac{d}{d \yhat} \ell(\yhat, y)} \leq 1$ and the range of $\ell$ is in $[-1, 1]$
    \begin{align*}
        \| \partial \ell \|_{W^{1,2}([0,1])} &\leq \|  \ell(\cdot, 0) \|_{W^{1,2}([0,1])} + \|  \ell(\cdot, 1) \|_{W^{1,2}([0,1])} \\
        &\leq 4 
    \end{align*}
    
    Next, we show that $\pi_\ell$ is a Lipschitz function of $p$. 
    The intuition is that, since $\ell$ is strongly convex, it has a unique minimum, and small changes to $p$ cannot induce large changes in $\pi_\ell$.
    Lipschitzness of $\pi_\ell$ implies $\pi_\ell \in W^{1,2}([0,1])$ since Lipschitz functions are absolutely continuous and hence differentiable almost everywhere and equal to their Lebesgue integral almost everywhere.
    The proof of Lipschitzness follows by using the same analysis used in Theorem 3.5 
 of~\cite{perdomo2020performative} (albeit with slightly different assumptions). Let $p$ and $\tilde p$ be two different predicted probabilities in $[0,1]$. Also, define:
    \begin{align}
        f(\yhat ) = p \cdot \ell(\yhat, 1) + (1-p) \cdot \ell(\yhat, 0) \\ 
        \tilde{f}(\yhat) = \pt \cdot \ell(\yhat, 1) + (1-\pt) \cdot \ell(\yhat, 0)
    \end{align}
    and $f' = {\partial f}/ {\partial \yhat }$.
    With this notation, we have that $\pi_\ell(p) \in \argmin_{\yhat} f(\yhat)$ and likewise $\pi_\ell(\pt) = \argmin_{\yhat} \tilde{f}(\yhat)$. First, we have that,
    \begin{align*}
    f(\pi_\ell(p)) - f(\pi_\ell(\pt)) &\geq  f'(\pi_\ell(\pt)) (\pi_\ell(p) - \pi_\ell(\pt)) + \frac{\gamma}{2} (\pi_\ell(p) - \pi_\ell(\pt))^2 \\ 
    f(\pi_\ell(\pt)) - f(\pi_\ell(p)) &\geq \frac{\gamma}{2} (\pi_\ell(p) - \pi_\ell(\pt))^2, 
    \end{align*}
    where the first line follows by strong convexity of $f$, and the second line follows by strong convexity of $f$ and the fact that $\pi_\ell(p)$ is the unique minimizer of $f$ so $f'(\pi_\ell(p)) = 0$. Combining these two inequalities, we get that: 
    \begin{align}
    \label{eq:pp_1}
        -\gamma (\pi_\ell(p) - \pi_\ell(\pt))^2 \geq f'(\pi_\ell(\pt)) (\pi_\ell(p) - \pi_\ell(\pt)).
    \end{align}
    Next, we derive a lower bound for $f'(\pi_\ell(\pt)) (\pi_\ell(p) - \pi_\ell(\pt))$ in terms of $p, \tilde p$. Observe that, by definition, 
    \begin{align*}
        f'(\pi_\ell(\pt)) - \tilde{f}'(\pi_\ell(\pt)) = (p - \pt) \ell'(\pi_\ell(\pt), 1) + (1-p - (1 - \pt)) \ell'(\pi_\ell(\pt), 0).
    \end{align*}
    Hence, 
        $|f'(\pi_\ell(\pt)) - \tilde{f}'(\pi_\ell(\pt))| \leq 2 |p-\pt|$.
    Then, we get that, 
    \begin{align*}
        (\pi_\ell(p) - \pi_\ell(\pt))f'(\pi_\ell(p)) &\geq  (\pi_\ell(p) - \pi_\ell(\pt))f'(\pi_\ell(p)) - (\pi_\ell(p) - \pi_\ell(\pt))\tilde{f}'(\pi_\ell(\pt)) \\ 
        & \geq |\pi_\ell(p) - \pi_\ell(\pt)|\cdot  |f'(\pi_\ell(\pt)) - \tilde{f}'(\pi_\ell(\pt))| \\ 
        & \geq -2 |\pi_\ell(p) - \pi_\ell(\pt)| \cdot |p - \pt|
    \end{align*}
    where the first line follows from the fact that $\tilde f'(\pi_\ell(\tilde p)) = 0$, and the second line follows from the first order optimality conditions for convex functions, 
        $(\pi_\ell(p) - \pi_\ell(\pt))\tilde{f}'(\pi_\ell(\pt)) \geq 0$.
    Combining this last chain of inequalities with \cref{eq:pp_1}, we get that 
    \begin{align*}
     -\gamma (\pi_\ell(p) - \pi_\ell(\pt))^2  \geq -2 |\pi_\ell(p) - \pi_\ell(\pt)| \cdot |p - \pt|. 
    \end{align*}
    After simplifying and rearranging, we get $|\pi_\ell(p) - \pi_\ell(\pt)| \leq 2 \gamma^{-1} |p - \pt|$, so $\| \pi_\ell \|_{W^{1,2}([0,1])} \leq 2(1 + 2 \gamma^{-1})$.  Finally, using the kernel associated with $W^{1,2}([0,1])$, 
    the feature norm is upper bounded by $3$.
    \\ 

    \noindent \emph{{Proof for $\cL_{(3)}$.}} We apply \cref{thm:finite-membership-rkhs}, which says that finite sets of functions taking values in $[-1, 1]$ are in an RKHS with function and feature norms bounded by 1. Let the $\cX$ in the lemma be $ [0,1]$ and let $\cC = \cL_{(3)}$. Denote the induced RKHS $\cF$. Then the lemma implies that $\| \ell \|_{\cF} \leq 1$, and by the fact that $| \cF_{(3)}| \leq m$ and losses are assumed to be bounded in $[-1, 1]$, the feature norm must be bounded by $m$. 
    \end{proof}

    Intuitively, the previous says that if a loss class satisfies common regularity conditions like truthfulness (i.e. a proper scoring rule), smoothness/convexity, or is finite, then there exists a kernel satisfying KDOI. 
    Additionally, it says that we can combine any sets of losses satisfying the above conditions and still satisfy KDOI.
    Notice that the Sobolev proper scoring losses include, for example, squared error, while the continuous, smooth and strongly convex losses $\cL_{\mathrm{SC}}$ include ($\ell_2$ regularized) absolute error, Huber loss,  and exponential loss.
    Losses that don't fit into the previous categories, such as the truncated cross-entropy loss, the 0-1 loss or the hinge loss may be included in the finite set of losses $\cL_m$.
\subsection{Comparator and loss classes satisfying kernel hypothesis OI.} \label{sub:hpexs}

    Having analyzed how one can guarantee kernel \emph{decision} OI with respect to common classes of losses, we now move only to analyze pairs $\cL, \cH$ that satisfy kernel \emph{hypothesis} OI. That is, we aim to design kernels $k$ with functions spaces $\cF$ such that the functions $\ell \circ h \in \cF$ (see \Cref{def:hc}).

    \paragraph{Regression trees.} 
    
        Our first result in this section shows one can guarantee kernel hypothesis OI for the class $\cH$ of bounded-depth regression trees on binary features (an infinite comparator class) and $\cL$ that consists of all bounded losses functions:

        \begin{proposition}
        \label{prop:dectree}
            Let $\cH \subseteq \{\{\pm1\}^n \rightarrow \R\}$ be the set of all regression trees of depth at most $d \in \N$ over the boolean hypercube an let $\cL$ be a set of all loss functions $\ell(\yhat, y)$ bounded in $[-1, 1]$. 
            There exists an computationally efficient kernel satisfying  $(\cL, \cH)$-KHOI with parameter $B$ bounded by $(n+1)^{d/2} \cdot 2^d$.
        \end{proposition}

        \begin{proof}
            We first note that regression trees on binary features are low degree polynomials, which are contained an RKHS $\cF$ associated with the degree $d$ polynomial kernel (see \cref{ex:polynomialkernels} for a definition and discussion of polynomial kernels).
            
            To see this, we can write each tree in the following form: 
            For a given regression tree, let $b \in \{0, 1\}^d$ represent the path down the regression tree with $m$th element $b_m$. 
            Let $c_b$ be the leaf value assigned to path $b$. Let $i_{b, j}$ represent the index of the decision variable at the $j$th decision on path $b$. Then, any regression tree can be written in terms of $\{ c_b \}_{b \in \{ 0, 1 \}^d}$ and $\{ i_{b, j} \}_{b \in \{ 0, 1 \}^d, j \in [d]}$:
            \begin{align} 
                h(x) = \sum_{b \in \{ 0, 1 \}^d} c_b \prod_{m = 0}^{d-1} \paren{(1- x_{i_{b, m}})(1 - b_m) + x_{i_{b, m}} b_m} \label{eq:tree2}
            \end{align}
            By distributing each product, combining like terms, and using the notation $x_{I} \defeq \prod_{i \in I} x_i$, we can recover the following more concise expression: 
            \begin{align} 
                h(x) = \sum_{I \in \cI} a_I x_I \label{eq:tree}
            \end{align}
            where $\cI \subseteq 2^{\{ 0, 1\}^d}$, $a_I \in \R$ for all $I \in \cI$.
            Moreover, the latter form reveals that each nonzero $a_I$ corresponds to some $I$ with no more than $d$ terms. 
            Thus, $\cH \subseteq \cF$.
            (See Definition 3.13 in \cite{odonnel2021analysis} for more discussion of representing decision trees on Boolean inputs as polynomial functions.)

            Next, notice that functions $\ell(h(\cdot), 1)$ and $\ell(h(\cdot), 0)$ for $\ell \in \cL$ and $h \in \cH$ can themselves be written as depth-$r$ regression trees by taking each leaf value $c_b$ of $h$ and replacing it with $\ell(c_b, 0)$ and $\ell(c_b, 1)$, respectively. That is, for each $h \in \cH$, we create two new trees $h_0, h_1 \in \cF$ to be $h$ with its leaf values replaced with the corresponding value of $\ell(c_b, y)$ for $y\in \{0,1\}$.
            Finally, using \cref{lem:kernelsum} and \cref{lem:scalarmult}, this implies that $\{ \partial \ell \circ h \; | \: h \in \cH, \ell \in \cL \} \subseteq \cF$.
            %

            Since there are $2^d$ leaves and each leaf has absolute value bounded by 1, $\| h_y \|_\cF \leq 2^d$.
            Also, since the kernel function associated with $\cF$ is $(1 + \langle x, x' \rangle)^d$, then $k(x,x)$ is bounded by $(1+n)^d$.
        \end{proof}

    \paragraph{Any finite set of real-valued functions $\cH$.} 
    

        In our next construction, we show how to guarantee kernel hypothesis OI for the case where $\cH$ is any finite set of comparator functions and $\cL$ is a set of losses that can be represented in an RKHS.
        
        This could of interest in setting where there are pre-specified predictors (like an existing link prediction system) that we would like the $\algname$ to compete with.

        \begin{proposition} \label{prop:finitesethyp}
            Let $\cH = \{ h_1, \dots, h_m \}$ be any finite set of real-valued functions on $\cX$ and let $\cL$ be any set of loss functions $\ell(\yhat, y)$. 
            Let $k$ be a kernel with RKHS $\cF$ such that $\cL \subseteq \cF$, $\| \ell \|_{\cF} < 1$ for all $\ell \in \cL$, and $\sup_t k(t,t) \leq 1$. Then, 
            \begin{enumerate}
                \item There exists a kernel k' that is $(\cL, \cH)$-KHOI with parameter $B_{\mathrm{KHOI}}$ at most $2\sqrt{m}$.
                \item The kernel $k'$ is computable in time at most $\cO(m \cdot \mathsf{time}(k) \cdot \mathsf{time}(\cH))$ where $\mathsf{time}(k)$ is a uniform upper bound on the runtime of the kernel $k$ and $\mathsf{time}(\cH)$ is a uniform upper bound on the runtime of computing any function $h \in \cH$.
            \end{enumerate}

        \end{proposition}
\begin{proof}

The main idea is that one can compose kernels in the following fashion. Let $k(t,t'):\R \times \R \rightarrow \R$ be a kernel with corresponding RKHS $\cF$ such that $\ell(\cdot, 1)$ and $\ell(\cdot,0)$ are both in $\cF$ for all $\ell \in \cL$.
Then, for any fixed function $h_i:\cX \rightarrow [0,1]$, the kernel $k_i:\cX \times \cX \rightarrow \cR$ defined as:
\begin{align*}
    k_i(x,x') = k(h_i(x), h_i(x'))
\end{align*}
has an RKHS $\cF_i$ which contains $\ell(h_i(x), 1)$ and $\ell(h_i(x),0)$ for all $\ell \in \cL$.  Furthermore, if the functions $\ell(\cdot, 1)$ and $ \ell(\cdot,0)$ have norm at most 1 in $\cF$, then the composed functions $\ell(h_i(\cdot), 1), \ell(h_i(\cdot),0)$  will also have norm at most 1 in $\cF_i'$
This is a neat fact from the theory of RKHSs (\Cref{lem:compositionkernel}). 

Since we can construct an RKHS for each $\partial \ell \circ h_i$ individually, we can construct an RKHS that contains all of the $h_i$ simultaneously simply by summing the individual kernels together. 

In particular, by \Cref{lem:kernelsum}, the kernel,
\begin{align}
    k'(x,x') = \sum_{h_i \in \cH} k(h_i(x), h_i(x'))
\end{align}
contains $\partial \ell \circ h = \ell(h(\cdot),1) - \ell(h(\cdot),0)$ for all $h \in \cH$ and $\ell \in \cL$. Moreover, since each $\ell(h(x),y)$ has norm at most 1 (for $y \in  \{0,1\}$), then (by the triangle inequality) the functions $\partial \ell \circ h $ have norm at most 2 in the RKHS corresponding to $k'$. Furthermore, 
\begin{align*}
    \sup_x k'(x,x) = \sum_{h \in \cH} k(h(x), h(x)) \leq m, 
\end{align*}
so the kernel $k'$ is $(\cL,\cH)$-KHOI with parameter $B_{\mathrm{KHOI}}$ bounded by $2\sqrt{m}$.
\end{proof}

This result in particular implies that given any finite set of real-valued functions $\cH$, we can guarantee kernel hypothesis OI when for all losses $\ell(\yhat, y)$ that are continuous and differentiable in $\yhat$. Given the previous construction in \Cref{prop:dc} showing that one can also guarantee kernel decision OI with respect to any finite class $\cH$, this establishes that one can in fact guarantee omniprediction with respect to any finite set $\cH$ and smooth losses $\cL$ at rates $\cO(\sqrt{T|\cH|})$.

\paragraph{Asymptotic KHOI for all continuous functions.}

RKHSs can contain very rich function classes which can be used as benchmark classes.
Indeed, some RKHSs are \textit{universal approximators} in the sense that they contain arbitrarily precise approximations of all continuous functions.
        
Formally, an RKHS $\cF$ is a universal approximator if, for all $\varepsilon$ and continuous $g: \cX \to \cR$, there exists some $f \in \cF$ such that $\sup_x |f(x) - g(x)| \leq \varepsilon$. Several common kernels like the Gaussian (or RBF) kernel, $k(x,x') = \exp(-\|x-x'\|^2)$ fall into this class.
We refer the reader to \cite{steinwart2008support}, Section 4.6 for further examples and background.

Universal approximators can be used to guarantee KHOI with respect to any continuous benchmark function $h$ and loss $\ell$. However, the result is best understood in an asymptotic sense since it is not always tractable to control relevant function norms in the RKHS.

Here, we outline a general approach for doing so. The template matches those of similar results in the literature (see e.g. the discussion in Section C of \cite{foster2006calibration}).
Let $\cH$ be a comparison class of continuous functions and $\cL$ be a class of continuous losses. Since the composition of continuous functions is continuous, the functions in $\partial \cL \circ \cH$ are also continuous.
        For a universal approximator $\cF$, denote by $\cF_\varepsilon \subseteq \cF$ a set such that for all $\partial \ell \circ h \in \partial \cL \circ \cH$, there exists some $f \in \cF_\varepsilon$ such that  $\| f - \partial \ell \circ h\|_\infty \leq \varepsilon$.
        Define
        $$B_\varepsilon = \inf_{\cF_\varepsilon \subseteq \cF} \sup_{f \in \cF_\varepsilon} \| f \|_{\cF}$$ be the infimum of a uniform upper bound on the norm of subsets $\cF_\varepsilon$ satisfying the property.
        Notice that $B_\varepsilon \geq B_{\varepsilon'}$ for all $\varepsilon \leq \varepsilon'$ since any $\cF_{\varepsilon'}$ satisfying the $\varepsilon'$-approximation property also satisfies $\varepsilon$-approximation.
        Then, one can chose a sequence $\varepsilon_T$ for $T=1,2,\dots$ such that $\lim_{T \to \infty} \varepsilon = 0$ and $B_{\varepsilon_T}= o(\sqrt{T})$.
        Then, the universal approximator can be used to satisfy an asymptotic, approximate version of KHOI with respect to $\cH$ and $\cL$.
        %
        
   \subsection{Generalizing kernel OI to separable losses.} \label{sec:product-loss}

    So far, we've established structural properties of losses $\ell(\yhat, y)$ that guarantee kernel decision and hypothesis OI. Here, we generalize these analyses to include losses that also depend on the features $x$.
    In particular, we prove that these requisite OI conditions also for a wide variety of \emph{separable} loss functions $\ell(x, \yhat, y)$: those where each loss function can be factorized into a function of the feature vector $x$ and of the decision-outcome pair $(\yhat, y)$.

        \begin{definition}[Separable Losses]
        A loss function $\ell(x, \yhat, y)$ is separable if there exists functions $\ell_x:\cX \to \cR$ and $\ell_y:[0,1]^2 \to \cR$ such that for all $(x, \yhat, y)$,
        \begin{align*}
            \ell(x, \yhat, y) = \ell_x(x) \ell_y(\yhat, y).
        \end{align*}
        Similarly, we say that a set of losses $\cL$   For a separable loss class $\cL$, we will define two new sets $\cL_x$ and $\cL_y$ to consist of the sets of the feature and decision-outcome components of the losses, respectively:
        \begin{align*}
            \cL = \{ \ell_x(x) \ell_y(\yhat, y): \ell_x \in \cL_x, \ell_y \in \cL_y \}.
        \end{align*}
      We refer to $\cL_x$ and $\cL_y$ as the \emph{factors} of the separable class $\cL$.
        \end{definition}

        Separable loss classes capture many important examples of loss functions that depend on features.
        For example, $\cL_x$ may consist of indicator functions for set membership, so that the loss only accumulates for members of a certain set.
        More generally, $\cL_x$ can be interpreted to consist of any (re)weighting of the loss function over feature vectors $x$.
        These kinds of losses will be important for our results on link prediction at the end of this section.

        We next state a simple result showing how to construct kernels for separable loss classes.
        Intuitively, the result says that any of the feature-invariant losses in the previous subsection can be reweighted by functions of the features $x$, as long as these functions are themselves in an RKHS with bounded norms.

        \begin{proposition}[Corollary to \cref{lem:productkernel}]\label{prop:product-loss}
        Let $\cL$ be a separable class of losses with factors $\cL_x, \cL_y$ and let $\cH$ be a comparator set of functions. Assume that $k_x$ has an RKHS $\cF_x$ such that $\cL_x \subseteq \cF_x$ and 
          \begin{align*}
                &\sqrt{\sup_{\ell_x \in \cL_x} \| \ell_x \|_{\cF_x}^2 \cdot \sup_{x \in \cX} k_x(x, x)}\leq B_x.
            \end{align*}
        \begin{enumerate}
            \item If $k_y$ is a kernel that is $(\cL_y, \cH)$-KHOI with parameter $B_y$. Then, then the product kernel, $$k((x,p),(x',p')) = k_x((x,p),(x',p')) \cdot k_y((x,p),(x',p')),$$ is $(\cL, \cH)$-KHOI with parameter $B_x B_y$.
            \item If $k_y$ is a kernel that is $(\cL_y)$-KDOI with parameter $B_y$. Then, then the same product kernel is $(\cL, \cH)$-KDOI with parameter $B_x B_y$.
        \end{enumerate}
        \end{proposition}
        \begin{proof}
        The result follows directly from \cref{lem:productkernel}, which says that the product of functions in an RKHS are contained in an RKHS and that the norm of the product function is no more than the product of norms of component the functions. 
        \end{proof}


        Letting the separable loss class be functions where $\cL_x$ is composed of a set membership kernel (as described in \cref{thm:finite-membership-rkhs} or any of the examples in \Cref{sec: OI}) and letting $\cL_y$ consist of loss functions $\ell_y(\yhat, y)$ which we know satisfy KDOI or KHOI from our previous analyses in \Cref{sub:dcexs,sub:hpexs}  illustrates the expressive power of separable loss classes.
        In particular, $\cF_x$ could consist of any collection of functions indexed by a set $\cI$ where for all $x \in \cX$ and $\ell_x \in \cL_x \subseteq \cF_x$, it holds $\sum_{i \in \cI} \ell_i(x)^2 \leq B$. 
        These could include, but are not limited to any finite set of group membership indicators.
        In this case, $k_{x}(x, x) \leq m$ and $\| \ell_x \|_{\cF_x} \leq 1$.
        $\cF_y$ could consist of any of the classic loss functions considered in \Cref{prop:dc} such as squared loss, log loss, or any bounded loss function.
        
        We leave exploration of non-separable loss functions where $\ell(x, \yhat, y)$ cannot be written as a product to future work.
    
    \subsection{Guarantees for online regression.}
    \label{subsec:regression}
    Before moving onto to discussing the application of these techniques in the link prediction context, we briefly remark on how these ideas apply to the specific problem of online regression. 
    
    Online squared loss regression oracles are algorithms which generate a transcript $\{(x_t, \Delta_t, y_t)\}_{t=1}^T$ satisfying the following guarantee: 
    \begin{align}
        \sum_{t=1}^T \E_{p_t\sim \Delta_t} (p_t - y_t)^2 \leq \min_{h \in \cH} \sum_{t=1}^T(h(x_t) - y_t)^2 + o(T).
    \end{align}
    In addition to being their intrinsic guarantees, online regression is a fundamental building block in the design of algorithms for other online learning problems like contextual bandits
    \cite{foster2020beyond} and online omniprediction \cite{garg2024oracle}.

    Here, we show that whenever there exists a kernel $k$ whose RKHS $\cF$ contains a comparator class of functions $\cH \subseteq \{\cX \rightarrow \R\}$, then the $\algname$ run with the kernel $k$ solves online regression. 

    \begin{proposition}
    Let $\cH$ be a set of comparator functions and let $k:\cX \times  \cX \rightarrow \R$ be an efficient kernel whose RKHS $\cF$ satisfies, $\cH \subseteq \cF$ and $\fnorm{h} \leq 1$ for all $h \in \cH$. Then, the $\algname$ algorithm instantiated with the kernel,
    \begin{align*}
        k((x,p),(x',p')) = k(x,x') + pp' + 1
    \end{align*}
    runs in polynomial time and generates a transcript $\{(x_t, \Delta_t, y_t)\}_{t=1}^T$ satisfying, 
    \begin{align}
        \sum_{t=1}^T \E_{p_t\sim \Delta_t} (p_t - y_t)^2 \leq \min_{h \in \cH} \sum_{t=1}^T(h(x_t) - y_t)^2 + 6\frac{\sqrt{1 + \sum_{t=1}^T \E_{p_t \sim \Delta_t} p_t (1-p_t) k((x_t, p_t), (x_t, p_t))}}{T}
    \end{align}
    \end{proposition}
    \begin{proof}
        The proof follows almost directly from \Cref{lem:onlinelossoi}. For the case of squared loss, 
          \begin{align*}
                \partial \ell(x, \yhat) &= \ell(x, \yhat, 1) - \ell(x, \yhat, 0) \\
                &= (\yhat - 1)^2 - (\yhat - 0)^2 \\
                &= 1 - 2 \yhat.
            \end{align*}
        Therefore, $\partial \ell(x, h(x)) = 1 - 2h(x)$ and  $\partial \ell(x, \pi_\ell(p)) = 1 - 2p$ (since $\pi_\ell(p)=p$ for the squared loss). 
        
        By assumption the RKHS for $k$ contains $h(x)$ and hence $2h(x)$ since RKHS are closed under scalar multiplication. Furthermore, the linear kernel $k_{\mathrm{lin}}(p,p') = 1 + pp'$ has an RKHS that contains all affine functions $a+ bp$. Moreoverv, both of these functions $1- 2h(x)$ and $1-2p$ have norm at most 3 in the corresponding RKHS. 
        
        By adding these two kernels together, we can guarantee online OI with respect to the union of both distinguishers by \Cref{thm:indistinguishability_main}.        
    \end{proof} 

    In short, by specializing our omniprediction analysis to the case where $\cL$ is a singleton set containing the squared loss, we show how to perform online regression with respect to any RKHS. Furthermore, the bounds have the advantage that they depend on the variance of the predictions $p_t$.\footnote{Bounds with this property are often referred to as second order bounds in the literature.} This result implies that the algorithms in \cite{garg2024oracle} are unconditionally computationally efficient whenever the class $\cH$ is contained in an RKHS.

    It has been previously observed that, since online gradient descent kernelizes, any time $\cH$ is in an RKHS, one can run online gradient descent (OGD) to produce an online squared error regression predictor \cite{foster2020beyond}. And, in fact, there are various other algorithms for online regression \cite{azoury2001relative,vovk2001competitive}, some of which achieve $\cO(\log(T))$ regret \cite{hazan2007logarithmic}.
    The point of this analysis is that the $\algname$ is yet another alternative.
    Each algorithm has different trade-offs in terms of computational complexity and regret that justify use of one or the other in different contexts.

\subsection{Specializing regret minimization to online link prediction.} \label{sub:omnilinkpred}

    As we outlined in the introduction to this paper, the link prediction problem has several distinctive properties that make it different from the traditional problems considered in prior work in online omniprediction \cite{garg2024oracle, gupta2022online}. In particular, the link prediction problem involves
        \begin{enumerate}[(a)]
            \item objectives that depend on characteristics of individuals or their communities; 
            \item diverse and time-varying objectives, such as high predictive performance and encouraging desirable outcomes; and
            \item comparator classes that are particularly suited to graph settings, either because they are expressive, such as graph neural networks, or they leverage some interpretable structure of graphs, such as R-convolution kernels.
        \end{enumerate}
        In the remainder of this section, we demonstrate how the results developed thus far can be instantiated so that the \algname{} solves online omniprediction in the link prediction context.

        \paragraph{Feature-dependent objectives.}
        
        Depending on the way social networks affect outcomes, different properties of networks may be socially desirable.
        For example, platform may want to facilitate integration \cite{abebe2022effect, calvo2004effects, zeltzer2020gender, stoica2018algorithmic, okafor2020social} or encourage homophily or heterophily along different dimensions \cite{mcpherson2001birds, kossinets2009origins, zeltzer2020gender}.
        It may be desirable to take into account structural cohesion measures \cite{eagle2010network,reagans2003network, ugander2012structural,granovetter1985economic} such as embeddedness.
        %
        %
        Our next result provides such a guarantee.

        \begin{proposition}
            Suppose the sequence of graphs $\cG_t$ is known to have nodes of degree bounded by a constant $m$ and $\cL$ consists of functions of the form $\ell(x, \yhat, y) = \nu(x) \gamma(\yhat, y)$, where 
            \begin{enumerate}[(a)]
                \item $\{ \gamma \; : \; \ell = \nu \cdot \gamma, \ell \in \cL\} \subseteq \cF$ for an RKHS $\cF$ associated with computationally efficient kernel $k$ where $\cF$ is KDOI with constant $B_1$, and
                \item $\nu$ may be any of the tests described in \cref{sub:oi_graphs} (dropping dependence on the prediction $p$), including 
            \begin{enumerate}[(i)]
                \item any set of measures $\cF' \subseteq \{ \cU^2 \to \R \}$ of (dis)similarity of individuals where $\sum_{f' \in \cF'} \nu(u) \leq m$, or
                \item any $c$-embeddedness test for $c \in \N$: $\nu(u, u') = \1 \{ \embed_t(u) = c\}$ (or, more generally, any isomorphism indicator function $\1\{ G_t \in \bar{G} \}$).
            \end{enumerate}
            \end{enumerate} 
            Additionally, suppose the exists an efficient kernel $k$ that is $(\cL, \cH)$-KHOI with parameter $B_{\mathrm{KHOI}}$.
            Then there exists a computationally efficient kernel $k'$ such that the $\algname$ instantiated with the kernel $k'$ is an $(\cL, \cH, (B_{\mathrm{KHOI}} + B_1(1 + \sqrt{m})) \sqrt{T + 1})$-online omnipredictor. 
        \end{proposition}

        \begin{proof}
            We will show that $\cL$ is KDOI with constant $B_1(1 + \sqrt{m})$. With, \cref{thm:omniprediction-via-vovk}, this will imply the result.
            Indeed, from \cref{prop:product-loss} that, since $\cF$ is KDOI with constant $B_1$, all we need to show is that 
            functions in (i) have function and feature norm $\sqrt{m}$ and functions in (ii) by 1. Then, we can combine the RKHS for (i) with the one from (ii) with \cref{lem:kernelsum}. The bound for (i) is proved in \cref{prop:all_pairs} and (ii) in \cref{prop:embeddedness}, \cref{prop:isomorphism} for embeddedness tests and isomorphism indicators, respectively.
        \end{proof}

        \paragraph{Diverse and time-varying objectives.}
        Platforms may need to make predictions for a class of loss functions if they are taking multiple actions on the basis of a single prediction, or the loss function is not known until decision time, perhaps because a platform is running experiments to learn which of a class of losses is best to optimize for long-term objectives.
        
        For a digital platform making link predictions, it may be important either to \textit{forecast} how link formation will affect relevant properties of networks, or to \textit{steer} the outcomes appropriately using recommendations.
        Many of the properties above can be encoded as loss functions in our setting, especially as separable losses \cref{sec:product-loss}. 

        \begin{proposition}
            Suppose $\cL$ consists of functions of the form $\ell(x, \yhat, y) = \nu(x) \gamma(\yhat, y)$, where 
            \begin{enumerate}[(a)]
                \item $\{ \nu \; : \; \ell = \nu \cdot \gamma, \ell \in \cL\} \subseteq \cF$ for an RKHS $\cF$ associated with computationally efficient kernel $k$ where $\cF$ is KDOI with constant $B_1$, and
                \item $\gamma$ may be
            \begin{enumerate}[(i)]
                \item any of the feature-invariant losses described in \cref{prop:dc},
                \item any polynomial function $f\; : \; \{0, 1\} \to [-1, 1]$ of outcomes $y$ of degree no more than $d$, or
                \item any finite convex combination of functions $\{ \gamma \; : \; \ell = \nu \cdot \gamma \}$ satisfying (a) or (b).
            \end{enumerate}
            \end{enumerate}
            Additionally, suppose there exists a kernel $k$ that is  $(\cL, \cH)$-KHOI with parameter $B_{\mathrm{KHOI}}$.
            Then there exists a kernel $k'$ such that the $\algname$ instantiated with the kernel is an $(\cL, \cH, (B_{\mathrm{KHOI}} + B_1 \sqrt{3 + 2^d}((4\sqrt{m}(3 + \gamma^{-1})) + 1)\sqrt{T + 1})$-online omnipredictor. 
        \end{proposition}

        \begin{proof}
            As in the proof of the previous proposition, we simply need to prove that $\cL$ is in an RKHS that is KDOI with constant $B_1 \sqrt{3 + 2^d}((4\sqrt{m}(3 + \gamma^{-1})) + 1)$, which implies the result.
            The bound on functions in (i) is $4\sqrt{m}(2 + \gamma^{-1})$ from \cref{prop:dc} and the bound on features is $\sqrt{3}$. 
            For (ii), since the dimension of $y$ is 1, the bound on functions is $1^d = 1$ for any polynomial of degree $d$ by \cref{prop:low-deg-booleans}. 
            The bound on the features is $2^{d/2}$, since $(1+ \inner{y}{y'})^d \leq 2^d$.
            We do not need to add any constant for the functions in (iii) because of the fact that convex combinations and the triangle inequality imply that the norm of any such function is no more than the norm of a function in parts (i) or (ii). 
            We can combine the  RKHSs associated with (i) and (ii) using \cref{lem:kernelsum}: the function norm associated with this combined RKHS is $4\sqrt{m}(2 + \gamma^{-1}) + 1$, and the feature norm is $\sqrt{3 + 2^{d} }$
            By the Moore-Aronszajn theorem (\cref{thm:moore-aronszajn}) the functions in (iii) are in the RKHS that contains those in (i) and (ii) by the fact that RKHSs are closed under linear combinations and the triangle inequality.
            
        \end{proof}
        
        Of course, in our setting, loss functions can only depend on features, decisions and outcomes, so networks can only hope to steer networks towards more desirable outcomes on a decision-by-decision basis. 
        Elsewhere, this local optimization has been described as a best response in a game-theoretic formulation of the problem \cite{noarov2023high}, or a greedy algorithm for steering the network towards desirable outcomes.
        We leave an exploration of non-greedy, global approaches to network optimization to future work.

        \paragraph{Graph-specific comparator classes.}
    Link prediction has a long history and a rich literature (see e.g., \cite{10.1145/3012704, kumar2020link}, which we can use to build comparator classes in our kernel omniprediction framework.
    Broadly, comparator classes fall into two categories: those containing flexible, expressive models, and those containing simple, interpretable ones.
    Expressive classes can be used to show that the \algname{}, instantiated with an appropriate kernel, can be used to compete with state-of-the-art and tailor-made models for a particular context, while the latter classes can be used to validate known dynamics, pass sanity checks, or guarantee trustworthiness with respect to the predictor.
    
    For expressive comparator classes, any finite set of pre-existing graph neural network link predictors \cite{NEURIPS2018_53f0d7c5, yun2019graph} or other powerful predictive models can be used to instantiate \cref{prop:finitesethyp}, which, informally, says that the $\algname$ can compete with any finite set of pre-existing functions.
    Prior work (e.g., \cite{garg2024oracle}) could not provide such guarantees because it required comparators to have binary rather than real-valued outputs.

    On the other hand, especially in socially sensitive contexts or high stakes decisions, interpretable models can be important (see, e.g., \cite{rudin2019stop, hays2023simplistic} for further discussion of interpretability in socially salient prediction).
    Interpretable function classes may include regression trees on pairs of node features or linear or polynomial regressions.
    They may also include the graph-specific models, like convolution kernels or other regression methods based on network topology as discussed in \cref{sub:topology}.
    %

        %

\subsection{Connections to Performative Prediction}
\label{subsec:performativity}
We close this section with some brief remarks interpreting these loss minimization guarantees within the context performative prediction.

Recall that in the online prediction protocol, $x_t\in \cX$ can chosen arbitrarily and in particular as a function of the history $\pi_{t-1} = \{(x_i, p_i, y_i)\}_{i=1}^{t-1}$. Outcomes $y_t$ can be chosen as a function both of the history $\pi_{t-1}$ and the current distribution over predictions $\Delta_t$. Hence in this setup, both the features $x_t$ and the outcomes $y_t$ can be performative. That is, they can be a \emph{function of the predictive model}. Furthermore, no restrictions are made regarding how Real Life responds to realized sequence of predictions. Please see \cite{perdomo2020performative,hardt2023performative,perdomo2023performative} for further background on the performative prediction literature. 

In particular, given an algorithm $\cA$, let $\{(x_t(\cA), \yhat_t(\cA), y_t(\cA))\}_{t=1}^T$ be the sequence of features, decisions and outcomes that are induced by making predictions $p_t \sim \Delta_t$ according to $\cA$ in the online protocol where $\yhat_t = \pi_\ell(x_t, p_t)$. Similarly,let $\{(x_t(h), \yhat_t(h), y_t(h))\}_{t=1}^T$ be the sequence of features, predictions and outcomes that are induced by making predictions according to some other function $h$. 
The algorithms we introduce in this section satisfy the following guarantee:
\begin{align*}
 \frac{1}{T} \sum_{t=1}^T \ell(x_t(\cA), \yhat_t(\cA), y_t(\cA)) \leq \min_{h \in \cH} \frac{1}{T} \sum_{t-1}^T \ell(x_t(\cA), h(x_t(\cA)), y_t(\cA)) +  o(1)
\end{align*}
This condition states that, in hindsight over the sequence of data induced by the algorithm $\cA$, no alternative $h$ in the comparator class would have higher loss. We think of this as a version of online performative stability (see \cite{perdomo2020performative} for a formal definition of performative stability). 

This is different than performative optimality.\footnote{Also note that both guarantees are the same if the data sequence $(x_t,y_t)$ is not influenced by the predictions.} The most natural definition for an algorithm $\cA$ to guarantee performative optimality would be the following statement where we change the dependency structure on the right hand side of the bound above:
\begin{align}
\label{eq:perf_omni}
 \frac{1}{T} \sum_{t=1}^T \ell(x_t(\cA), \yhat_t(\cA), y_t(\cA)) \leq \min_{h \in \cH} \frac{1}{T} \sum_{t-1}^T \ell(x_t(h), h(x_t(h)), y_t(h)) +  o(1).
\end{align}
While stability is about making good predictions in hindsight over the data that you induce, optimality is inherently a counterfactual statement. To achieve performative optimality, one compares performance not on the same data sequence, but on the data that \emph{would have resulted} by making decisions according to some other function $h$. Our algorithms guarantee the former, but not the latter.

In the batch setting, we by now know how to achieve performative optimality (see \textit{e.g.} \cite{miller2021outside}) and even performative omniprediction \cite{kim2023making}. We believe it is an interesting direction for future work to understand how one might guarantee online performative omniprediction.
That is, algorithms which achieve the guarantee in \Cref{eq:perf_omni} simultaneously over many losses. 

\section{New Algorithms for Online Quantile \& Vector Regression, Distance to Multicalibration, and Extensions to the Batch Case}
\label{sec: other kernel functions}

As an addeded benefit of our investigation into kernel methods for online indistinguishability and omniprediction, we obtain algorithms for other, seemingly different, online prediction problems.
In this section, we illustrate how to generalize the ideas presented previously beyond the binary setting to quantile regression and vector-valued predictions. As was true previously, the RKHS perspective provides a computationally efficient way to generate predictions that are indistinguishable with respect to rich classes of real-valued test functions in these settings. 

In addition to these new algorithms, we also initiate the study of distance to multicalibration and prove that the classical problem of weak agnostic learning of a function class $\cF$ can be solved efficiently whenever $\cF$ is a reproducing kernel Hilbert space.

\subsection{Quantile regression.}
\label{subsec:quantile_regression}

Unlike the binary case where means (i.e $\E[y \mid X=x]$) provide a complete description of the conditional distribution over outcomes, knowing the mean of a real-valued outcome $y$ often provides a misleading picture of the future. In domains like finance and weather prediction where outcomes are noisy and heavy-tailed, $y$ and $\E[y \mid x]$ can be very different. In these cases, we often often want estimates of best or worst case outcomes for $y_t$. Quantile prediction provides a rigorous way to estimate these best/worst case outcomes and quantify uncertainty. 

\paragraph{Prediction protocol.} The online protocol for quantile calibration mirrors that of binary prediction. At every round $t$, Real Life chooses features $x_t \in \cX$ arbitrarily, the learner chooses a distribution $\Delta_t$ over outcomes $p_t \in \R$. 
Finally, Nature selects a distribution $o_t$ over outcomes $y_t \in [Y_{\min}, Y_{\max}]$, possibly as a function of $\Delta_t$ and $x_t$. Throughout this section, we will assume that Real Life selects outcomes from a Lipschitz distribution. This is a technical condition, standard in online quantile prediction \cite{roth2022uncertain}, which requires that small changes in predictions also imply small changes in the CDF of $y$: 

\begin{definition}[Lipschitz Distribution]
A conditional label distribution $o$ over outcomes $y \in [Y_{\min}, Y_{\max}]$ is $\rho$-Lipschitz continuous for some parameter $\rho > 0$ if for all $p_1, p_2 \in [Y_{\min}, Y_{\max}]$,
\begin{align*}
\Pr_{y \sim o}[y \leq p_1] - \Pr_{y \sim o}[y \leq p_2]|  \leq \rho \cdot |p_1 -p_2|.
\end{align*}
\end{definition}

We aim to design online algorithms which satisfy the following guarantee:

\begin{definition}[Online Quantile Indistinguishability]

An algorithm $\cA$ guarantees online quantile indistinguishability with respect to class of functions $\cF \{\cX \times \R \rightarrow \R \}$ if it is guaranteed to generate a transcript $\{(x_t, \Delta_t, y_t)\}_{t=1}^T$ satisfying 
\begin{align*}
  \big| \sum_{t=1}^T   \E_{p_t \sim \Delta_t, y_t \sim o_t} (1\{y_t \leq p_t\} - q) f(x_t,p_t) \big| \leq \Reg(T,f)
\end{align*}
for all $f \in \cF$ where $\Reg(T,f)$ is $o(T)$ for every $f$. 
\end{definition}

As discussed in previous sections, we refer to the above guarantee as indistinguishability instead of as multicalibration since we generally assume that the functions $f$ are real-valued rather than binary valued. However, both terms are essentially interchangeable \cite{dwork2021outcome}.

\begin{figure}[t!]
\begin{boxedminipage}{\textwidth}
\begin{center}
\vspace{2pt}
{\centering{\underline{The Quantile Any Kernel Algorithm}}}
\end{center}
\noindent \textbf{Input:}  A kernel $k: (\cX \times [\Ymin, \Ymax])^2 \rightarrow \R$, quantile $q \in (0,1)$, bounds on outcome $[Y_{\min}, Y_{\max}]$\\ 

\noindent For $t=1, 2, \dots $:
\begin{enumerate}
    \item Given $\{(x_i, p_i, y_i)\}_{i=1}^{t-1}$ and current features $x_t$ define  
    \begin{align*}
            S^q_t(p) \defeq \sum_{i=1}^{t-1} k((x_t, p), (x_i, p_i))(1\{y_i \leq p_i\} -q) + \frac{1}{2} k((x_t, p), (x_t, p)) (1-2q).
    \end{align*}
    \item If $S^q_t(Y_{\min}),S^q_t(Y_{\max}) \geq 0$, \textbf{return} $\Delta_t = p_t = Y_{\min}$. 
    \item Else if, $S^q_t(Y_{\min}),S^q_t(Y_{\max}) \leq 0$, \textbf{return} $\Delta_t = p_t = Y_{\max}$. 
    \item Otherwise, let $B_t = \max_{t' \leq t} k((x_{t'},p_{t'}),(x_{t'},p_{t'}))$, 
    \begin{itemize}[$\bullet$]
        \item Run binary search to find $p_{t,1}$ and $p_{t,2}$ such that $S_t^{q}(p_{t,1})$ and $S_t^{q}(p_{t,2})$ have opposite signs and $|p_{t,1} - p_{t,2}| \leq 1 / (10 \cdot B_t  \cdot t^{3}).$ 
        \item \textbf{return} \[\Delta_t = \begin{cases}
            p_{t,1} &\text{with probability } \tau\\
            p_{t,2} &\text{with probability } 1 - \tau.
        \end{cases} \quad \text{ for } \tau = \frac{|S_t(p_{t,2})|}{|S_t(p_{t,1})|+|S_t(p_{t,2})|} \in [0, 1]\]
    \end{itemize}
\end{enumerate}
\vspace{2pt}
\end{boxedminipage}
\caption{Extension of $\algname$ for quantiles. The algorithm is essentially identical to the $\algname$, except that the $S_t$ function has been defined slightly differently. As before, the algorithm is near-deterministic. The distribution $\Delta_t$ is either a point mass, or supported on two points that are very close together.}
\figurelabel{fig:quantile_alg}
\end{figure}

\paragraph{Algorithm.} The algorithm to guarantee online quantile calibration is almost identical to (randomized) version of the K29* algorithm for binary calibration. The only difference is that function $S_t$ which the learner optimizes is slightly different.
\begin{align*}
S_t^{q}(p) \defeq \sum_{i=1}^{t-1} k((x_t, p), (x_i, p_i))(1\{y_i \leq p_i\} - q) + \frac{1}{2} k((x_t, p), (x_t, p)) (1-2q).
\end{align*}

\paragraph{Guarantees.} The proof for why this algorithm guarantees online quantile indistinguishability matches the template from previous analyses. The main idea is again to use the representer theorem to show that it suffices to bound the correlation between the quantile errors, $1\{y_t \leq p_t\} -q$, and the feature maps $\Phi(x_t, p_t)$:  
\begin{align}
\label{eq:init_decomposition_quantiles}
 \big|  \E_{p_t \sim \Delta_t, y_t \sim o_t}  \sum_{t=1}^T (1\{y_t \leq p_t\} - q) f(x_t,p_t) \big| &= \big| \E_{p_t \sim \Delta_t, y_t \sim o_t}  \langle \sum_{t=1}^T (1\{y_t \leq p_t\} - q) \Phi(x_t, p_t), c\rangle_{\cF}   \big| \\
 & \leq \fnorm{f}\cdot \E  \bigg \|  \sum_{t=1}^T (1\{y_t \leq p_t\} - q) \Phi(x_t, p_t) \bigg\|_{\cF}
\end{align}
From this decomposition, we can leverage the defensive forecasting approach \cite{vovk2005defensive,shafer2005probability, vovk2007non} to find a prediction strategy which guarantees that the last term,
\begin{align*}
    \E \bigg \|  \sum_{t=1}^T (1\{y_t \leq p_t\} - q) \Phi(x_t, p_t) \bigg\|_{\cF},
\end{align*}
grows sublinearly, \textit{i.e.} is bounded by $\cO(\sqrt{T})$. As we now formalize in the following lemma, this is ensured by carefully choosing the $S_t^q(\cdot)$ function in the \qalgname{}  and incorporating the forecasting hedging ideas from \cite{foster2021forecast}. We break the analysis up into a series of lemmas:

\begin{lemma}
Assume that the learner makes predictions in such a way that, for all choices of Nature, 
\begin{align*}
    \E[S_t^{q}(p_t)(1\{y_t \leq p_t\} - q)] \leq \eps_t 
\end{align*}
for all $t\geq 1$. Then, 
\begin{align*}
\E \bigg\|\sum_{t=1}^T (1\{y_t \leq p_t\} - q) \cdot \Phi(x_t, p_t) \bigg\|^2_{\cF}  \leq 2\sum_{t=1}^T \eps_t +  \E \sum_{t=1}^T q(1-q) \bigg \| \Phi(x_t, p_t) \bigg \|_{\cF}^2.
\end{align*}
\end{lemma}

\begin{proof}
By definition of $S_t^q$ we have that $\sum_{t=1}^T\E[S_t^{q}(p_t)(1\{y_t \leq p_t\} - q)]$ is equal to:
\begin{align*}
\sum_{t=1}^T \sum_{i = 1}^{t-1} k((x_t, p_t), (x_i, p_i))(1\{y_t \leq p_t\} - q) (1\{y_i \leq p_i\} - q)  + \frac{1}{2} \sum_{t=1}^T k((x_t, p_t), (x_t, p_t)) (1 - 2q)(1\{y_t \leq p_t\} - q) .
\end{align*}
Increasing the top limit of the first sum from $t-1$ to $T$, we can rewrite this as:
\begin{align*}
& \frac{1}{2} \sum_{t=1}^T \sum_{i=1}^T k((x_t, p_t), (x_i, p_i))(1\{y_t \leq p_t\} - q) (1\{y_i \leq p_i\} - q) - \frac{1}{2} \sum_{t=1}^T k((x_t, p_t), (x_t, p_t)) (1\{y_t \leq p_t\} - q)^2   \\
& + \frac{1}{2} \sum_{t=1}^T k((x_t, p_t), (x_t, p_t)) (1 - 2q)(1\{y_t \leq p_t\} - p_t) 
\end{align*}
Now, using the identity that for binary $v$, $(v-q)^2 = q(1-q) + (1-2q)(v-q)$, we get:
\begin{align*}
 \frac{1}{2} \sum_{t=1}^T \sum_{i=1}^T k((x_t, p_t), (x_i, p_i))(1\{y_t \leq p_t\} - q) (1\{y_i \leq p_i\} - q) -  \frac{1}{2} \sum_{t=1}^T k((x_t, p_t), (x_t, p_t)) q(1-q).
\end{align*}
Finally, since $k((x_t, p_t), (x_i, p_i))(1\{y_t \leq p_t\} - q) (1\{y_i \leq p_i\} - q)$  is equal to $$\langle \Phi(x_i, p_i) (1\{y_i \leq p_i\} - q),   \Phi(x_t, p_t) (1\{y_t \leq p_t\} - q)\rangle_{\cF},$$
we arrive at the identity that:
\begin{align*}
\sum_{t=1}^T\E[S_t^{q}(p_t)(1\{y_t \leq p_t\} - q)] = \frac{1}{2} \bigg\|\sum_{t=1}^T (1\{y_t \leq p_t\} - q) \cdot \Phi(x_t, p_t) \bigg\|^2_{\cF} - \frac{1}{2} \sum_{t=1}^T q(1-q) \bigg \| \Phi(x_t, p_t) \bigg \|_{\cF}^2.
\end{align*}
Lastly, by our assumption that $ \E[S_t^{q}(1\{y_t \leq p_t\} - q)] \leq \eps_t$, we get our desired result:
\begin{align*}
\E \bigg\|\sum_{t=1}^T (1\{y_t \leq p_t\} - q) \cdot \Phi(x_t, p_t) \bigg\|^2_{\cF}  \leq 2\sum_{t=1}^T \eps_t +  \E \sum_{t=1}^T q(1-q) \bigg \| \Phi(x_t, p_t) \bigg \|_{\cF}^2.
\end{align*}
\end{proof}

Given this result, the final step in the analysis is to show that the Quantile Any Kernel Algorithm generates predictions such that $ \E[S_t^{q}(p_t)(1\{y_t \leq p_t\} - q)] \approx 0$.

\begin{lemma}
\label{lemma:negativity}
Assume that the learner makes predictions $p_t \sim \Delta_t$ according to the Quantile Any Kernel algorithm and that Real Life selects outcomes $y_t$ from a $\rho$-Lipschitz conditional distribution $o_t$, then
\begin{align*}
    | \E_{y_t \sim \Delta_t, p_t \sim o_t} S_t^{q}(p_t)(1\{y_t \leq p_t\} - q) | \leq \frac{1}{10t^2} \rho.
\end{align*}
\end{lemma}

\begin{proof}
If $S_t^q(Y_{\min})$ and $S_t^{q}(Y_{\max})$ are both non-negative or non-positive then the inequality,
\begin{align*}
    S_t^{q}(p_t)(1\{y_t \leq p_t\} - q) \leq 0,
\end{align*}
holds trivially regardless of the outcome $y_t.$ If they have opposite signs, recall that by definition of the algorithm, the learner plays $p_{t,1}$ with probability $r_1 = \tau$ and $p_{t,2}$ with probability $r_2 = 1 - r_1$. With his in mind,
\begin{align*}
\E_{y_t, p_t} \left[ S_t^{q}(p)(1\{y_t \leq p_t\} - q) \right]  = r_1 \cdot S_t^{q}(p_{t,1})\E[1\{y_t \leq p_{t,1}\} - q] + r_2 \cdot S_t^{q}(p_{t,2})\E[1\{y_t \leq p_{t,2}\} - q].
\end{align*}
By adding and subtracting, $r_1 \cdot S_t^{q}(p_{t,1}) \E[1\{y_t \leq p_{t,2}\} -q]$, we can rewrite this as, 
\begin{align*}
    [r_1 S_t^{q}(p_{t,1}) + r_2 S_t^{q}(p_{t,2})] \cdot \E[1\{y_t \leq p_{t,2}\}- q] + r_1 S_t^{q}(p_{t,1}) \E[1\{y_t \leq p_{t,1}\} - 1\{y_t \leq p_{t,2} \}].
\end{align*}
By choice of $r_1,p_{t,1}$ and $p_{t,2}$, we have that, $r_1 S_t^{q}(p_{t,1}) + r_2 S_t^{q}(p_{t,2}) = 0$, so the first term drops out. Then, since Real Life is required to select outcomes from a Lipschitz distribution, 
\begin{align*}
   r_1 S_t^{q}(p_{t,1}) \E[1\{y_t \leq p_{t,1}\} - 1\{y_t \leq p_{t,2} \}] 
 & \leq   |S_t^{q}(p_{t,1})| \cdot |\Pr[y_t \leq p_{t,1}] - \Pr[y_t \leq p_{t,2}]|   \\
   & \leq |S_t^{q}(p_{t,1})| \cdot \rho  \cdot |p_{t,1} - p_{t,2}|.
\end{align*}
The bound follows from the fact that $|S_t^{q}(p_{t,1})| \leq B_t$ and $|p_{t,1} - p_{t,2}| \leq 1 / (10 \cdot B_t \cdot t^3)$.
\end{proof}
Taken together, these lemmas establish the following theorem which summarizes the final guarantee of the \qalgname{}. 
\begin{theorem}
Let $k$ be a kernel with associated reproducing kernel Hilbert space $\cF$. If outcomes $y_t$ are drawn from a $\rho$-Lipschitz conditional distribution, then, the \qalgname{} generates a transcript $\{(x_t, \Delta_t, y_t)\}$ such that for all $f \in \cF$,
\begin{align*}
  \big| \sum_{t=1}^T   \E (1\{y_t \leq p_t\} - q) f(x_t,p_t) \big| \leq \fnorm{f} \sqrt{\rho +  \sum_{t=1}^T q(1-q) \E k((x_t,p_t), (x_t, p_t))} 
\end{align*}
Furthermore, if the kernel is bounded by $B$, $$\sup_{(x,p) \in \cX \times [0,1]} k((x,p),(x,p)) \leq B,$$ then the per round runtime of the algorithm is bounded by ${\cO}(t  \cdot  \log(t B) \cdot \mathsf{time}(k))$, where $\mathsf{time}(k)$ is a uniform upper bound on the runtime of computing the kernel function $k$.
\end{theorem}

\paragraph{Discussion.} To the best of our knowledge this is the first online algorithm for quantile regression with respect to functions spaces $\cF$ that are an RKHS. As was the case with the \algname{}, the algorithm is very simple to implement. At every time step, one only needs to solve a binary search problem over the unit interval. Furthermore, the guarantees are adaptive and illustrates how certain quantiles $q$ (those closer to 0 or 1) lead to lower OI error bounds than those closer to 1/2. Lastly, the algorithm is hyperparameter free, one does not need to know the Lipschitz constant $\rho$ ahead of time. The only requirement is that we know bounds $\Ymin, \Ymax$ on the outcome $y$.

\subsection{Vector-valued, high-dimensional regression.}

In addition to quantile regression, the RKHS and defensive forecasting  viewpoint also provides a simple way of generating indistinguishable predictions in settings where outcomes are high-dimensional. That is, instead of binary or scalar-valued outcomes, in this subsection we consider the case where $y_t \in \cY \subset \R^d$ and $\cY$ is a compact, convex set (e.g $\cY = [-1,1]^d$). 

\paragraph{Formal setup.} The online protocol is identical to that of scalar prediction. At every round $t$, Real Life chooses features $x_t \in \cX$ arbitrarily, the learner chooses a distribution $\Delta_t$ over  $p_t \in  \cY$. Finally, Nature selects a distribution $o_t$ over outcomes $y_t \in \cY$, possibly as a function of $\Delta_t$ and $x_t$.

\begin{definition}[Online Vector-Valued Indistinguishability] An algorithm $\cA$ guarantees online high-dimensional indistinguishability with respect to class of functions $\cF \subseteq \{ \cX \times \cY \rightarrow \R^d\}$ if it is guaranteed to generate a transcript satisfying the following guarantee,
\begin{align*}
  \big| \sum_{t=1}^T   \E_{p_t \sim \Delta_t, y_t \sim o_t} (y_t - p_t)^\top f(x_t, p_t) \big| \leq \Reg(T,f)
\end{align*}
where $\Reg: \N \times \cF \rightarrow \R$ is $o(T)$ for every $f$.
\end{definition}

Note that in this setting the test functions $c(x_t, p_t)$ are \emph{vector-valued}. High-dimensional indistinguishability asks that, when averaged over the sequence, prediction errors $y_t - p_t$ are uncorrelated with any test function $f \in \cF$,
\begin{align*}
\lim_{t\rightarrow \infty} \frac{1}{T} \sum_{t=1}^T (y_t - p_t)^\top f(x_t, p_t)  = 0.
\end{align*}
\paragraph{Background on vector-valued RKHSs} As was the case previously, the algorithm has guarantees with respect to set functions that form an RKHS, but in this functions take values in $\R^d$ rather than $\R$.  A vector-valued RKHS is a set of functions $\cF \subset \{\cX \rightarrow \R^d\}$, where the set $\cF$ is itself a Hilbert space, equipped with an inner product $\langle \cdot, \cdot \rangle_{\cF}$.

A kernel $K$ for a vector-valued RKHS is a mapping from $\cX \times \cX$ to $\R^{d\times d}$.
To disambiguate from the scalar case, we use capital $K$ to denote matrix-valued kernels and lower case $k$ to denote a scalar-valued kernel. 

For a more comprehensive background on vector-valued kernels, we refer the reader to the excellent survey by Alvarez, Rosasco, and Lawrence \cite{vectorRKHSreview}. For the context of our results, we will only need two main facts. First, as with the scalar case, the kernel $K$ has the reproducing property such that for any function $f: \cX \rightarrow \R^d$ in the RKHS and vector $v \in \R^d$.
\begin{align}
\label{eq:vector_kernel_reproducing}
    f(z)^\top v = \langle f, \Phi(z) v \rangle_{\cF}
\end{align}
Here $\Phi(x)$ is the feature map of $x$. For any fixed $x$, $\Phi(x)$ is a mapping from $\R^d$ to $\cF$. The last property we need is part a) from Proposition 2.1 in \cite{micchelli2005learning} which states that for any $x,x'\in \cX$ and $v, v' \in \R^d$:
\begin{align}
\label{eq:pontil_lemma}
    v^\top K(z,z') v' = \langle \Phi(z') v', \Phi(z)v \rangle_{\cF}
\end{align}

\begin{figure}[t!]
\begin{boxedminipage}{\textwidth}
\begin{center}
\vspace{2pt}
{\centering{\underline{The Vector Any Kernel Algorithm}}}
\end{center}
\noindent \textbf{Input:} A compact, convex set $\cY \subseteq \R^d$, a kernel $K: (\cX \times \cY)^2 \rightarrow \R^{d \times d}$\\ 

\noindent For $t=1, \dots,$:
\begin{enumerate}
    \item Given history $\{(x_i, p_i, y_i)\}_{i=1}^{t-1}$ and current features $x_t$ define  
    \begin{align*}
            S_t(p) \defeq \sum_{i=1}^{t-1} K((x_t, p), (x_i, p_i))(y_i -p_i)  \in \R^d
    \end{align*}
    \item If $K$ is continuous in $p$, \textbf{return} $ \Delta_t = p_t \in \cY$ that solves the variational-inequality:
    \begin{align}
    \label{eq:deterministic_variational_inequality}
        \sup_{y \in \cY}(y - p_t)^\top S_t(p) \leq 0
    \end{align}
    For discontinuous kernels, \textbf{return} $\Delta_t$, where $\Delta_t$ is a distribution over $p_t \in \cY$ satisfying
    \begin{align}
    \label{eq:randomized_variational_inequality}
        \E_{p_t\sim \Delta_t} \sup_{y \in \cY}(y - p_t)^\top S_t(p) \leq 0 
    \end{align}
    
\end{enumerate}
\vspace{2pt}
\end{boxedminipage}
\caption{Extension of the \algname{} for high-dimensional prediction. For simplicity, we state the algorithm assuming that the variational inequalities are solved exactly. However, as illustrated previously in quantile and binary prediction, the analysis can be easily modified to accomodate approximate solutions. The behavior of the algorithm for continuous kernels is the same as in \cite{vovk2005defensive}. The extension to the discontinuous case is new.}
\figurelabel{fig:deterministic_k29}
\end{figure}

\paragraph{Algorithmic guarantees.} As before, the advantage of this approach is that the final algorithm has strong guarantees of performance, and is additionally very simple to state and analyze. The main computational difference relative to previous settings is that the learner needs to solve a \emph{variational inequality} (\cref{eq:deterministic_variational_inequality,eq:randomized_variational_inequality}). Variational inequalities are a rich and well-developed area of research within the optimization literature \cite{kinderlehrer2000introduction,noor1988general}, with earliest work dating back to the papers by Signori and Fichera \cite{fichera1963sul}. 
These optimization problems always have a solution. Furthermore, these solutions can be found efficiently in various settings. 

However, before discussing these ideas further, we state the final end-to-end result for the \valgname{}:

\begin{theorem}
Let $K$ be a kernel for a vector-valued reproducing kernel Hilbert space $\cF$. Then, the $\valgname$ is guaranteed to generate a transcript such that for any $f \in \cF$,
\begin{align*}
    \bigg| \sum_{t=1}^T \E_{p_t \sim \Delta_t}(y_t - p_t)^\top f(x_t, p_t) \bigg| \leq  \| f\|_{\cF} \sqrt{\sum_{t=1}^T \E_{p_t \sim \Delta_t} (y_t - p_t)^\top K((x_t, p_t), (x_t,p_t)) (y_t - p_t)}.
\end{align*}
If we further assume that the kernel K is uniformly bounded by $B$ over $\cX \times \cY$, and that the diameter of the set $\cY$ is at most $D$, 
\begin{align*}
    \sup_{x\in \cX, p\in \cY} \| K((x,p),(x,p))\|_{\mathrm{op}} \leq B, \quad \sup_{p,p' \in \cY} \|p- p'\|^2_2 \leq D
\end{align*}
then, the above guarantee implies that:
\begin{align*}
    \bigg| \sum_{t=1}^T \E (y_t - p_t)^\top c(x_t, p_t) \bigg| \leq  \| c\|_{\cF} \sqrt{BDT}.
\end{align*}
Furthermore, the per round runtime of the algorithm is at most $\cO(t \mathsf{timeVE)})$ where $\mathsf{timeVE)}$ is an upper bound on the time it takes solve the variational inequality problems in \Cref{eq:deterministic_variational_inequality} and \Cref{eq:randomized_variational_inequality}.
\end{theorem}
\begin{proof}
We start the analysis by again showing that it suffices to bound the correlation between the features $\Phi(x_t,p_t)$ and the errors $(y_t - p_t)$. Using the reproducing property for vector-valued RKHSs, \cref{eq:vector_kernel_reproducing}, we first show the following bound:
\begin{align}  \bigg| \E \sum_{t=1}^T (y_t - p_t)^\top c(x_t, p_t) \bigg| &= \bigg| \E\sum_{t=1}^T \langle c, \Phi(x_t, p_t) (y_t - p_t) \rangle_{\cF} \bigg| \notag\\
    &=\bigg| \langle c,  \sum_{t=1}^T  \E  \Phi(x_t, p_t) (y_t - p_t) \rangle_{\cF} \bigg| \notag \\
    & \leq \| c \|_{\cF} \cdot \big\|  \sum_{t=1}^T  \E \left[ \Phi(x_t, p_t) (y_t - p_t) \right] \big\|_{\cF}. \label{eq:vector_decomp}
\end{align}
Next, we show that the \valgname{}  bounds the second term. In particular, by construction, the algorithm guarantees that:
\begin{align*}
        \E_{p_t\sim \Delta_t, \Delta_t \sim o_t} (y_t - p_t)^\top S_t(p_t) \leq 0 \text{ where } S_t(p) = \sum_{i=1}^{t-1} K((x_t, p), (x_i, p_i))(y_i -p_i).
\end{align*}
Summing up this quantity over all $T$ rounds, 
\begin{align*}
    0 &\geq \sum_{t=1}^T \sum_{i=1}^{t-1} \E\left[ (y_t - p_t)^\top K((x_t, p_t), (x_i, p_i))(y_i - p_i) \right]\\ 
    &= \frac{1}{2} \sum_{t=1}^T \sum_{i=1}^T \E \left[(y_t - p_t)^\top K((x_t, p_t), (x_i, p_i)) (y_i-p_i) \right]\\ 
    &- \frac{1}{2} \sum_{t=1}^T \E \left[(y_t - p_t)^\top K((x_t, p_t), (x_t, p_t)) (y_t-p_t)\right]. 
\end{align*}
Hence, 
\begin{align}
 \sum_{t=1}^T \sum_{i=1}^T \E \left[(y_t - p_t)^\top K((x_t, p_t), (x_i, p_i)) (y_i-p_i) \right] \leq    \sum_{t=1}^T \E \left[ (y_t - p_t)^\top K((x_t, p_t), (x_t, p_t)) (y_t-p_t) \right]. \label{eq:vk29_first}
\end{align}
Now, by applying \cref{eq:pontil_lemma}, we see that,
\begin{align}
    (y_t - p_t)^\top K((x_t, p_t), (x_t,p_t)) (y_t - p_t) &= \langle \Phi(x_t, p_t) (y_t - p_t), \Phi(x_t, p_t) (y_t - p_t) \rangle_{\cF}  \notag\\
    & = \big \| \Phi(x_t, p_t) (y_t- p_t)\big \|_{\cF}^2 
    \label{eq:vk29_second}
\end{align}
And, 
\begin{align}
    \sum_{t=1}^T \sum_{i=1}^T (y_t - p_t)^\top K((x_t,p_t), (x_i,p_i)) (y_i-p_i) = \big \| \sum_{t=1}^T \Phi(x_t, p_t)  (y_t - p_t) \big \|_{\cF}^2 \label{eq:vk29_third} 
\end{align}
Combining \cref{eq:vk29_first,eq:vk29_second,eq:vk29_third} (and Jensen's inequality) we get that the Vector Any Kernel algorithm generates sequence satisfying,
\begin{align*}
\big\|\E \left[ \Phi(x_t, p_t) (y_t - p_t) \right] \big\|_{\cF}^2 \leq    \E \left[ \Big \| \sum_{t=1}^T \Phi(x_t, p_t)  (y_t - p_t) \Big \|_{\cF}^2 \right]\leq \sum_{t=1}^T \E \left[ \big \|  \Phi(x_t, p_t) (y_t- p_t) \big \|_{\cF}^2 \right]
\end{align*}
Together with the first inequality, \cref{eq:vector_decomp}, we get our desired data-dependent guarantee,
\begin{align*}
    \bigg| \sum_{t=1}^T \E(y_t - p_t)^\top c(x_t, p_t) \bigg| \leq  \| c\|_{\cF} \sqrt{\sum_{t=1}^T \E \left[ \big \|  \Phi(x_t, p_t) (y_t- p_t) \big \|_{\cF}^2 \right] }.
\end{align*}
\end{proof}

\paragraph{Variational inequalities.} As seen from the description of the algorithm, the main computational step is the Vector Any Kernel algorithm is to solve for a vector $p_t$, or a distribution $\Delta_t$ over vectors $p_t$ that satisfies, 
\begin{align*}
    (y - p_t)^\top S_t(p) \leq 0 \quad \forall y \in \cY.
\end{align*}
From a first glance, it is not obvious that such a $p_t$ exists. However, in a recent, related paper on online calibration, Foster and Hart show that these ``outgoing fixed points'' exists under very mild conditions. We restate their result below:

\begin{proposition}[Theorem 4 \& Corollary 6 in \cite{foster2021forecast}]
Let $\cY \subset \R^d$ be a compact, convex set and let $S: \cY \rightarrow \R^d$ be a continuous function. Then, there exists a point $p_* \in \cY$ such that, 
\begin{align*}
   (y - p_*)^\top S(p_*) \leq 0 \quad \forall y \in \cY.
\end{align*}
If $S: \cY \rightarrow \R^d$ is not necessarily continuous, but bounded in the sense that, $$\sup_{y \in \cYh} \|S(y) \|_2 < \infty,$$ then, for all $\eps > 0$ there exists a distribution $\Delta$ supported on at most $d+3$ points in $\cYh$ such that,
\begin{align*}
    \E_{p_* \sim \Delta} (y-p_*)^\top S(p_*) \leq \eps \quad \forall y \in \cYh.
\end{align*}
\end{proposition}

Not only do these fixed points exists, but by now there is an increasingly extensive literature on algorithms for finding them \cite{censor2012extensions, burachik1998generalized,kinderlehrer2000introduction,noor1988general} under various regularity conditions on the function $S$.

\paragraph{Discussion.} The \valgname{} is most closely related to the K29 (not star) algorithm from Vovk \cite{vovk2005defensive,vovk2007non}. By using the forecast hedging idea from \cite{foster2021forecast}, we extend the algorithm so that it works for any matrix valued kernel. Modulo this extension, the regret guarantees are nearly identical. 

To the best of our knowledge, the other most closely related work is the recent paper by Noarov, Ramalingam, Roth, and Xie \cite{noarov2023high}. Using different techniques to ours (from online minimax optimization), they introduce an algorithm that achieves the following guarantee,
\begin{align*}
    \|\sum_{t=1}^T (y_t - p_t)^\top f(x_t,p_t) \|_{\infty} \leq \cO(\sqrt{T}).
\end{align*}
This is essentially the same goal we consider (up to poly d factors). However, their result holds with respect to functions $f$ taking values in $\{0,1\}^d$ (they refer to $f$ as events) and sets $\cF$ which are finite. In our case, $|\cF|$ is infinite and $\cF$ is real-valued since it is an RKHS. 

Furthermore, their runtime is guarateed to be polynomial whenever $|\cF|$ is polynomially sized whereas our results are best understood as being oracle efficient. The algorithm runs in polynomial time whenever there exists an efficient oracle that can solve the corresponding variational inequality. These efficent algorithms exist for instance when the functions $S$ are \emph{monotone}, however they may be computationally difficult in general.  

Please see the supplemantary material for results on how one can design matrix valued kernels whose corresponding RKHS contain an arbitrary finite set of functions $\cF \subseteq \{\cX \times \cY \rightarrow \R^d\}$.

\subsection{Distance to online multicalibration.}
\label{sec:distance-to-multicalibration}

In this subsection, we show that instantiating the \algname{} with a particular kernel $k$ achieves small \emph{distance to online multicalibration}, a novel extension of the canonical notion of \emph{distance to (online) calibration} from \cite{blasiok2023unifying, qiao2024distance} which we introduce in this paper.

To start, we start by recalling what it means for a predictor to be perfectly calibrated and restate the definition of distance to calibration from \cite{blasiok2023unifying, qiao2024distance}.  

\begin{definition}[Perfect Online (Multi)Calibration]
\label{def:perfect-online-multicalibration}
    Suppose we are given fixed sequences of predictions $\bm{p} = (p_1, \ldots, p_T) \in [0, 1]^T$, features $\bm{x} = (x_1, \ldots, x_T) \in \cX^T$, outcomes $\bm{y} = (y_1, \ldots, y_T) \in \{0, 1\}^T$, and a collection $\cC \subseteq \{0, 1\}^\cX$ of group indicator functions. 
    We say that $\bm{p}$ is \emph{perfectly multicalibrated with respect to the collection $\cC$} if for all $v \in [0, 1]$ and $c \in \cC$, 
    \begin{align*}
            \sum_{t=1}^T (y_t - v)c(x_t)\bm{1}[p_t = v] = 0.
    \end{align*}
    Likewise, we say that a prediction is perfectly calibrated if it is multicalibrated with respect to the collection $\cC$ that just contains the constant 1 function.
\end{definition}

Given a function $c : \cX \to \{0, 1\}$, let $\mathrm{PC}(c)$ denote the set of prediction sequences $\bm{q} = (q_1, \ldots, q_T) \in [0, 1]^T$ that are perfectly calibrated on $c$. Let $\mathrm{PC}(\cC)$ be the intersection of $\mathrm{PC}(c)$ for all $c \in \cC$.

While defining perfect calibration is relatively straightforward, defining distance to calibration is not. In their recent work, \cite{blasiok2023unifying} propose a unifying notion of distance to calibration. Here, we state the online version of their definition as presented in \cite{qiao2024distance}.
 
 \begin{definition}[Distance to Online Calibration \cite{qiao2024distance}]
    Suppose we are given fixed sequences of predictions $\bm{p} = (p_1, \ldots, p_T) \in [0, 1]^T$, features $\bm{x} = (x_1, \ldots, x_T) \in \cX^T$, outcomes $\bm{y} = (y_1, \ldots, y_T) \in \{0, 1\}^T$. The \emph{distance to online calibration} is
    \[
        \mathrm{dCE}_{\bm{y}}(\bm{p}) = \inf_{\bm{q} \in \mathrm{PC}(\1)} \sum_{t=1}^T |p_t - q_t|,
    \]
    where $\1 : \cX \to \{0, 1\}$ denotes the all-ones function.
\end{definition}

With these definitions in hand, we now introduce our definition of distance to (online) \emph{multicalibration}. Given a collection $\cC$ of group indicator functions there are several ways of defining distance to multicalibration. Here, we present two such versions, showing how one is efficiently achievable and the other is in fact impossible to achieve in general.

\begin{definition}[Distance to Online Multicalibration, Standard and Strong Variants]\;
\label{def:dist-to-mc}
Suppose we are given fixed sequences of predictions $\bm{p} = (p_1, \ldots, p_T) \in [0, 1]^T$, features $\bm{x} = (x_1, \ldots, x_T) \in \cX^T$, outcomes $\bm{y} = (y_1, \ldots, y_T) \in \{0, 1\}^T$, and a collection $\cC \subseteq \{0, 1\}^\cX$ of group indicator functions.

We define the \emph{distance to online multicalibration} $\mathrm{dMCE}_{\bm{y},\cC}$ and \emph{strong distance to online multicalibration} $\mathrm{dMCE}^{\mathrm{strong}}_{\bm{y},\cC}$ as follows:
\begin{align*}
    \mathrm{dMCE}_{\bm{y},\cC}(\bm{p}) &= \sup_{c \in \cC} \inf_{\bm{q} \in \mathrm{PC}(c)} \sum_{t=1}^T |p_t - q_t|, \\
    \mathrm{dMCE}^{\mathrm{strong}}_{\bm{y},\cC}(\bm{p}) &= \inf_{\bm{q} \in \mathrm{PC}(\cC)} \sum_{t=1}^T |p_t - q_t|,\\
\end{align*}
where $\mathrm{PC}(\cC)$ is as defined in \cref{def:perfect-online-multicalibration}. 
\end{definition}

Several remarks about \cref{def:dist-to-mc} are in order. First, it is easy to see that even the first of these two notions of distance to multicalibration is still stronger than a global notion of distance to calibration. For example, in the online setting, consider a single subsequence indicator $c : \cX \to \{0, 1\}$ such that for each $t = 1, \ldots, T$,
\[c(x_t) = \begin{cases}
    1 & \text{if $t$ is odd} \\
    0 & \text{if $t$ is even.}
\end{cases}\]
Suppose the outcome sequence $y$ follow the same pattern, so $y_t = c(x_t)$, but we predict $p_t = 1/2$ for all time steps $t \in [T]$. In this case, $\bm{p}$ will be perfectly calibrated with respect to $\bm{y}$ in a global sense, but $\mathrm{dMCE}_{\bm{y},\{c\}}(\bm{p}) = T/4 = \Omega(T)$.

Next, observe that in the definition of distance to online multicalibration, the constraint $q \in \mathrm{PC}(c)$ only restricts the values that $q$ takes during time steps $t \in [T]$ such that $c(x_t) = 1$. In other words, during time steps for which $c(x_t) = 0$, it is clearly optimal to take $q_t = p_t$ if the goal is to minimize the sum on the right side, because this ensures that the $t$\textsuperscript{th} term satisfies $|p_t - q_t| = 0$. Consequently, we have the equality \[\mathrm{dMCE}_{\bm{y},\cC}(\bm{p}) = \sup_{c \in \cC} \inf_{\bm{q} \in \mathrm{PC}(c)} \sum_{t=1}^T |p_t - q_t|c(x_t)\]

Next, we establish the relationship between our standard and strong notions of distance to online multicalibration:
\begin{theorem}
    For any prediction, feature, and outcome sequences, and for any collection $\cC$, \[\mathrm{dMCE}_{\bm{y},\cC}(\bm{p}) \le \mathrm{dMCE}^{\mathrm{strong}}_{\bm{y},\cC}(\bm{p}).\] Moreover, this inequality can be strict; in fact, there exists a distribution over feature and outcome sequences, as well as a collection $\cC$, such that for any prediction algorithm used to generate $\bm{p}$,
    \[\mathrm{dMCE}_{\bm{y},\cC}(\bm{p}) \le O(1)\] 
    but with high probability,
    \[
        \mathrm{dMCE}^{\mathrm{strong}}_{\bm{y},\cC}(\bm{p}) \ge \Omega(T).
    \]
\end{theorem}
\begin{proof}
Using the fact that $\bm{q} \in \mathrm{PC}(\cC)$ necessarily belongs to $\mathrm{PC}(c)$ for each $c \in \cC$, it is clear that \[\mathrm{dMCE}_{\bm{y},\cC}(\bm{p}) \le \mathrm{dMCE}^{\mathrm{strong}}_{\bm{y},\cC}(\bm{p})\] for any prediction sequence $\bm{p}$. To see that this inequality can be strict, consider a setting in which $\cX = \N$ and $x_t = t$ at each time step $t \in [T]$. Consider the collection $\cC_{\mathrm{singleton}}$ consisting of all ``singleton'' indicator functions $c_t$ of the form $c_t(s) = \bm{1}[s = t]$ for some fixed $t \in [T]$. In this case, being perfectly calibrated on the set $\{t\}$ amounts to exactly predicting the $t$\textsuperscript{th} bit---in other words, the event that $p_t = y_t \in \{0, 1\}$. Consequently, the set $\mathrm{PC}(\cC_{\mathrm{singleton}})$ of perfectly $\cC$-multicalibrated prediction sequences is a singleton set that only contains the true outcome sequence $\bm{y}$, which implies that \[\mathrm{dMCE}^{\mathrm{strong}}_{\bm{y},\cC_{\mathrm{singleton}}}(\bm{p}) = \sum_{t=1}^T |p_t - y_t|.\] On the other hand, using the aforementioned characterization of the standard notion of distance to online multicalibration, we see that \[\mathrm{dMCE}_{\bm{y},\cC_{\mathrm{singleton}}}(\bm{p}) = \max_{t \in [T]}\, |p_t - y_t|,\] the \emph{maximum} error made at any particular time step. In particular, in this example, we have that $\mathrm{dMCE}_{\bm{y},\cC_{\mathrm{singleton}}}(\bm{p}) \le 1$ for any prediction sequence $\bm{p}$. However, if $y_t \in \{0, 1\}$ is sampled uniformly and independently of the history of predictions and outcomes before time step $t$, we have $\mathrm{dMCE}^{\mathrm{strong}}_{\bm{y},\cC_{\mathrm{singleton}}}(\bm{p}) \ge \Omega(T)$ with high probability, regardless of the algorithm used to make the predctions at each time step.
\end{proof}

To conclude this section, we show that the \algname{} can be used to achieve small distance to online multicalibration, provided that we aim for the standard notion, as opposed to the strong notion.

\begin{theorem}
\label{thm:dist-to-mc}
    Given a collection $\cC$ of indicator functions for subpopulations of a population $\cX$, let $k_{\mathrm{Lap}} = k_{\R}$ be the \emph{Laplace kernel} as defined in \cref{ex:sobolev}, let $\mathsf{Int}_{\cC} : \cX \times \cX \to \R$ denote the intersection kernel
    \[
        \mathsf{Int}_{\cC}(x, x') = |\{c \in \cC : c(x) = c(x') = 1\}|,
    \]
    and let $k_{\mathrm{MC}} : (\R \times \cX) \times (\R \times \cX) \to \R$ denote the product kernel
    \[
        k_{\mathrm{MC}}((p, x), (p', x')) = k_{\mathrm{Lap}}(p, p') \cdot \mathsf{Int}_\cC(x,x'),
    \]
    which is uniformly bounded by \[m = \max_{x \in \cX} |\{c \in \cC : c(x) = 1\}|.\] Let $\pi_{1:T} = \{(x_t, p_t, y_t)\}_{t=1}^T$ denote the transcript at the end of the \algname{} when instantiated with the kernel $k_{\mathrm{MC}}$. Then, \[\mathrm{dMCE}_{\bm{y},\cC}(\bm{p}) \leq \cO(\sqrt{mT}).\]
\end{theorem}

\begin{proof}
    \cref{thm:indistinguishability_main} guarantees that the transcript ultimately satisfies
    \[
        \left|\sum_{t=1}^T (p_t - y_t)f(p_t)c(x_t)\right| \le \sqrt{m T + 1}
    \]
 for all $f$ with norm at most $1$ in the RKHS corresponding to $k_{\mathrm{Lap}}$, and for all $c \in \cC$ (these have norm at most $1$ in the RKHS corresponding to $\mathsf{Int}_\cC$ by \cref{thm:finite-membership-rkhs}). Next, we fix a particular function $c \in \cC$ and rewrite this inequality as 
    \[
        \left|\sum_{\substack{t \in [T] \\ c(x_t) = 1}} (p_t - y_t)f(p_t)\right| \le \sqrt{m T + 1}.
    \]
    Letting $\bm{y}_c,\bm{p}_c \in [0,1]^{|S|}$ denote the restriction of $\bm{y},\bm{p}\in [0,1]^T$ to the set $S$ of $t \in [T]$ for which $c(x_t) = 1$, this implies that the \emph{kernel calibration error}, defined as follows, also is at most $\sqrt{mT + 1}$:
    \[
        \mathrm{kCE}^{k_{\mathrm{Lap}}}_{\bm{y}_c}(\bm{p}_c) := \sup_{f : \lVert f \rVert_{\mathrm{Lap}} \le 1} \left|\sum_{\substack{t \in [T] \\ c(x_t) = 1}} (p_t - y_t)f(p_t)\right|\le \sqrt{m T + 1}.
    \]
    By Lemma 7.3 of \cite{blasiok2023unifying}, Theorem 8.5 of \cite{blasiok2023unifying}, and Theorem 2 of \cite{qiao2024distance}, we deduce that there exists a prediction sequence $q \in \mathrm{PC}(c)$ (which may depend on $\pi_{1:t}$) such that
    \[
        \sum_{\substack{t \in [T] \\ c(x_t) = 1}} |p_t - q_t| \leq  \cO(\sqrt{mT}).
    \]
    Since our initial choice of $c \in \cC$ was arbitrary, we conclude that
    \[
        \mathrm{dMCE}_{\bm{y},\cC}(\bm{p}) = \sup_{c \in \cC} \inf_{\bm{q} \in \mathrm{PC}(c)} \sum_{t=1}^T |p_t - q_t| = \cO(\sqrt{m T}). \qedhere
    \]
\end{proof}

We remark that if $\cC = \{1\}$ just has the constant one function, then the {\algname} guarantees an asymptotic bound of $\cO(\sqrt{T})$ distance to online calibration. See  \cite{arunachaleswaran2024elementary} for a different algorithm that guarantees a non-asymtotic bound. 

\paragraph{On measuring distance to multicalibration.} A priori, it is not clear from \cref{def:dist-to-mc} how, given a prediction sequence $\bm{p}$, one would go about measuring its distance to online multicalibration. For our standard notion of distance, \cref{thm:dist-to-mc} gives a useful, computable metric for this purpose. Indeed, by \cref{thm:dist-to-mc}, one can upper bound the distance by the kernel calibration error with respect to $k_{\mathrm{MC}}$, given by the following formula:
\[
    \sup_{\substack{ f \in \cF_{\mathrm{MC}} \\ \|f\|_{\cF} \le 1}}\, \sum_{t=1}^T f(x_t, p_t)(y_t - p_t) = \sqrt{\sum_{t=1}^T\sum_{s=1}^T (y_t - p_t)(y_s - p_s)k_{\mathrm{MC}}((x_t, p_t), (x_s, p_s))}.
\]

\subsection{Offline results: weak agnostic learning and online to batch conversions.} \label{sub:offline}

In this section, we shift our attention to the offline setting where samples are drawn i.i.d from some fixed distribution $\cD$. We prove two main results. 

The first shows that one can efficiently solve weak agnostic learning over function classes $\cF$ that are an RKHS. Given the tight connection between weak agnostic learning and multicalibration \cite{hkrr}, this result shows that any multicalibration algorithm that relies on the existence of a weak agnostic learner is unconditionally efficient whenever $\cF$ is an RKHS. 

Second, we show to convert the online learning algorithms into offline algorithms with strong guarantees for the batch setting. This adaptation in particular implies omniprediction and outcome indistinguishability algorithms for the batch case with end-to-end computational efficiency and near-optimal statistical guarantees.

\paragraph{Efficient (strong) learning over an RKHS.} We start by recalling the definition of weak agnostic learning. Here, we state the definition as presented in \cite{gopalan2024swap}:

\begin{definition}[Weak Agnostic Learning]
Let $\cD$ be a distribution over $\cX \times [-1,1]$. Given a comparator class $\cH \subseteq \{\cX \rightarrow [-1,1]\}$, a weak agnostic learner for $\cH$ solves the following promise problem: Given an accuracy parameter $\gamma$, if there exists $h \in \cH$ such that 
\begin{align*}
    \E_{(x,y) \sim \cD} [h(x)y] \geq \gamma 
\end{align*}
then weak agnostic learner returns a function $h':\cX \rightarrow [-1,1]$ (not necessarily in $\cH$) such that 
\begin{align*}
    \E_{(x,y) \sim \cD} [h(x)y] \geq \poly(\gamma).
\end{align*}
\end{definition}

Using the representer theorem, we prove that one can efficient solve a stronger version of the optimization problem above when $\cH$ is an RKHS.

\begin{proposition}[Existence of a Strong Learner over an RKHS]
 Let $k$ be a efficiently computable kernel with associated RKHS $\cF \subseteq \{\cX \rightarrow \R\}$ with $\sup_x k(x,x) \leq 1$ and let $\cF_B \subseteq \cF$ be the subset of functions with norm at most $B$,
\begin{align*}
    \cF_{B} = \{f \in \cF: \|f\|_{\cF} \leq B\}.
\end{align*}
Then, there exists a polynomial-time algorithm such that for any $\gamma \geq 0$,  given $n \geq \poly(1/\gamma, \log(1/\delta))$ samples $(x, y) \sim \cD$, returns a function $f' \in \cF$ such that: 
\begin{align*}
   \Pr\left[\max_{f \in \cF_B} \E_{(x,y) \sim \cD} [f(x)y] - \E_{(x,y) \sim \cD} [f'(x)y] \geq \gamma \right]  \leq \delta.
\end{align*}
\end{proposition}

\begin{proof}
The proof consists of two parts. First, we show that the corresponding empirical risk minimization problem can be solved in polynomial time. Second, we prove a uniform convergence bound showing that the empirical risk and the true risk of the functions in this class are close.  
    Let $S_n = \{(x_i,y_i)\}_{i=1}^n$ for $x_i \in \cX$ and $y_i \in \R$ be a dataset.

Starting with the first part, let $\{(x_i, y_i)\}_{i=1}^n$ be set of samples drawn i.i.d from $\cD$. By the Moore-Aronszajn theorem (\Cref{thm:moore-aronszajn}), we can write any function $f \in \cF$ as $\sum_{i=1}^n \alpha_i \Phi(x_i) + v$ where $v$ lies in the orthogonal complement to 
\begin{align*}
\overline{\mathsf{span}} \{\Phi(x): x \in \{x_i\}_{i=1}^n\}\}.
\end{align*}
Therefore, using the representer theorem, $f(x_i) = \langle f, \Phi(x_i)\rangle_{\cF}$,we can write the following optimization problem over a Hilbert space $\cF$ 
\begin{align*}
\argmax_{f \in \cF_B}  \frac{1}{n} \sum_{i=1}^n f(x_i)y_i 
\end{align*}
as an optimization problem over $\R^n$:
\begin{align*}
&\argmax_{\alpha \in \R^n}  \frac{1}{n} \sum_{i=1}^n \langle \sum_{j=1}^n \alpha_j \Phi(x_j) +v, \Phi(x_i) \rangle_{\cF}  \\ 
& \text{s.t} \quad \langle \sum_{i=1}^n \alpha_i \Phi(x_i),  \sum_{i=1}^n \alpha_i \Phi(x_i) \rangle_{\cF} \leq B^2.
\end{align*}
If we let $K \in \R^{n \times n}$ be the matrix with $k(x_i,x_j) = \langle \Phi(x_i), \Phi(x_j) \rangle_{\cF}$ as its $(i,j)$th entry, this becomes,
\begin{align}
\label{eq:strong_learning_cvx}
&\argmax_{\alpha \in \R^n}  \frac{1}{n} \alpha^\top K y  \\ 
& \text{s.t} \quad \alpha^\top K \alpha \leq B^2. \notag
\end{align}
This is a convex optimization problem (linear objective, quadratic constraints) and can hence be solved to any tolerance $\gamma$ in time polynomial in $n$ and $1/\gamma$.

To finish the proof, we prove a uniform convergence bound showing that all of the functions in $\cF_B$ are close to their empirical counterparts with high probability:
\begin{align}
\label{eq:unif_convergence_Fb}
       \Pr\left[\sup_{f \in \cF_B}|\frac{1}{n}\sum_{i=1}^n f(x_i)y_i - \E f(x) y| \geq B \sqrt{\frac{2\log(1/\delta)}{n}} \right]  \leq \delta.
\end{align}
The proof of this fact follows from observing that by applying the representer theorem and linearity of inner products, we can avoid union bounding over all $f \in \cF_B$ and instead just bound a quantity involving the feature vectors:
\begin{align*}
    |\frac{1}{n} \sum_{i=1}^n f(x_i) y_i - \E[f(x)y]| &= |\frac{1}{n} \sum_{i=1}^n \langle f, \Phi(x_i)\rangle_{\cF} y_i - \E[\langle f, \Phi(x)\rangle_{\cF}y]| \\ 
    & \leq \| f\|_{\cF} \| \frac{1}{n} \sum_{i=1}^n \Phi(x_i)y_i - \E \Phi(x)y \|.
\end{align*}
Now, since $\sup_x k(x,x) = \|\Phi(x)\|_{\cF} \leq 1$ and $y \in [-1,1]$, the vectors $z = \Phi(x)y$ are sub-Gaussian (have norm bounded by 1 a.s). Therefore, we can just apply standard concentration bounds for sub-Gaussian vectors. In particular, we apply Proposition 7 in \cite{maurer2021concentration} (\Cref{lemma:hilbert_sg}) to get that with probability $1-\delta$, 
\begin{align*}
   \| \frac{1}{n} \sum_{i=1}^n \Phi(x_i)y_i - \E \Phi(x)y \|_{\cF} \leq 8e \sqrt{\frac{2\log(1/\delta)}{n}}.
\end{align*}
This completes the proof of the claim in \Cref{eq:unif_convergence_Fb}. The proof of the main result then follows directly by combining this concentration result with the optimization fact from \Cref{eq:strong_learning_cvx}. In particular, let $f'$ be an $\gamma$ approximate optima for \Cref{eq:strong_learning_cvx} (which can be computed in polynomial time), and let $f$ be any other function in $\cF_B$. Then, 
\begin{align*}
    \E[f'(x)y] & \geq \frac{1}{n} \sum_{i=1}^n f'(x_i)y_i - \cO(B \sqrt{\log(1/\delta) /n}) \\
    & \geq \frac{1}{n}\sum_{i=1}^n f(x_i)y_i - \cO(B \sqrt{\log(1/\delta) /n}) - \gamma \\
    & \geq \E[f(x_i)y_i] - \cO(B \sqrt{\log(1/\delta) /n}) - \gamma.
\end{align*}
Letting $n \geq  \poly(B, \gamma^{-1}, \log(1/\delta)))$, we get that $\E[f'(x)y] \geq \sup_{f \in \cF_B} \E[f(x)y] - \cO(\gamma)$.
\end{proof}

\paragraph{Online to batch conversions.} For the sake of completeness, we also illustrate how one can convert any of the online algorithms we study in this paper into batch algorithms. The proof of the following result is somewhat standard and uses classical martingale decompositions, but we include it for completeness. 

\begin{proposition}
Let $k$ be a kernel with RKHS $\cF$ satisfying $$\sup_{x \in \cX,p\in  [0,1]} k((x,p),(x,p))  \leq B < \infty$$ and let $\{(x_i, y_i)\}_{i=1}^n$ be a dataset of i.i.d samples drawn from a fixed distribution $\cD$ over $\cX \times \{0,1\}$.

Furthermore, let $S = \{(x_i, y_i, p_i\}_{i=1}^n$ be transcript generated from running the \algname{} on the samples $(x_i, y_i)$ and $h_i: \cX \rightarrow [0,1]$ be the randomized function induced by the \algname{} conditioned on $\pi_{1:i-1} = \{(x_j, y_j)\}_{j=1}^{i-1}$. 

If we define $\hbar_S$ be the randomized predictor which selects a function from the set $\{h_i\}_{i=1}^n$ uniformly, then with probability $1-\delta$ over the randomness of the $n$ samples and the predictor $\hbar_S$, the following inequality holds for all $f \in \cF$,where $c_0$ and $c_1$ are universal constants:
\begin{align*}
     \E_{S \sim \cD^{(n)}, (x,y) \sim \cD}[(y - \hbar_S(x))f(x, \hbar_S(x))] \leq  c_0\frac{1}{\sqrt{n}} \fnorm{f} B  +  c_1\sqrt{\frac{1 + \log(1 /\delta)}{n}}.
\end{align*}
\end{proposition}
\begin{proof}
We use a similar decomposition as in the previous results. We start by using the reproducing property of the RKHS, linearity of expectation and then applying Cauchy-Schwarz:
\begin{align*}
        \E[(y - \bar{h}_S(x))f(x, \bar{h}_S(x))] &= \E[(y - \bar{h}_S(x))\langle f, \Phi(x, \bar{h}_S(x)) \rangle_{\cF}] \\ 
        & = \langle f, \E[(y - \bar{h}_S(x)) \Phi(x, \bar{h}_S(x))] \rangle_{\cF}  \\ 
        & \leq \fnorm{f} \cdot \fnorm{\E[(y - \bar{h}_S(x)) \Phi(x, \bar{h}_S(x))] }.
\end{align*}
Having done this, the proposition follows by combining the following two statements:
\begin{align}
\label{eq:martingale_decomp}
     \E[(\hbar_S(x) - y)\Phi(\hbar_S(x),x)] \lesssim  \fnorm{ \frac{1}{n} \sum_{i=1}^n(p_i - y_i)\Phi(x_i, p_i)} +  \sqrt{\frac{1 + \log(1 /\delta)}{n}},
\end{align}
\begin{align*}
    \fnorm{\sum_{i=1}^n(p_i - y_i)\Phi(x_i, p_i)} \leq \sqrt{\sum_{i=1}^n\E p_i(1-p_i)} \leq \sqrt{n} 
\end{align*}
where the second one is exactly the guarantee shown for the \algname{} from \Cref{thm:indistinguishability_main} (see \Cref{eq:feature_regret}). We hence now focus on establishing the bound in \Cref{eq:martingale_decomp}. By definition of $\bar{h}_S$, 
\begin{align}
\label{eq:total_probability}
    \E[(\hbar_S(x) - y)\Phi(\hbar_S(x),x)] &= \sum_{i=1}^n \E[(h_i(x) - y)\Phi(x,h_i(x)) \mid h_i] \Pr[\hbar_S=h_i] \\ 
    &= \frac{1}{t}\sum_{s=1}^t \E[(h_s(x) - y) \Phi(x,h_s(x)) \mid h_s] .\notag
\end{align}
Now consider the following Hilbert-space valued martingale sequence $V_i$ adapted to the filtration $\cB_i = \sigma(\{(x_{i}, y_{i}), p_{o}\}_{i=1}^n)$ where $V_0 = 0$ and 
\begin{align*}
    V_{i+1} = V_{i} + \E_{(x,y) \sim \cD}[(h_i(x) - y)\Phi(x, h_i(x)) \mid \cB_{i-1}] - (p_i - y_i)\Phi(x_i, p_i).
\end{align*}
We can easily check that this process is indeed a martingale. Clearly, $V_t$ is adapted to $\cB_t$. Furthermore, since $\fnorm{(p_t -y_t)\Phi(x_t, p_t)} \leq B$, then $\E\fnorm{V_i} < \infty$. Lastly, since 
\begin{align*}
    \E[(p_i - y_i)\Phi(x_i, p_i) \mid \cB_{i-1}] = \E_{(x,y) \sim \cD}[(h_i(x) - y)\Phi(x, h_i(x)) \mid \cB_{i-1}],
\end{align*}
then, 
\begin{align*}
    \E[V_{i+1} \mid \cB_i] = \E[V_i \mid \cB_i] + 0 = V_i.
\end{align*}
Rewriting $V_i$ as 
\begin{align*}
    \sum_{i=1}^n\E_{(x,y) \sim \cD}[(h_i(x_i) - y_i)\Phi(x_i, h_i(x_i)) \mid \cB_{i-1}] - (p_i - y_i)\Phi(x_i, p_i)
\end{align*}
Using the Azuma-Hoeffding deviation inequality from \cite{naor2012banach} (\Cref{lemma:naor}), there exists a universal constant $c'$ such that with probability $1-\delta$, 
\begin{align*}
    \fnorm{V_i} \leq c' \sqrt{t \log(e^3 /\delta)},
\end{align*}
and hence by the reverse triangle inequality, 
\begin{align*}
    \fnorm{\sum_{s=1}^t\E_{(x,y) \sim \cD}[(h_s(x) - y)\Phi(x, h_s(x)) \mid \cB_{s-1}]} \leq \fnorm{\sum_{s=1}^t(\pt_s - y_s)\Phi(x_s, \pt_s)}.
\end{align*}
Plugging this into the decomposition from \Cref{eq:total_probability}, we get that with probability $1-\delta$,
\begin{align*}
     \E[(\hbar_t(x) - y)\Phi(\hbar(x),x)] \leq  \fnorm{ \frac{1}{t} \sum_{s=1}^t(\pt_s - y_s)\Phi(x_s, \pt_s)} +  c_0' \sqrt{\frac{\log(e^3 /\delta)}{t}}.
\end{align*}
This establishes our two previous conditions and hence concludes the proof of the result.
\end{proof}

\begin{lemma}[Theorem 1.5 in \cite{naor2012banach}]
\label{lemma:naor}
Let $\cF$ be a Hilbert space and let $\{V_t\}_{t=0}^\infty$ be an $\cF$-valued martingale satisfying $\fnorm{V_{t+1} - V_t} \leq 2$ for all $t\geq 0$. 
Then, there exists a universal constant $c_0$ such that for all $u > 0$ and positive integers $t \geq 0$,
\begin{align*}
    \Pr[\fnorm{V_{t} - V_0} \geq u] \leq e^3 \exp\left( \frac{-cu^2}{4t}\right).
\end{align*}
\end{lemma}

\begin{lemma}[Proposition 7 in \cite{maurer2021concentration}]
\label{lemma:hilbert_sg}
If $\cF$ is a Hilbert space and $\{X_i\}_{i=1}^n$ are i.i.d random variables taking values in $\cF$ such that $\fnorm{X_i}\leq B$. If $n \geq \log(1/\delta) \geq \log(2)$, then with probability $1-\delta$,
\begin{align*}
\fnorm{\frac{1}{n} \sum_{i=1}^n X_i - \E[X]} \leq 8eB \sqrt{\frac{2\log(1/\delta)}{n}}.
\end{align*}
\end{lemma}

\section*{Acknowledgments}

We would like to thank Aaron Roth for helpful comments and discussion on online algorithms and Tina Eliassi-Rad for pointers to the networking literature.
This work was supported in part by Simons Foundation Grant 733782 and Cooperative Agreement CB20ADR0160001
with the United States Census Bureau. JCP was in part supported by the Harvard Center for Research of Computation and Society.

\bibliography{refs}

\appendix

\section{Background on Reproducing Kernel Hilbert Spaces}

\label{sec:rkhs_background}

\subsection{Definition and properties.}

We start with a more detailed definition of an RKHS and some of its key properties.

\begin{definition}[Reproducing Kernel Hilbert Spaces]
    A set of functions $\cF \subseteq \{ f \; : \; \cX \to \cR \}$ is a \textnormal{reproducing kernel Hilbert space (RKHS)} if it satisfies the following properties.
    \begin{enumerate}
        \item There exists an inner product $\langle \cdot, \cdot\rangle_{\cF} \; : \; \cF \times \cF \to \cR$. That is, $\innerF{\cdot}{\cdot}$ is symmetric, linear in its first argument, and positive definite (for all $f$, $\innerF{f}{f} \geq 0$, and $\innerF{f}{f} = 0$ if and only if $f = 0$).
        \item The space is complete with respect to the norm $\| f \|_{\cF} \defeq \sqrt{\innerF{f}{f}}$. That is, for all Cauchy sequences $f_1, f_2, \dots \in \cF$, it holds $\lim_{i \to \infty} f_i \in \cF$.
        \item For all $x \in \cX$, there exists a function $K_x \in \cF$ such that
        \begin{align*}
            f(x) = \innerF{f}{K_x}
        \end{align*}
        for all $f \in \cF$ where $\innerF{\cdot}{K_x}$ is continuous.
        %
        %
    \end{enumerate}
\end{definition}
\noindent The map $\innerF{\cdot}{K_x} \; : \; \cF \to \cR$ is called the evaluation functional. The function $K(x, x') \defeq \innerF{K_x}{K_{x'}}$ is called the \textit{reproducing kernel} (or \textit{kernel} for short) of $\cF$. 
Next, we define positive semi-deminite functions, which will be used in \Cref{thm:moore-aronszajn}.

\begin{definition}[PSD function]
A symmetric function $k: \cX \times \cX \rightarrow \cR$ is positive semi-definite if for all $n \in \mathbb{N}$:
\begin{align*}
    \sum_{i=1}^n\sum_{j=1}^n \lambda_i \lambda_j k(x_i, x_j) \geq 0
\end{align*}
for all $x_1, \dots, x_n \in \cX$ and $\lambda_1, \dots, \lambda_n \in \cR$.
\end{definition}

The next theorem states that each positive semi-definite function corresponds to a unique RKHS.
    
\begin{theorem}[Moore-Aronszajn Theorem]
\label{thm:moore-aronszajn}
Let $k:\cX \times \cX \rightarrow \cR$ be a positive semi-definite function.  Then, there is a unique RKHS $\cF \subseteq \{f \; : \; \cX \rightarrow \R\}$ for which $k$ is the reproducing kernel. Moreover, $\cF$ consists of the completion of the linear span of $\{ k(\cdot, x) \; | \; x \in \cX\}$, \textit{i.e.}, the set
\begin{align*}
    \left\{ \sum_{i=1}^\infty \alpha_i k(\cdot, x_i)  \;\; \middle| \;\; \alpha_i \in \cR, x_i \in \cX, \lim_{m \to \infty} \sup_{n \geq m} \left\| {\sum_{i=m}^n \alpha_i k(\cdot, x_i)} \right\|_{\cF} = 0 \right\}.
\end{align*}
\end{theorem}

\noindent For example, if $|\cX|< \infty$ then, the RKHS induced by $k$ is
\begin{align*}
    \cF \defeq \Set{\sum_{x_i \in \cX } \alpha_{i}k(\cdot, x_i)\; : \; \alpha_i \in \cR}.
\end{align*}

Next, we state several lemmas that are useful for our analysis.

\begin{lemma}[Corollary to \cref{thm:moore-aronszajn}] \label{lem:scalarmult}
    Let $\cF$ be a RKHS on $\cX$. Then the zero function $x \mapsto 0$ is in $\cF$, and, more generally, for all $f \in \cF$ and $\alpha \in \R$, any linear function $x \mapsto \alpha f(x)$ is in $\cF$.
\end{lemma}

\begin{lemma}[Theorem 5.4, \cite{paulsen2016introduction}]\label{lem:kernelsum}
Let $k_1$ and $k_2$ be positive semi-definite kernels on $\cX$ with associated RKHSs $\cF_1$ and $\cF_2$ then $k = k_1 + k_2$ is a valid kernel with associated RKHS $\cF$ equal to the completion of the span of $$\{f_1 + f_2: f_1 \in \cF_1, f_2 \in \cF_2\}.$$ Moreover, direct implication of the above result is that, for $f_1 \in \cF_1, f_2 \in \cF_2$, $\| f_1 + f_2\|_\cF \leq \| f_1 \|_{\cF_1} + \| f_2 \|_{\cF_2}$.
\end{lemma}
\noindent A direct implication of the above result, since the zero function $x \mapsto 0$ is in every RKHS, is that $\cF_1 \cup \cF_2 \subseteq \cF$. 

\begin{lemma} [Theorem 5.11, \cite{paulsen2016introduction}]\label{lem:productkernel}
Let $k_1: \cX \times \cX \rightarrow \R$ and $k_2: \cY \times \cY \rightarrow \R$ be positive semi-definite kernels with associated RKHSs $\cF_1$ and $\cF_2$ then $k((x,y), (x, y')) = k_1(x, x')k_2(y,y')$ is a valid kernel. Furthermore, its associated function space is the completion of the span of the set $$\{f_1 \cdot f_2 \; : \; f_1 \in \cF_1, f_2 \in \cF_2  \}$$
where for any $f_1 \in \cF_1, f_2 \in \cF_2$ we define $f_1 \cdot f_2 \; : \; \cX \times \cY \to \cR$ to be the function $ (f_1 \cdot f_2)(x, y) = f_1(x)f_2(y) $ for all $ (x,y) \in \cX \times \cY$.
Moreover, for $f_1 \in \cF_1, f_2 \in \cF_2$, $\| f_1 \cdot f_2\|_\cF \leq \| f_1 \|_{\cF_1} \| f_2 \|_{\cF_2}$.
\end{lemma}

\begin{lemma}[Theorem 5.7, \cite{paulsen2016introduction}] \label{lem:compositionkernel}
    For any function $\phi \; : \; \cX \to \cR$ and RKHS $\cF_0 \subseteq \{ f \; : \; \cR \to \cR \}$ associated with kernel $k$, there exists an RKHS $\cF_{1}$ equal to the completion of the span of the set $ \{ f \circ \phi \; : \; f \in \cF_0 \}$ and
    associated with kernel $k \circ \phi \defeq k(\phi(\cdot), \phi(\cdot))$.
    Moreover, it holds $\| f \circ \phi \|_{\cF_1} \leq \| f \|_{\cF_0}$.
\end{lemma}

\begin{lemma}
\label{thm:finite-membership-rkhs}
    Let $\cX$ be any set and let $\cI$ be any index set. Let $\cF = \{f_i\}_{i \in \cI}$ be a collection of functions $f_i : \cX \to \R$ indexed by $\cI$. Suppose that for each $x \in \cX$, we have
    \begin{equation}
        \sum_{i \in \cI} f_i(x)^2 < m
    \end{equation}
    for some constant $m$,
    in which case the function $k : \cX \times \cX \to \R$ given by \[k(x, y) = \sum_{i \in \cI} f_i(x)\,f_i(y)\] is a valid kernel. Then, the RKHS $\cF$ corresponding to $k$ contains $\cF$, and $\|f_i\|_{\cF} \le 1$ for each $i \in \cI$.
\end{lemma}

\begin{proof}[Proof of \cref{thm:finite-membership-rkhs}]
    We introduce several pieces of notation:
    \begin{itemize}[$\bullet$]
        \item Let $\cH$ be the Hilbert space of ``coefficient sequences'' $\alpha : \cI \to \R$ that are $L^2$ bounded by $m$ with respect to the counting measure on $\cI$, which means that \(\sum_{i \in \cI} \alpha(i)^2 < m\).
        \item For each $x \in \cX$, define a coefficient sequence $\Phi_x : \cI \to \R$ by the formula $\Phi_x(i) = f_i(x)$. Note that $\Phi_x \in \cH$ by the assumption that $\sum_{i \in \cI} f_i(x)^2$ is finite. Note also that the kernel function $k$ satisfies \[k(x, y) = \langle \Phi_x,\, \Phi_{y} \rangle_{\cH}\] for any $x,y \in \cX$.
        \item Given a coefficient sequence $\alpha \in \cH$, let $f_\alpha : \cX \to \R$ denote the function \[f_\alpha(x) = \langle \alpha,\,\Phi_x \rangle_{\cH} = \sum_{i \in \cI} \alpha(i)\,f_i(x).\]
        \item Let $V \subseteq \cH$ be the closure in $\cH$ of the subspace $\mathrm{span}\{\Phi_x:x\in\cX\}$. In other words, let $V$ be the set of all finite linear combinations of coefficient sequences $\Phi_x$ for $x \in \cX$, together with their limit points in $\cH$. Relatedly, let $\mathrm{proj}_V$ denote the orthogonal projection of $\cH$ onto $V$, which satisfies $\mathrm{proj}_V(\alpha) \in V$ and \begin{equation}
            \label{eq:finite-membership-rkhs-projection-property}
            \big\langle \alpha - \mathrm{proj}_V(\alpha),\, \Phi_x \big\rangle_{\cH} = 0
        \end{equation} for each $\alpha \in \cH$ and $x \in \cX$.
    \end{itemize}
    Rephrased in this language, Moore-Aronszajn theorem and its proof simply show that the map $\alpha \mapsto f_\alpha$ is a distance-preserving, one-to-one correspondence (\textit{i.e.} an isometric isomorphism) between $V$ and the RKHS $\cF$ corresponding to the kernel $k$. Next, by \cref{eq:finite-membership-rkhs-projection-property} with $\alpha = e_i$, we see that for all $x \in \cX$ and $i \in \cI$, \[f_i(x) = \langle e_i,\, \Phi_x \rangle_{\cH} = \big\langle \mathrm{proj}_V(e_i),\, \Phi_x \big\rangle_{\cH} = f_{\mathrm{proj}_V(e_i)}(x).\] Here, $e_i$ denotes the $i$\textsuperscript{th} standard basis coefficient sequence \[e_i(j) = \begin{cases} 1 & \text{if } i = j, \\ 0 & \text{if } i \neq j.\end{cases}\] Using the aforementioned distance-preserving correspondence between $V$ and $\cF$, we see that \[\|f_i\|_{\cF} = \big\|f_{\mathrm{proj}_V(e_i)}\big\|_{\cF} = \|\mathrm{proj}_V(e_i)\|_{\cH} \le \|e_i\|_{\cH} = 1,\] which concludes the proof.
\end{proof}

We also remark that if $\cF$ is a (not necessarily finite) collection of indicator functions for subsets $S_i \subseteq \cX$ but each $x \in \cX$ belongs to at most finitely many such $S_i$, then \cref{eq:finite-membership-rkhs-assumption} is satisfied, so \cref{thm:finite-membership-rkhs} implies that the RKHS corresponding to the intersection kernel \[k(x,y) = \mathsf{Int}_\cF(x,y) = \big|\{i \in \cI : x, y \in S_i\}\big|\] contains all functions in $\cF$ and that their norms in $\cF$ are at most $1$.

\subsection{Key examples.}

\begin{example}[Linear functions]
Let $\cX=\R^d$, then $\flin$, the space of all linear functions from $\R^d$ to $\R$, defined as, $$\flin = \{f_w: w\in \R^d, f(x) = x \cdot w \} \subseteq \{\R^d \rightarrow \R\}$$ is an RKHS with corresponding kernel $k_{\mathrm{lin}}(x,x') = x \cdot x' = \sum_{i=1}^d x_i x_i'$ equal to the standard inner product. The feature mapping is just the identity function $\Phi(x) = x$. Note that each element $f\in \cF$ could be thought of both as a function from $\R^d$ to $\R$ as well as an element in the Hilbert space (which in this case is just $\R^d$). However, going back to our earlier comment, we see that we could have equivalently written out $\flin$ as,
\begin{align*}
    \flin = \mathsf{span}\Set{\sum_{x_i \in \cX } \alpha_{i}k_{\mathrm{lin}}(\cdot, x_i)\; : \; \alpha_i \in \R} = \mathsf{span}\Set{\sum_{x_i \in \R^d } \alpha_{i}x_i\; : \; \alpha_i \in \R}.
\end{align*}
\end{example}

\begin{example}[Polynomial functions] \label{ex:polynomialkernels}
    Consider the set of polynomials of degree $\leq k$ on $d$ variables with the inner product defined as the inner product of the coefficients on each monomial.
    In this case, $\cX = \R^d$.
    Since the space of coefficients is just $\R^\ell$ for some appropriate $\ell$ (depending on the dimension of the input space $d$ and $k$), it is complete and the inner product satisfies all the necessary properties.

    Then, to show that this has the reproducing property, let $K_x$ be the polynomial where the coefficient on a given monomial is determined by multiplying together the corresponding entries of $x$. So the coefficient on the $x_1 x_2^3$ term is the first entry of $x$ times the cube of the second entry  of $x$.
    Then, notice that for all $f \in \cH$, $f(x) = \inner{f}{K_x}_{\cH}$.
    It can be shown that the corresponding kernel is $$k(x, y) = (1 + \langle x , y \rangle)^k. $$
\end{example} 

\begin{example} [Boolean functions] \label{ex:booleanfunctions}
Consider the set of functions taking the form $f \; : \; \{ -1, 1 \}^d \to \{ -1, 1 \}$. 
\newcommand{\indicax}{1_{a}(x)}
First, notice that we can write $f$ as a polynomial. For $a,x \in \{ -1, 1 \}^d$, define the indicator polynomial
\begin{align*}
    \indicax &= \paren{\frac{1 + a_1 x_1}{2}}\paren{\frac{1 + a_2 x_2}{2}}\cdot\cdot\cdot \paren{\frac{1 + a_d x_d}{2}} \\
    &= \begin{cases}
        1 & \; \text{if } a = x \\
        0 & \; \text{otherwise}.
    \end{cases}
\end{align*}
Then, notice
\begin{align*}
    f(x) = \sum_{a \in \{ -1, 1 \}^d} f(a) \indicax.
\end{align*}
This is just the sum of $2^d$ different order $d$ polynomials and therefore a polynomial of order $d$.
Thus, Boolean functions are a subset of the polynomials and we can use the kernel $k(x,y) = (1 + \langle x, y \rangle)^d$.
The inner product is also the same as for the polynomials: the inner product is just the inner product of the coefficients on each monomial.

In fact, if we distribute the products in $\indicax$, we can see that every Boolean function can be written as
\begin{align*}
    f(x) = \sum_{I \in 2^d} \alpha_I x_I
\end{align*}
for $\alpha_I \in \R$ a constant and $x_I \defeq \prod_{i \in I}x_i$. See \cite{odonnel2021analysis} for more discussion of Boolean functions.
\end{example}

\begin{example}[Regression trees] \label{ex:dectrees}
As a special case of Boolean functions, we will write down the functions representing regression trees on Boolean inputs.
For a given regression tree, let $b \in \{0, 1\}^k$ represent the path down the decision tree, where $b_i = 0$ means go to the left child (\textit{i.e.}, the the decision variable in the $i$th decision following path $b$ is 0) and $b_i = 1$ means go to the right child at depth $i$. 
Let $c_b$ be the leaf assigned to path $b$. Let $i_{b, j}$ represent the index of the decision variable at the $j$th decision following path $b$. Then any decision tree can be specified by $\{ c_b \}_{b \in \{ 0, 1 \}^k}$ and $\{ i_{b, j} \}_{b \in \{ 0, 1 \}^k, j \in [k]}$:

    \begin{align*}
        f(x) = \sum_{b \in \{ 0, 1 \}^k} c_b \prod_{\ell = 0}^{k-1} \paren{(1- x_{i_{b, \ell}})(1 - b_\ell) + x_{i_{b, \ell}} b_\ell}
    \end{align*}
\end{example}

\begin{example}[Sobolev spaces $W^{1,2}(\Omega)$ for $\Omega \in \{ \brac{0,1}, \cR\}$, ] \label{ex:sobolev} 
This example comes from \cite{berlinet2011reproducing}, Section 7.4, Examples 13 and 24.
Consider the set of functions $\cF_0 \subseteq \{ \Omega \to \R \}$ for $\Omega \in \{ \brac{0,1}, \cR\}$ such that 
\begin{enumerate}[(a)]
    \item each function is differentiable almost everywhere and continuous, and
    \item each function and its derivative are square integrable.
\end{enumerate}
The completion of $\cF_0$ with respect to the norm 
\begin{align*}
    \| f \|_{\cF_0}^2 = \int_\Omega (f(x))^2 \; dx+ \int_\Omega (f'(x))^2 \; dx.
\end{align*}
is an RKHS $\cF$ (usually denoted $W^{1,2}(\Omega)$) where, if $\Omega = [0,1]$, the kernel is
\begin{align*}
    k_{[0,1]}(x, x') = \frac{(e^{x} + e^{-x})(e^{1-x'} + e^{x'-1})}{2(e  - e^{-1})} < 3.
\end{align*}
for $0 \le x \le x' \le 1$ and $k_{[0, 1]}(x, x') = k_{[0, 1]}(x', x)$ if $0 \le x' \le x \le 1$. If $\Omega = \cR$, the kernel is
\begin{align*}
    k_{\cR}(x, x') = \exp\{-\abs{x-x'}\}.
\end{align*}
The inner product in $\cF$ for differentiable functions $f,g \in \cF$ is
\begin{align*}
    \inner{f}{g}_{\cF} = \int_\Omega f(x) g(x) \; dx + \int_\Omega f'(x) g'(x) \; dx.
\end{align*}

\end{example}

Next, we state the following simple lemma about the composition of functions in $W^{1,2}([0,1])$.
For a set of differentiable functions $\cF$, let $\cF' = \{ f' \; | \; f \in \cF\}$ denote the set of derivatives.
\newcommand{\sob}{{W^{1,2}([0,1])}}
\begin{lemma} \label{lem:cdbfd}
    Suppose that there exists a universal constant $B \geq 1$ and sets of differentiable functions $\cF_0, \cF_1$ with $\mathrm{Im}(\cF_0) \subseteq [0,1]$, $\| \cF_0 \|_\sob \leq B$, $\mathrm{Im}(\cF_1) \subseteq [-B,B]$, and $\mathrm{Im}(\cF_1') \subseteq [-B,B]$.
    Then, $\{ f_1 \circ f_0 \; | \; f_0 \in \cF_0, f_1 \in \cF_1 \} \subseteq \cF$ and $\| f_1 \circ f_0 \|_{\cF} \leq 2B^2$.
\end{lemma}

\begin{proof}
    Fix $f_0 \in \cF_0, f_1 \in \cF_1$.
    Notice that by the uniform boundedness of $f_1$, $\| f_1 \circ f_0 \|_{L^2([0,1])} \leq B$. Also, $\| f_0' \|_{L^2([0,1])} \leq \| f_0' \|_{W^{1,2}([0,1])} \leq B$. Then,
    \begin{align*}
        \| (f_1' \circ f_0) f_0' \|_{L^2([0,1])} &\leq \| f_1' \circ f_0\|_{L^2([0,1])} \| f_0' \|_{L^2([0,1])} \\
        &\leq B^2
    \end{align*}
    where the first line comes from the Cauchy-Schwarz inequality and the second line comes from the plugging in the bounds on each norm.
    Also, by the uniform boundedness of $\cF_1$, $\| f_1 \circ f_0 \| \leq B$, which implies the desired bound.
    See, e.g., \cite{evans2018measure}, Theorem 4.4, part (ii) for more general conditions on the composition of functions in a Sobolev space.
\end{proof}

    


\begin{example}[Low-degree functions on $\{0, 1 \}^n$, \cite{Shawe-Taylor_Cristianini_2004}, Section 9.2] \label{ex:anova}
    Consider the set of functions $\cF_0 \subseteq \{ \{ -1,1\}^n \to [-1,1] \}$ whose Fourier spectrum is supported on monomials of degree at most $d$. 
    The kernel associated with the completion of $\cF_0$ is 
    \begin{align*}
        k(x, x') = \sum_{S \subseteq [n], |S| \leq d} x_S x_S'.
    \end{align*}
\end{example}

\subsection{Matrix-valued kernels}

We now introduce two standard definitions related to \emph{matrix-valued} kernels and their corresponding vector valued reproducing kernel Hilbert spaces. These standard facts can be found, for example, in \cite{vectorRKHSreview, minhvector}.

\begin{definition}
\label{def:vv-kernel}
    We say that a matrix-valued function $k : \cX \times \cX \to \R^{d \times d}$ is a valid \emph{kernel} if the following two ``positive semidefiniteness'' properties hold:
    \begin{itemize}
        \item For all $x, y \in \cX$, we have $k(x, y) = k(y, x)^\top$.
        \item For all $n \in \N$ and $x_1, \ldots, x_n \in \cX$ and $w_1, \ldots, w_n \in \R^d$, we have \[\sum_{a=1}^n\sum_{b=1}^n \langle w_a, k(x_a, x_b)w_b\rangle \ge 0.\]
    \end{itemize}
\end{definition}

\begin{definition}
\label{def:vv-rkhs}
    Given a matrix-valued kernel $k : \cX \times \cX \to \R^{d \times d}$, the \emph{reproducing kernel Hilbert space (RKHS) $\cF$ corresponding to $k$} is a Hilbert space consisting of vector-valued functions $f : \cX \to \R^d$. Specifically, $\cF$ is the completion of the space of all linear combinations of functions of the form \[x \mapsto \sum_{a=1}^n k(x, x_a)w_a\] for some $n \in \N$ and $x_1, \ldots, x_n \in \cX$ and $w_1, \ldots, w_n \in \R^d$. It is imbued with the unique inner product $\langle \cdot, \cdot \rangle_\cF : \cF \times \cF \to \R$ satisfying the following property: for all $x_1, x_2 \in \cX$ and $w_1, w_2 \in \R^d$, the inner product of the functions $f_1(x) = k(x, x_1)w_1$ and $f_2(x) = k(x, x_2)w_2$ is \[\langle f_1, f_2 \rangle_\cF = \langle w_1, k(x_1, x_2) w_2 \rangle,\] where the inner product on the right hand side is the standard inner product on $\R^d$.
\end{definition}

The following result illustrates how one might represent any finite set of vector valued functions using a matrix valued kernel:

\begin{lemma}
\label{thm:finite-membership-vv-rkhs}
    Let $\cX$ be any (not necessarily finite) population set and let $\cI$ be any (not necessarily finite) index set. Let $\cC = \{c_i\}_{i \in \cI}$ be a collection of functions $c_i : \cX \to \R^d$ indexed by $\cI$. Suppose that for each $x \in \cX$, we have
    \begin{equation*}
        \label{eq:finite-membership-vv-rkhs-assumption}
        \sum_{i \in \cI} \lVert c_i(x) \rVert^2 < \infty,
        \tag{$*$}
    \end{equation*}
    in which case the matrix-valued function $k : \cX \times \cX \to \R^{d \times d}$ given by \[k(x, y) = \sum_{i \in \cI} c_i(x)\,c_i(y)^\top\] is a valid kernel. Then, the RKHS $\cF$ corresponding to $k$ contains $\cC$, and $\|c_i\|_{\cF} \le 1$ for each $i \in \cI$.
\end{lemma}

\begin{proof}
    Given a fixed element $y \in \cX$ and $a \in \R^d$, consider the following vector-valued function from $\cX$ to $\R^d$: \[x \mapsto k(x,y)a.\] By \cref{def:vv-rkhs}, we know that the RKHS $\cF$ corresponding to the matrix-valued kernel $k$ is the completion of the set of all linear combinations of vector-valued functions of the above form. Next, consider the following related \emph{scalar-valued} kernel $k_{\mathrm{scalar}} : (\cX \times [d]) \times (\cX \times [d]) \to \R$, defined as follows: \[k_{\mathrm{scalar}}((x,a),(y,b)) = k(x,y)_{ab}.\] The RKHS $\cF_{\mathrm{scalar}}$ corresponding to $k_{\mathrm{scalar}}$ is given by the Moore-Aronszajn Theorem (\cref{thm:moore-aronszajn}), and comparing this description to the aforementioned description of $\cF$, it becomes clear that $\cF$ and $\cF_{\mathrm{scalar}}$ are isometrically isomorphic, i.e. there is a one-to-one, length-preserving correspondence between elements of $\cF$ and elements of $\cF_{\mathrm{scalar}}$. Specifically, the isomorphism maps a function $f : \cX \to \R^d$ in $\cF$ to the function $f_{\mathrm{scalar}} : \cX \times [d] \to \R$ given by \[f_{\mathrm{scalar}}(x, a) = f(x)_a\] for each $x \in \cX$ and $a \in [d]$. By \cref{thm:finite-membership-rkhs}, the space $\cF_{\mathrm{scalar}}$ contains the function $c_{\mathrm{scalar}} : \cX \times [d] \to \R$ for each $c \in \cC$, and these functions all have norm $\lVert c_{\mathrm{scalar}} \rVert_{\cF_{\mathrm{scalar}}} \le 1$. Consequently, $\cC \subseteq \cF$ and $\lVert c \rVert_{\cF} \le 1$ for each $c \in \cC$, as well.
\end{proof}

\end{document}

%% file: macros.tex
\usepackage{diagbox}
\usepackage{makecell}
\usepackage{booktabs}
\usepackage{tikz,tikz-3dplot}
\usepackage{amsmath}
\usepackage{bbm}
\usepackage{amsfonts}
\usepackage{amssymb}
\usepackage{amsthm}
\usepackage{url,ifthen}
\usepackage[shortlabels]{enumitem}
\usepackage{srcltx}
\usepackage{dsfont}
\usepackage{multirow}
\usepackage{boxedminipage}
\usepackage[margin=1.1in]{geometry}
\usepackage{nicefrac}
\usepackage{xspace}
\usepackage{graphicx}
\usepackage{color}
\usepackage{subfigure}
\usepackage{colortbl}
\usepackage{setspace}
\usepackage{pgfplots}
\pgfplotsset{compat=1.12}
\usepackage{algorithm,algorithmic}
\usepackage[algo2e]{algorithm2e}

\usepackage{xcolor}
\definecolor{DarkGreen}{rgb}{0.1,0.5,0.1}
\definecolor{DarkRed}{rgb}{0.5,0.1,0.1}
\definecolor{DarkBlue}{rgb}{0.1,0.1,0.5}
\definecolor{Gray}{rgb}{0.2,0.2,0.2}

\definecolor{c1}{RGB}{38, 70, 83}
\definecolor{c2}{RGB}{42, 157, 143}
\definecolor{c3}{RGB}{233, 196, 106}
\definecolor{c5}{RGB}{231, 111, 81}
\definecolor{c4}{RGB}{244, 162, 97}

\definecolor{c1}{RGB}{38, 70, 83}
\definecolor{c2}{RGB}{42, 157, 143}
\definecolor{c3}{RGB}{233, 196, 106}
\definecolor{c5}{RGB}{231, 111, 81}
\definecolor{c4}{RGB}{244, 162, 97}

\usepackage{listings}
\lstdefinestyle{mystyle}{
    commentstyle=\color{DarkBlue},
    keywordstyle=\color{DarkRed},
    numberstyle=\tiny\color{Gray},
    stringstyle=\color{DarkGreen},
    basicstyle=\footnotesize,
    breakatwhitespace=false,         
    breaklines=true,                 
    captionpos=b,                    
    keepspaces=true,                 
    numbers=left,                    
    numbersep=5pt,                  
    showspaces=false,                
    showstringspaces=false,
    showtabs=false,                  
    tabsize=2
}
\usepackage[small]{caption}
\usepackage[pdftex]{hyperref}
\hypersetup{
    unicode=false,          
    pdftoolbar=true,        
    pdfmenubar=true,        
    pdffitwindow=false,      
    pdfnewwindow=true,      
    colorlinks=true,       
    linkcolor=DarkBlue,          
    citecolor=DarkGreen,        
    filecolor=DarkRed,      
    urlcolor=DarkBlue,          
    pdftitle={},
    pdfauthor={},
}

\usepackage[capitalize]{cleveref}

\def\draft{1}

\def\submit{0}

\ifnum\draft=1 
    \def\ShowAuthNotes{1}
\else
    \def\ShowAuthNotes{0}
\fi

\ifnum\submit=1
\newcommand{\forsubmit}[1]{#1}
\newcommand{\forreals}[1]{}
\else
\newcommand{\forreals}[1]{#1}
\newcommand{\forsubmit}[1]{}
\fi

\ifnum\ShowAuthNotes=1
\newcommand{\authnote}[2]{{ \footnotesize \bf{\color{DarkRed}[#1's Note:
{\color{DarkBlue}#2}]}}}
\else
\newcommand{\authnote}[2]{}
\fi

\newcommand{\cynthia}[1]{{\color{blue}Cynthia: #1}}
\newcommand{\jcpnote}[1]{{\color{purple}Juanky: #1}}
\newcommand{\ch}[1]{{\color{olive}[Chris: #1]}}
\newcommand{\pranay}[1]{{\color{red}\textbf{Pranay:} #1}}
\newcommand{\nsi}[1]{{\color{violet}\textbf{Nicole:} #1}}
\newcommand{\todo}{{\color{red} \textbf{TODO}}}

\renewcommand{\cynthia}[1]{{}}
\renewcommand{\jcpnote}[1]{{}}
\renewcommand{\ch}[1]{{}}
\renewcommand{\pranay}[1]{{}}
\renewcommand{\nsi}[1]{{}}
\renewcommand{\todo}[1]{{}}

\newtheorem{theorem}{Theorem}[section]
\newtheorem{remark}[theorem]{Remark}
\newtheorem{lemma}[theorem]{Lemma}
\newtheorem{corollary}[theorem]{Corollary}
\newtheorem{proposition}[theorem]{Proposition}
\newtheorem{fact}[theorem]{Fact}

\theoremstyle{definition} 

\newtheorem{definition}[theorem]{Definition}

\newtheorem*{definition*}{Definition}
\newtheorem*{proposition*}{Proposition}

\theoremstyle{definition}
\newtheorem*{example*}{Example}
\newtheorem{example}[theorem]{Example}

\newtheoremstyle{example_contd}
{\topsep} {\topsep}%
{}
{}
{\bfseries}
{.}
{1em}
{\thmname{#1} \thmnumber{ #2}\thmnote{#3} (continued)}

\theoremstyle{example_contd}

\newcommand{\chapterref}[1]{\hyperref[ch:#1]{Chapter~\ref{ch:#1}}}

\newcommand{\claimref}[1]{\hyperref[claim:#1]{Claim~\ref{claim:#1}}}

\newcommand{\corollaryref}[1]{\hyperref[cor:#1]{Corollary~\ref{cor:#1}}}

\newcommand{\definitionref}[1]{\hyperref[def:#1]{Definition~\ref{def:#1}}}

\newcommand{\equationref}[1]{\hyperref[eq:#1]{Equation~\ref{eq:#1}}}

\newcommand{\factref}[1]{\hyperref[fact:#1]{Fact~\ref{fact:#1}}}
\newcommand{\figurelabel}[1]{\label{fig:#1}}
\newcommand{\figureref}[1]{\hyperref[fig:#1]{Figure~\ref{fig:#1}}}

\newcommand{\tableref}[1]{\hyperref[tab:#1]{Table~\ref{tab:#1}}}

\newcommand{\itemref}[1]{\hyperref[item:#1]{Item~(\ref{item:#1})}}

\newcommand{\lemmaref}[1]{\hyperref[lem:#1]{Lemma~\ref{lem:#1}}}

\newcommand{\propref}[1]{\hyperref[prop:#1]{Proposition~\ref{prop:#1}}}

\newcommand{\propositionref}[1]{\hyperref[prop:#1]{Proposition~\ref{prop:#1}}}

\newcommand{\remarkref}[1]{\hyperref[rem:#1]{Remark~\ref{rem:#1}}}

\newcommand{\sectionref}[1]{\hyperref[sec:#1]{Section~\ref{sec:#1}}}

\newcommand{\theoremref}[1]{\hyperref[thm:#1]{Theorem~\ref{thm:#1}}}

\newcommand{\assumptionref}[1]{\hyperref[ass:#1]{Assumption~\ref{ass:#1}}}

\usepackage[utf8]{inputenc}
\usepackage[T1]{fontenc}
\usepackage{kpfonts}
\usepackage{microtype}

\newcommand{\Esymb}{\mathbb{E}}

\DeclareMathOperator*{\E}{\Esymb}

\renewcommand{\Pr}{\mathrm{Pr}}

\newcommand{\cA}{{\cal A}}
\newcommand{\cB}{{\cal B}}
\newcommand{\cC}{{\cal C}}
\newcommand{\cD}{{\cal D}}

\newcommand{\cF}{{\cal F}}
\newcommand{\cG}{{\cal G}}
\newcommand{\cH}{{\cal H}}
\newcommand{\cI}{{\cal I}}

\newcommand{\cL}{{\cal L}}

\newcommand{\cN}{{\mathbb{N}}}
\newcommand{\cO}{{\cal O}}
\newcommand{\littleo}{{o}}
\newcommand{\cP}{{\cal P}}

\newcommand{\cR}{{\mathbb{R}}}

\newcommand{\cU}{{\cal U}}

\newcommand{\cX}{{\cal X}}
\newcommand{\cY}{{\cal Y}}
\newcommand{\cZ}{{\cal Z}}

\newcommand{\defeq}{\stackrel{\small \mathrm{def}}{=}}
\renewcommand{\leq}{\leqslant}
\renewcommand{\le}{\leqslant}
\renewcommand{\geq}{\geqslant}
\renewcommand{\ge}{\geqslant}
\newcommand{\paren}[1]{(#1 )}

\newcommand{\brac}[1]{[#1 ]}

\newcommand{\Set}[1]{\left\{#1\right\}}
\newcommand{\abs}[1]{\lvert#1\rvert}

\newcommand{\R}{\mathbb{R}}

\newcommand{\N}{\mathbb N}

\usepackage{bm}

\newcommand{\ignore}[1]{}

\newcommand{\sign}{\mathrm{sign}}

\newcommand{\poly}{{\rm poly}}

\DeclareMathOperator*{\argmin}{arg\,min}
\DeclareMathOperator*{\argmax}{arg\,max}

\renewcommand{\epsilon}{\varepsilon}

\newcommand{\eps}{\epsilon}

\newcommand{\remove}[1]{}

\newcommand{\calY}{\mathcal{Y}}
\newcommand{\calYhat}{\hat{\calY}}

\newcommand{\yt}{\tilde{y}}
\newcommand{\pt}{\tilde{p}}
\newcommand{\cYh}{\calYhat}
\newcommand{\Ber}{\mathrm{Ber}}
\newcommand{\Reg}{\mathcal{R}_\cA}
\newcommand{\Rarg}[1]{\mathcal{R}_{#1}}
\newcommand{\DOI}{\mathrm{DOI}}
\newcommand{\HOI}{\mathrm{HOI}}
\newcommand{\red}[1]{\textcolor{red}{#1}}
\renewcommand{\red}[1]{{}}
\newcommand{\yhat}{\hat{y}}

\newcommand{\innerF}[2]{\langle {#1}, {#2} \rangle_{\cF}}

\newcommand{\inner}[2]{\langle {#1}, {#2} \rangle}

\newcommand{\kalg}{\text{K29}^*}
\newcommand{\embed}{\mathsf{Em}}

\newcommand{\flin}{\cF_{\mathrm{lin}}}
\renewcommand{\hbar}{\overline{h}}

\newcommand{\fnorm}[1]{\| #1\|_{\mathcal{F}}}

\newcommand{\algname}{\text{Any Kernel algorithm}}
\newcommand{\qalgname}{\text{Quantile Any Kernel algorithm}}
\newcommand{\valgname}{\text{Vector Any Kernel algorithm}}
\newcommand{\1}{{1}}
\renewcommand{\cU}{\mathcal{U}}
\newcommand{\rkhsnorm}[1]{\fnorm}
\newcommand{\kf}{kernel function}
\newcommand{\fair}{fair}
\newcommand{\Cont}{Cts}
\newcommand{\Ymin}{Y_{\min}}
\newcommand{\Ymax}{Y_{\max}}